%% file: main.tex
\documentclass[final,12pt]{colt2022} 

\title[TS Achieves Optimal Regret in Linear Quadratic Control]{Thompson Sampling Achieves $\tilde O(\sqrt{T})$  \\ Regret in Linear Quadratic Control}

\newcommand{\alg}{\textsc{\small{TSAC}}\xspace}

\input{defs}
\coltauthor{%
 \Name{Taylan Kargin}$^{\mathbf{*,1}}$ \Email{tkargin@caltech.edu}\\
 \Name{Sahin Lale}$^{\mathbf{*,1}}$ \Email{alale@caltech.edu}\\
 \Name{Kamyar Azizzadenesheli}$\mathbf{^2}$  \Email{kamyar@purdue.edu} \\
 \Name{Anima Anandkumar}$\mathbf{^1}$ \Email{anima@caltech.edu}\\
  \Name{Babak Hassibi}$\mathbf{^1}$ \Email{hassibi@caltech.edu}\\
 \addr  $^*$Equal contribution \\ $\mathbf{^1}$California Institute of Technology\\
 $\mathbf{^2}$Purdue University%
}

\begin{document}

\maketitle

\input{sections/0abstract}


\section{Introduction}\label{sec:intro}

\input{sections/1introduction}
\section{Preliminaries}\label{sec:prelim}
\input{sections/preliminaries}

\section{\alg Framework}\label{sec:new_algo}
\input{sections/new_algorithm}

\section{Theoretical Analysis}\label{sec:analysis}
\input{sections/analysis}

\section{Proof Outline of Sampling Optimistic Models with Constant Probability}\label{sec:overview}
\input{sections/proof_overview}
\section{Numerical Experiments}\label{sec:numerical}
\input{sections/numerical}


\section{Conclusion and Future Directions}\label{sec:discussion}
\input{sections/discussion}


\newpage

\bibliography{references}

\newpage
\appendix

\begin{center}
{\huge Appendix}
\end{center}
\input{appendix/0organization}

\section{Notation} \label{apx:notation}
\input{appendix/notation}

\newpage

\section{System Identification and Stabilization Guarantees} \label{apx:stabilization}
\input{appendix/stabilization}

\section{Boundedness of State, Proof of Lemma \ref{lem:bounded_state} } \label{apx:bounded_state}
\input{appendix/bdd_state}

\section{Constant Probability of Sampling Optimistic Models, Proof of Theorem \ref{thm:optimistic_prob}}\label{apx:proof}
\input{appendix/optimism_proof}

\section{Regret Decomposition}\label{apx:regret_dec}
\input{appendix/regret_decomposition}

\section{Regret Analysis}\label{apx:regret_anlz}
\input{appendix/regret_analysis}

\section{Technical Theorems}\label{apx:technical}
\input{appendix/technical}

\section{Implementation Details of Numerical  Experiments}\label{apx:simulation}
\input{appendix/simulation}


\end{document}

%% file: defs.tex
\usepackage{hyperref}       
\usepackage{url}            
\usepackage{booktabs}       
\usepackage{nicefrac}       

\usepackage{floatrow}
\usepackage{algorithm}
\usepackage{algorithmic}
\usepackage{bm}
\usepackage{bbm}
\usepackage{mathtools}
\usepackage{tikz}
\usepackage{times}
\usepackage{xspace}
\usepackage{enumitem}
\usepackage{lipsum}

\usepackage{autobreak}
\usepackage[rightcaption]{sidecap}
\sidecaptionvpos{figure}{m}
\usepackage{caption}

\usepackage[capitalize,noabbrev]{cleveref}
\DeclareMathAlphabet{\mathcal}{OMS}{cmsy}{m}{n}

\newcommand{\QED}{\hfill $\blacksquare$}

\newcommand{\poly}{\operatorname{poly}}
\newtheorem{assumption}[theorem]{Assumption}

\renewcommand{\tt}{\Theta}
\DeclareMathOperator{\Tr}{tr}

\newcommand{\ttt}{\tilde{\Theta}}
\newcommand{\tth}{\hat{\Theta}}
\newcommand{\tts}{\Theta_*}

 \newcommand{\Tw}{T_w}

\newcommand{\opt}{\operatorname{opt}}

\newcommand{\N}{\mathbb{N}} 
\newcommand{\R}{\mathbb{R}} 
\renewcommand{\S}{\mathcal{S} } 

\newcommand{\M}{\clf{M}}

\newcommand{\ie}{\textit{i.e.}}
\newcommand{\nn}{\nonumber}

\newcommand{\pr}[1]{\left({#1}\right)} 
\newcommand{\br}[1]{\left[{#1}\right]} 
\newcommand{\cl}[1]{\left\{{#1}\right\}} 
\newcommand{\abs}[1]{\left\vert{#1}\right\vert} 
\newcommand{\norm}[2][\text{}]{\Vert{#2}\Vert_{#1}} 
\newcommand{\Norm}[2][\text{}]{\left\Vert{#2}\right\Vert_{#1}} 

\newcommand{\indic}{\mathbbm{1}} 

\newcommand{\defeq}{\coloneqq} 
\newcommand{\eqdef}{\eqqcolon} 

\newcommand{\clf}[1]{\mathcal{#1}} 
\newcommand{\half}{\frac{1}{2}} 

\newcommand{\tr}{\operatorname{tr}} 
\newcommand{\tp}{\intercal}
\newcommand{\psdleq}{\preccurlyeq} 
\newcommand{\psdgeq}{\succcurlyeq} 
\newcommand{\psdl}{\prec} 
\newcommand{\psdg}{\succ} 

\newcommand{\grad}{\nabla} 
\newcommand{\diff}{\mathrm{d}} 
\newcommand{\ddx}[1]{\frac{\diff}{\diff{#1}}} 

\newcommand{\E}{\operatorname{\mathbb{E}}} 
\renewcommand{\Pr}{\mathbb{P}} 
\newcommand{\normal}{\operatorname{\mathcal{N}}} 
\newcommand{\F}{\mathcal F} 
\newcommand{\sampled}[1][\text{}]{\stackrel{#1}{\sim}}
\newcommand{\goesto}{\rightarrow} 
\newcommand{\Cov}{\operatorname{Cov}}
\newcommand{\cond}{\,\big| \,}

\newcommand{\OO}{\Tilde{O}}



%% file: sections/0abstract.tex
\begin{abstract}%
Thompson Sampling (TS) is an efficient method for decision-making under uncertainty, where an action is sampled from a carefully prescribed distribution which is updated based on the observed data. In this work, we study the problem of adaptive control of stabilizable linear-quadratic regulators (LQRs) using TS, where the system dynamics are unknown. Previous works have established that $\tilde O(\sqrt{T})$ frequentist regret is optimal for the adaptive control of LQRs. However, the existing methods either work only in restrictive settings, require a priori known stabilizing controllers, or utilize computationally intractable approaches. We propose an efficient TS algorithm for the adaptive control of LQRs, \textbf{TS}-based \textbf{A}daptive \textbf{C}ontrol, \alg, that attains $\tilde O(\sqrt{T})$ regret, even for multidimensional systems, thereby solving the open problem posed in \citet{abeille2018improved}. \alg does not require a priori known stabilizing controller and achieves fast stabilization of the underlying system by effectively exploring the environment in the early stages. Our result hinges on developing a novel lower bound on the probability that the TS provides an optimistic sample. By carefully prescribing an early exploration strategy and a policy update rule, we show that TS achieves order-optimal regret in adaptive control of multidimensional stabilizable LQRs. We empirically demonstrate the performance and the efficiency of \alg in several adaptive control tasks.
\end{abstract}

%% file: sections/1introduction.tex
There has been a significant development in data-driven methods for controlling dynamical systems in recent years due to the development of novel reinforcement learning approaches and techniques \citep{hou2013model}. Adaptive control of unknown linear dynamical systems has been the main focus due to its simplicity and its ability to capture the crux of the problem and give insights on more challenging tasks \citep{recht2019tour}. Among linear dynamical systems, Linear Quadratic Regulators (LQRs) are the canonical settings with quadratic regulatory costs to design desirable controllers and have been studied in an array of prior works \citep{abbasi2011lqr,faradonbeh2017optimism,abeille2018improved,mania2019certainty,simchowitz2020naive,chen2020black,lale2022reinforcement}. These works provide finite-time performance guarantees of adaptive control algorithms in terms of \textit{regret}, which is the difference between the attained cumulative cost and the expected cost of the optimal controller. In particular, they show that $\tilde{O}(\sqrt{T})$ regret after $T$ time steps is optimal in adaptive control of LQRs. They utilize several different paradigms for algorithm design such as Certainty Equivalence, Optimism or Thompson Sampling, yet, they suffer either from the inherent algorithmic drawbacks or limited applicability in practice.

\textbf{Certainty equivalent control and its challenges:} Certainty equivalent control (CEC) is one of the most straightforward paradigms for control design in adaptive control of dynamical systems. In CEC, an agent obtains a nominal estimate of the system, and executes the optimal control law for this estimated system. Even though \citet{mania2019certainty,simchowitz2020naive} show that this simple approach attains optimal regret in LQRs, the proposed algorithms have several drawbacks. First and foremost, CEC is sensitive to model mismatch and requires significantly small model estimation error to a point that exploration of the system dynamics is not required. Since this level of refinement is challenging to obtain for an unknown system, these methods rely on access to an initial stabilizing controller to enable a long exploration. In practice, such a priori known controllers may not be available, which hinders the deployment of these algorithms.

\textbf{Optimism-based control and its challenges:} Optimism is one of the most prominent methods to effectively balance exploration and exploitation in adaptive control \citep{bittanti2006adaptive}. In optimism-based control, an agent executes the optimal policy for the model with the lowest cost within a set of plausible models. In \citet{abbasi2011lqr,faradonbeh2017optimism}, the authors use optimism-based control design to achieve $\tilde{O}(\sqrt{T})$ regret with exponential dimension dependency. Both algorithms solve a non-convex optimization problem to find the optimistic controllers, which is an NP-hard problem in general \citep{agrawal2019recent}. Unfortunately, this computational inefficiency severely limits their practicality. Recently, \citet{abeille2020efficient} proposed a relaxation to the optimistic controller computation, which makes the optimism-based controllers efficient. However, their algorithm also requires a significantly well-refined model estimate and a given initial stabilizing policy, similar to CEC. 

\textbf{Restricted LQR settings in the prior works:} In our work, we study the stabilizable multi-dimensional LQR setting. Stabilizability is necessary and sufficient condition to have a well-posed LQR control problem \citep{kailath2000linear}. On the contrary, prior works usually consider the controllable LQR setting, which is a subclass of stabilizable LQRs \citep{cohen2019learning,chen2020black}. While the controllability condition simplifies the learning and control problem, it is also often violated in many real-world control systems \citep{friedland2012control}. Recently, \citet{lale2022reinforcement} proposed an adaptive control algorithm that does not need an initial stabilizing controller and achieves optimal regret in stabilizable LQRs. However, their method relies on optimism, and unfortunately inherits the aforementioned computational complexity of optimistic methods.

\textbf{Thompson Sampling and its challenges:} Thompson Sampling (TS) is one of the oldest strategies to balance the exploration vs. exploitation trade-off \citep{thompson1933likelihood}. In TS, the agent samples a model from a distribution  computed based on prior control input and observation pairs, and then takes the optimal action for this sampled model and updates the distribution based on its novel observation. Since it relies solely on sampling, this approach provides polynomial-time algorithms for adaptive control. Therefore, it is a promising alternative to overcome the computational burden faced in optimismic control design. For this reason, \citet{abeille2017thompson,abeille2018improved} propose adaptive control algorithms using TS. In particular, \citet{abeille2018improved} provide the first TS-based adaptive control algorithm for LQRs that attains optimal regret of $\tilde{O}(\sqrt{T})$. However, their result \textit{only holds for scalar} stabilizable systems, since they were able to show that TS samples optimistic parameters with constant probability in only scalar systems. Further, they conjecture that this is true in multidimensional systems as well and TS-based adaptive control can provide optimal regret in multidimensional LQRs, and provide a simple numerical example to support their claims. 

\vspace{-0.6em}
\subsection*{Contributions} 

In this work, we give an affirmative answer to the conjecture posed in \citet{abeille2018improved}: 

\begin{enumerate}[wide, labelindent=0pt,label=\textbf{\arabic*}), itemsep=0.2pt, topsep=0.2pt]
    \item We propose an efficient adaptive control algorithm, \textbf{T}hompson \textbf{S}ampling-based  \textbf{A}daptive  \textbf{C}ontrol (\alg), that attains $\tilde{O}(\sqrt{T})$ regret in multidimensional stabilizable LQRs. This makes \alg the first efficient adaptive control algorithm to achieve order-optimal regret in all stabilizable LQRs without the prior knowledge of a stabilizing policy (Table \ref{table:1}).
    \item We empirically demonstrate the performance of \alg and compare to the optimism (heuristic) and TS-based methods that do not require initial stabilizing policy in flight control of Boeing 747 with linearized dynamics. We show that \alg effectively explores the system to find a stabilizing policy and achieves the competitive regret performance, while being computationally feasible.
\end{enumerate}

\begin{table}
\centering
\caption{Comparison with the prior works that attain $\Tilde{O}(\sqrt{T})$ regret on LQR, $\dagger=$ 1-dim LQRs}
\label{table:1}
 \begin{tabular}{l c c c}
\toprule
\textbf{Work} &  \textbf{Setting} &   \textbf{Stabilizing Controller} & \textbf{Computation}  \\
 \hline 
 
 \citep{abeille2018improved} & Stabilizable$^\dagger$ & Not Required & Feasible \\
 
 \citep{mania2019certainty}  & Controllable & Required & Feasible \\
  \citep{simchowitz2020naive}  & Stabilizable &Required & Feasible \\
  
 
\citep{chen2020black}  & Controllable   & Not required & Feasible\\

 \citep{lale2022reinforcement} & Stabilizable & Not required & Infeasible \\ 
 
 \textbf{This work} & Stabilizable & Not required & Feasible
\end{tabular}
 \vspace{-0.6em}
\end{table}

The design of \alg and our regret guarantee hinge on three important pieces missing in prior works: Fixed policy update rule, improved exploration in early stages of adaptive control, and a novel lower bound that shows TS samples optimistic parameters with non-zero probability in multidimensional LQRs. Unlike the frequent policy update rule of \citet{abeille2018improved} in scalar LQRs, \alg updates its policy with fixed time periods. This policy update rule prevents fast policy changes that would cause state blow-ups in stabilizable LQRs. In the beginning of agent-environment interaction, \alg focuses on quickly finding a stabilizing controller to avoid state blow-ups due to lack of a known initial stabilizing policy. By using isotropic exploration in the early stages along with the exploration of TS policy, we show that \alg achieves fast stabilization. 

After stabilizing the unknown system dynamics, \alg relies on the effective exploration of the TS to find desirable controllers. In particular, we show that the TS samples optimistic parameters with a constant probability in any LQR setting. This novel lower bound shows that the TS is an efficient alternative to optimism in all adaptive control problems in LQRs. Combining this lower bound with the fixed policy update rule, we derive the optimal regret guarantee for \alg. 

\vspace{-0.3em}

%% file: sections/preliminaries.tex
\textbf{Notation:} We denote the Euclidean norm of a vector $x$ as $\|x\|_2$. For a matrix $A\! \in\! \mathbb{R}^{n\times d}$, we denote $\rho(A)$ as the spectral radius of $A$, $\| A\|_F$ as its Frobenius norm and $\| A \|$ as its spectral norm. $\tr(A)$ denotes its trace, $A^\tp$ is the transpose. For any positive definite matrix $V$, $\|A\|_V \!=\! \|V^{1/2} A \|_F $. For matrices $A,B\! \in\! \R^{n\times d}$, $A\bullet B \!=\!\tr(A B^\tp)$ denotes their Frobenius inner product. The j-th singular value of a rank-$n$ matrix $A$ is $\sigma_j(A)$, where $\sigma_{\max}(A)\!:=\!\sigma_1(A) \!\geq\! \ldots \! \geq \! \sigma_{\min}(A)\!:=\!\sigma_n(A)$. $I$ represents the identity matrix with the appropriate dimensions. $\M_n\!=\!\R^{n\times n}$ denotes the set of $n$-dimensional square matrices. $\mathcal{N}(\mu, \Sigma)$ denotes normal distribution with mean $\mu$ and covariance $\Sigma$. $Q(\cdot)$ denotes the Gaussian $Q$-function. $O(\cdot)$ and $o(\cdot)$ denote the standard asymptotic notation and  $f(T)=\omega(g(T))$ is equivalent to $g(T)=o(f(T))$. $\tilde{O}(\cdot)$ presents the order up to logarithmic terms.

\subsection{Setting}
Suppose we are given a discrete time linear time-invariant system with the following dynamics,
\begin{equation} \label{output}
    x_{t+1} = A_* x_t + B_* u_t + w_t, 
\end{equation}
where $x_t \in \mathbb{R}^{n}$ is the state of the system, $u_t \in \mathbb{R}^{d}$ is the control input, $w_{t} \in \mathbb{R}^{n}$ is i.i.d. process noise at time $t$. At each time step t, the system is at state $x_t$ where the agent observes the state. Then, the agent applies a control input $u_t$ and the system evolves to $x_{t+1}$ at time $t + 1$. The underlying system (\ref{output}) can be represented as $ x_{t+1} = \tts^\tp z_t + w_t$, where $\Theta_*^\tp = [A_* \enskip B_*]$ and $z_t = [x_t^\tp \enskip u_t^\tp]^\tp$. In this work, we consider stabilizable linear dynamical systems $\tts$, such that there exists a controller $K$ where $\rho(A_* + B_*K) < 1$. More precisely, the systems with the following property:
\begin{assumption}[Bounded and $(\kappa,\gamma)$-stabilizable System]\label{asm_stabil}
 The unknown system $\tts$ is a member of a set $\mathcal{S}$ such that $\mathcal{S} \subseteq \big\{ \Theta'=[A', B'] ~\big|~ \Theta' \text{ is } (\kappa, \gamma)\text{-stabilizable, } \|\Theta'\|_F \leq S\big\}$ for some $\kappa \geq 1$ and $0 < \gamma \leq 1$. In particular, for the underlying system $\tts$, we have $\|K(\tts)\| \leq \kappa$ and there exists $L$ and $H\succ 0$ such that $A_* + B_*K(\tts) = HLH^{-1}$, with $\|L\| \leq 1-\gamma$ and $\|H\| \|H^{-1} \| \leq \kappa $. 
\end{assumption}

Note that the stabilizability condition is necessary and sufficient condition to define the optimal control problem~\citep{kailath2000linear} and it is weaker than the controllability assumption considered in prior works~\citep{abbasi2011lqr,cohen2019learning,chen2020black}. In particular, the set of stabilizable systems subsumes the set of controllable systems. Moreover, $(\kappa,\gamma)$-stabilizability is merely a quantification of stabilizability for the finite-time analysis and it is adopted in recent works~\citep{cohen2018online,cohen2019learning,cassel2020logarithmic}. One can show that any stabilizable system is also $(\kappa,\gamma)$-stabilizable for some $\kappa$ and $\gamma$, conversely, $(\kappa,\gamma)$-stabilizability implies stabilizability (Lemma B.1 \citet{cohen2018online}). We have the following assumption on $w_t$.

\begin{assumption}[Gaussian Process Noise]
\label{asm_noise}
There exists a filtration $\mathcal{F}_{t}$ such that for all $t\geq 0$, $x_t, z_t$ are $\mathcal{F}_{t}$-measurable and $w_t|\mathcal{F}_{t} =  \mathcal{N}(0, \sigma_w^2 I)$ for some known $\sigma_w>0$.
\end{assumption}

Note that this assumption is standard in literature and adopted for simplicity of exposure. The following results can be extended to sub-Gaussian process noise setting using the techniques developed in~\citet{lale2022reinforcement}. At each time step, the regulating cost is $c_t = x_t^\tp Q x_t + u_t^\tp R u_t$, where $Q \in \mathbb{R}^{n\times n}$ and $R \in \mathbb{R}^{d \times d}$ are known positive definite matrices such that $\|Q\|,\|R\| < \Bar{\alpha}$ and $\sigma_{\min}(Q), \sigma_{\min}(R) > \underline{\alpha} > 0$. The goal is to minimize the average expected cost
\begin{equation} \label{opt control}
J(\Theta_*)=\lim _{T \rightarrow \infty} \min _{u=\left[u_{1}, \ldots, u_{T}\right]} \frac{1}{T} \mathbb{E}\Big[\sum\nolimits_{t=1}^{T} x_{t}^{\tp} Q x_{t}+u_{t}^{\tp} R u_{t}\Big],
\end{equation}
by designing control inputs based on past observations. This problem is the canonical infinite horizon linear quadratic regulator (LQR) problem. If the underlying system $\Theta_*$ is known, the solution of the optimal control problem is a linear feedback control $u_t = K(\Theta_*)x_t$ with $K(\Theta_*) = -(R+B_*^{\tp} P(\Theta_*) B_*)^{-1} B_*^{\tp} P(\Theta_*) A_*$, where $P(\Theta_*)$ is the unique positive definite solution to
\begin{equation} \label{DARE}
P(\Theta_*) = A_*^\tp P(\Theta_*) A_* + Q -  A_*^\tp P(\Theta_*) B_* ( R + B_*^\tp P(\Theta_*) B_* )^{-1} B_*^\tp P(\Theta_*) A_*,
\end{equation}
\textit{i.e.}, the discrete algebraic Riccati equation (DARE), and $J(\Theta_*) = \sigma_w^2\Tr(P(\Theta_*))$. Note that since the system is stabilizable, $J(\tts) < \infty$. In fact, using Assumption \ref{asm_stabil}, one can show that $\|P(\Theta') \| \leq D \coloneqq \Bar{\alpha} \gamma^{-1} \kappa^2(1+\kappa^2) $ for all $\Theta' \in \mathcal{S}$, including $\tts$ (Lemma 2.1 of \citet{lale2022reinforcement}).

\subsection{Finite-Time Adaptive Control Problem} \label{subsec:freq_regret}

In this work, we consider the adaptive control setting, where $\Theta_*$ is \textit{unknown}. The goal in the finite-time adaptive control problem is to minimize the cumulative cost \textit{i.e.}, $\sum_{t=0}^T c_t$. In order to design a controller that achieves this goal, the controlling agent needs to interact with the system to learn the $\Theta_*$ that governs the dynamics. However, due to a lack of knowledge of model dynamics, the agent takes sub-optimal actions. In this work, we use \textit{regret}, $R_T$, as the metric to evaluate the finite-time performance of the controlling agent. The regret quantifies the difference between the performance of the agent and the expected performance of the optimal controller, $R_T = \sum\nolimits_{t=0}^T (c_t - J(\Theta_*)).$

\vspace{-0.5em}
\subsection{Learning the System Dynamics} \label{subsec:learning}
For any given input and state pairs up to time $t$, $\Theta_*$ can be estimated using regularized least squares (RLS) for some $\mu> 0$: $ \min_\Theta \sum_{s=0}^{t-1} \tr \left((x_{s+1}\!-\!\Theta^{\tp} z_{s})(x_{s+1}\!-\!\Theta^{\tp} z_{s})^{\tp}\right) + \mu \|\Theta\|_F^2$. The solution is given as $\hat{\Theta}_t = V_t^{–1} \sum_{s=0}^{t–1} z_s x_{s+1}^\tp$ where $V_t = \mu I + \sum_{s=0}^{t–1} z_s z_s^\tp$. Using Theorem 1 of \citet{abbasi2011lqr}, for any $\delta \in (0,1)$, for all $0\leq t \leq T$, the underlying parameter $\Theta_*$ lives in $\mathcal{E}_t^{\text{RLS}}(\delta)$ with probability at least $1-\delta$ where $\mathcal{E}_t^{\text{RLS}}(\delta) = \{\Theta : \| \Theta-\hat{\Theta}_t\|_{V_t} \leq \beta_t(\delta) \}$
for $\beta_t(\delta) = \sigma_w \sqrt{2 n \log ( ( \det(V_{t})^{1 / 2}) / (\delta \det(\mu I)^{1 / 2} ))} + \sqrt{\mu}S$. 




%% file: sections/new_algorithm.tex
In this section, we present \alg, a sample efficient TS-based adaptive control algorithm for the unknown stabilizable LQRs. The algorithm is summarized in Algorithm \ref{algo_exact}. It has two phases: 1) TS with improved exploration and 2) Stabilizing TS.  

\paragraph{TS with Improved Exploration}
Due to lack of a priori known stabilizing controller, \alg focuses on rapidly learning stabilizing controllers in the early stages of the algorithm. To achieve this, \alg explores the system dynamics effectively in this phase. At any time-step $t$, given the RLS estimate $\tth_t$ and the design matrix $V_t$ as described in Section \ref{subsec:learning}, \alg samples a perturbed model parameter $\ttt_t = \mathcal{R}_{\mathcal{S}}(\tth_t + \beta_t(\delta)V_t^{-1/2}\eta_t),  $
where $\mathcal{R}_{\mathcal{S}}$ denotes the rejection sampling operator associated with the set $\mathcal{S}$ given in Assumption \ref{asm_stabil} and $\eta_t \in \mathbb{R}^{(n+d) \times n}$ is a matrix with independent standard normal entries. Here $\mathcal{R}_{\mathcal{S}}$ guarantees that $\ttt_t \in \mathcal{S}$ and $\beta_t(\delta)V_t^{-1/2}\eta_t$ randomizes the sampled parameter coherently with the RLS estimate and the uncertainty associated with it.  Using this sampled model parameter, \alg constructs the optimal linear controller $\bar{u}_t = K(\ttt_t) x_t$ for $\ttt_t$. 

However, to obtain stabilizing controllers for an unknown linear dynamical system, one needs to explore the state-space in all directions (Lemma 4.2 of \citep{lale2022reinforcement}). Unfortunately, due to lack of reliable estimates in the early stages, deploying the policy achieved via TS, $\bar{u}_t$, may not achieve such effective exploration. Therefore, in the early stages of interactions with the underlying system, \alg deploys isotropic perturbations along with the sampled policy. In particular, for the first $T_w$ time-steps, \alg uses $u_t = \bar{u}_t + \nu_t$ as the control input where $\nu_t \sim \mathcal{N}(0, 2\kappa^2\sigma_w^2 I)$. This improved exploration policy effectively excites and explores all dimensions of the system to certify the design of stabilizing controllers. \alg sets $\Tw$ such that all the sampled controllers $K(\ttt_t)$ are guaranteed to stabilize the underlying system $\tts$ for all $t>\Tw$ (Appendix \ref{apx:stabilization}).

Unlike most of the popular RL strategies that follow lazy updates, \alg updates its sampled policy in every fixed $\tau_0$ steps, \ie, the same sampled policy $K(\ttt_t)$ is deployed for $\tau_0$ time-steps. This update rule is carefully chosen such that \alg samples enough optimistic policies to reduce the cumulative regret and avoids too frequent policy changes which would cause state blow-ups.

\paragraph{Stabilizing TS} After guaranteeing the design of stabilizing policies with improved exploration in the first phase, \alg starts the adaptive control with only TS. In particular, for the remaining time-steps, \alg deploys $u_t =  K(\ttt_t) x_t$ for $\ttt_t = \mathcal{R}_{\mathcal{S}}(\tth_t + \beta_t(\delta)V_t^{-1/2}\eta_t)$ and updates the sampled model parameter in every $\tau_0$ time-steps. Note that, even though all the policies during this phase are stabilizing, frequent policy changes can still cause undesirable state growth. \alg prevents this possibility by applying the same control policy for $\tau_0$ time-steps in this phase as well. During this phase, \alg decays the possible state blow-ups in the first phase and maintains stable dynamics. 

\vspace{-0.4em}

\begin{algorithm}[t] 
	\caption{\alg}
	\begin{algorithmic}[1]
		\STATE \textbf{Input:}  $\kappa$, $\gamma$, $Q$, $R$, $\sigma_w^2$ , $V_0 = \mu I$, $\hat{\Theta}_{0} = 0$
		
	\FOR{$i = 0, 1, \ldots$}
    \STATE Estimate $\hat{\Theta}_i$ \& Sample $\tilde{\Theta}_i = \mathcal{R}_{\mathcal{S}}(\hat{\Theta}_i + \beta_tV_t^{-1/2}\eta_t)$
    \FOR{$t = i\tau_0, \ldots, (i+1)\tau_0 -1$ }
        \IF {$t\leq \Tw$} 
			 \STATE Deploy  $u_t \!=\! K(\tilde{\Theta}_{i}) x_t \!+\! \nu_t $ \hfill \algorithmiccomment{ TS with Improved Exploration} 
		\ELSE
	         \STATE Deploy  $u_t \!=\! K(\tilde{\Theta}_{i}) x_t $ \hfill \algorithmiccomment{Stabilizing TS}
		\ENDIF
    \ENDFOR
\ENDFOR	
	\end{algorithmic}
	\label{algo_exact} 
\end{algorithm}

%% file: sections/analysis.tex
In this section, we study the theoretical guarantees of \alg. The following states the first order-optimal frequentist regret bound for TS in multidimensional stabilizable LQRs, our main result. 


\begin{theorem}[Regret of \alg]\label{reg:exp_s}
Suppose Assumptions~\ref{asm_stabil} and~\ref{asm_noise} hold and set $\tau_0 = 2\gamma^{-1} \log(2\kappa\sqrt{2})$ and $T_0 = \poly(\log(1/\delta), \sigma_w^{-1}, n, d, \bar{\alpha}, \gamma^{-1}, \kappa)$. Then, for long enough $T$, \alg achieves the regret $R_T \!=\! \OO\left((n+d)^{(n+d)}\sqrt{T\log(1/\delta)}\right)$  w.p. at least $1-10\delta$, if  $\Tw = \max\pr{T_0,\, c_1 (\sqrt{T}\log{T})^{1+o(1)}}$ for a constant $c_1>0$. Furthermore, if the closed loop matrix of the optimally controlled underlying system, $A_{c,*}\!\defeq \! A_* + B_* K_*$, is non-singular, w.p. at least $1-10\delta$, \alg achieves the regret $R_T \!=\! \OO\left(\poly(n,d)\sqrt{T\log(1/\delta)}\right)$ if $\Tw =\max\pr{T_0, \, c_2(\log{T})^{1+o(1)}}$ for a constant $c_2>0$.
\end{theorem}

This makes \alg the \textit{first efficient} adaptive control algorithm that achieves optimal regret in adaptive control of all LQRs \textit{without an initial stabilizing policy}. To prove this result, we follow similar approach as the existing methods in literature, and define the high probability joint event $E_t = \hat{E}_t \cap \Tilde{E}_t \cap \Bar{E}_t$, where $\hat{E}_t$ states that the RLS estimate $\tth$ concentrates around $\tts$, $\Tilde{E}_t$ states that the sampled parameter $\ttt$ concentrates around $\tth$, and $\Bar{E}_t$ states that the state remains bounded respectively (Appendix \ref{apx:bounded_state}). Conditioned on this event, we decompose the frequentist regret as, $R_T\indic_{E_T} \leq R_{\Tw}^{\text{exp}} + R_{T}^{\text{RLS}} + R_{T}^{\text{mart}} + R_{T}^{\text{TS}} + R_{T}^{\text{gap}}$, 
    %
where $R_{\Tw}^{\text{exp}}$ accounts for the regret attained due to improved exploration, $R_{T}^{\text{RLS}}$ represents the difference between the value function of the true next state and the predicted next state, $ R_{T}^{\text{mart}}$ is a martingale with bounded difference, $R_{T}^{\text{TS}}$ measures the difference in optimal average expected cost between the true model $\tts$ and the sampled model $\ttt$, and $R_{T}^{\text{gap}}$ measures the regret due to policy changes. The decomposition and expressions are given in Appendix \ref{apx:regret_dec}. In the analysis, we bound each term separately (Appendix \ref{apx:regret_anlz}). Before discussing the details of the analysis, we first consider the prior works that use TS for adaptive control of LQRs and discuss their shortcomings. Further, we highlight the challenges in adaptive control of multidimensional stabilizable LQRs using TS and present our approaches to overcome these.

\vspace{-0.6em}
\subsection{Prior Work on TS-based Adaptive Control and Challenges} \label{subsec:prior}
For the frequentist regret minimization problem given in Section \ref{subsec:freq_regret}, the state-of-the-art adaptive control algorithm that uses TS is  \citet{abeille2018improved}. They consider the ``contractible'' LQR systems, i.e. $|A_* + B_* K(\tts) | < 1$, and provide $\Tilde{O}(\sqrt{T})$ regret upper bound for scalar LQRs, i.e. $n=d=1$. Notice that the set of contractible systems is a small subset of the set $\mathcal{S}$ defined in Assumption \ref{asm_stabil} and they are only equivalent for scalar systems since $\rho(A_* - B_* K(\Theta_*)) = |A_* - B_* K(\Theta_*)|$. This simplified setting allow them to reduce the regret analysis into the trade-off between $R_{T}^{\text{TS}} = \sum_{t=0}^{T} \{ J(\Tilde{\Theta}_t) -J(\Theta_*)\}$ and $  R_{T}^{\text{gap}} \!=\!  \sum_{t=0}^{T} \mathbb{E}[x_{t_1}^\tp (P(\Tilde{\Theta}_{t+1}) \!-\! P(\Tilde{\Theta}_{t}) x_{t+1} \,\big|\,\clf{F}_t ]$. 

These regret terms are central in the analysis of several adaptive control algorithms. In the certainty equivalent control approaches, $R_{T}^{\text{TS}}$ is bounded by the quadratic scaling of model estimation error after a significantly long exploration with a known stabilizing controller~\citep{mania2019certainty,simchowitz2020naive}. In the optimism-based algorithms, $R_{T}^{\text{TS}}$ is bounded by $0$ by design~\citep{abbasi2011lqr,faradonbeh2017optimism}. Similarly, in Bayesian regret setting, \citep{ouyang2017learning} assume that the underlying parameter $\tts$ comes from a known prior that the expected regret is computed with respect to. This true prior yields $\mathbb{E}[R_{T}^{\text{TS}}]\!=\!0$ in certain restrictive LQRs. The conventional approach in the analysis of $R_{T}^{\text{gap}}$ is to have lazy policy updates, \ie, $O(\log T)$ policy changes, via doubling the determinant of $V_t$~\citep{abeille2017thompson,lale2022reinforcement} or exponentially increasing epoch durations~\citep{faradonbeh2020adaptive,cassel2020logarithmic}.\looseness=-1  

On the other hand, \citet{abeille2018improved} bound $R_{T}^{\text{TS}}$ by showing that TS samples the optimistic parameters, $\ttt_t$ such that $J(\Tilde{\Theta}_t) \leq J(\Theta_*)$, with a constant probability, which reduces the regret of non-optimistic steps. Unlike the conventional policy update approaches, the key idea in \citet{abeille2018improved} is to update the control policy every time-steps via TS, which increases the amount of optimistic policies during the execution. They show that while this frequent update rule reduces $R_{T}^{\text{TS}}$, it only results with $R_{T}^{\text{gap}} = \OO(\sqrt{T})$. However, they were only able to show that this constant probability of optimistic sampling holds for scalar LQRs. 

The difficulty of the analysis for the probability of optimistic parameter sampling lies in the challenging characterization of the optimistic set. Since $J(\ttt) \!=\! \sigma_w^2\tr(P(\ttt))$, one needs to consider the spectrum of $P(\ttt)$ to define optimistic models, which makes the analysis difficult. In particular, decreasing the cost along one direction may be result in an increase in other directions. However, for the scalar LQR setting considered in \citet{abeille2018improved}, $J(\ttt) \!=\! P(\ttt)$ and using standard perturbation results on DARE suffices. As mentioned in \citet{abeille2018improved}, one can naively consider the surrogate set of being optimistic in all directions, \textit{i.e.} $P(\ttt) \psdleq P(\tts)$. Nevertheless, this would result in probability that decays linear in time and does not yield sub-linear regret. In this work, we propose new surrogate sets to derive a lower bound on the probability of having optimistic samples, and show that TS in fact samples optimistic model parameters with constant probability. 

In designing TS-based adaptive control algorithms for multidimensional stabilizable LQRs, one needs to maintain bounded state. In bounding the state, \citet{abeille2018improved} rely on the fact that the underlying system is contractive, $\| \Tilde{A} + \Tilde{B} K(\ttt)  \| < 1$. However, under Assumption \ref{asm_stabil}, even if the optimal policy of the underlying system is chosen by the learning agent, the closed-loop system may not be contractive since for any symmetric matrix $M$, $\rho(M) \leq \|M\|$. Thus, to avoid dire consequences of unstable dynamics, TS-based adaptive control algorithms should focus on finite-time stabilization of the system dynamics in the early stages. 

Moreover, the lack of contractive closed-loop mappings in stabilizable LQRs, prevent frequent policy changes used in \citet{abeille2018improved}. From the definition of $(\kappa, \gamma)$-stabilizability (Assumption \ref{asm_stabil}), for any stabilizing controller $K'$, we have that $A_* + B_*K' = H'LH'^{-1}$, with $\|L\| < 1$ for some similarity transformation $H'$. Thus, even if all the policies are stabilizing, changing the policies at every time step could cause couplings of these similarity transformations and result in linear growth of state over time. Thus, TS-based adaptive control algorithms need to find the balance in rate of policy updates, so that frequent policy switches are avoided, yet, enough optimistic policies are sampled. In light of these observations, our results hinge on the following:

\begin{enumerate}[wide, labelindent=0pt,label=\textbf{\arabic*}),itemsep=0.1pt, topsep=0.2pt]
     \item Improved exploration of \alg, which allows fast stabilization of the system dynamics,
    \item Fixed policy update rule of \alg, which prevents state blow-up and reduces $R_{T}^{\text{gap}}$ and $R_{T}^{\text{TS}}$,
    \item A novel result that shows TS samples optimistic model parameters with a constant probability for multidimensional LQRs and gives a novel bound on $R_{T}^{\text{TS}}$.
\end{enumerate}

\subsection{Details of the analysis}
The improved exploration along with TS in the early stages allows \alg to effectively explore the state-space in all directions. The following shows that for a long enough improved exploration phase, \alg achieves consistent model estimates and guarantees the design of stabilizing policies.

\begin{lemma}[Model Estimation Error and Stabilizing Policy Design]\label{lem:2norm_bound} Suppose Assumptions \ref{asm_stabil} and \ref{asm_noise} hold. For $t \geq 200 (n+d) \log \frac{12}{\delta}$ time-steps of TS with improved exploration, with probability at least $1-2\delta$, \alg obtains model estimates such that $\|\tth_{t}-\tts\|_2 \leq 7\beta_{t}(\delta) / (\sigma_w \sqrt{t})$. Moreover, after $\Tw \geq T_0 \coloneqq \poly(\log(1/\delta), \sigma_w^{-1}, n, d, \bar{\alpha}, \gamma^{-1}, \kappa)$ length TS with improved exploration phase, with probability at least $1-3\delta$, \alg samples controllers $K(\ttt_t)$ such that the closed-loop dynamics on $\tts$ is $(\kappa\sqrt{2}, \gamma/2)$ strongly stable for all $t > \Tw$, i.e. there exists $L$ and $H\succ 0$ such that $A_* + B_* K(\ttt_t) = HLH^{-1}$, with $\|L\| \leq 1-\gamma/2$ and $\|H\| \|H^{-1} \| \leq \kappa\sqrt{2} $. 
\end{lemma}

The proof and the precise expression of $\Tw$ can be collected in Appendix \ref{apx:stabilization}. In the proof, we show that the inputs $u_t = K(\ttt_i)x_t + \nu_t$ for $\nu_t \!\sim\! \mathcal{N}(0,2 \kappa^2 \sigma_w^2 I)$ guarantees persistence of excitation with high probability, \textit{i.e.}, the smallest eigenvalue of the design matrix $V_t$ scales linearly over time. Combining this result, with the confidence set construction given in Section \ref{subsec:learning}, we derive the first result. Using the first result and the fact that there exists a stabilizing neighborhood around the model parameter $\tts$, such that all the optimal linear controllers of the models within this region stabilize $\tts$, we derive the final result. Due to early improved exploration, \alg stabilizes the system dynamics after $\Tw$ samples and starts stabilizing adaptive control with only TS. Using the stabilizing controllers for fixed $\tau_0 \!=\! 2\gamma^{-1} \log(2\kappa\sqrt{2})$ time-steps, \alg decays the state magnitude and remedy possible state blow-ups in the first phase. To study the boundedness of state, define $T_r = \Tw + (n+d)\tau_0\log(n+d)$. The following shows that the state is bounded and well-controlled.

\begin{lemma}[Bounded states]\label{lem:bounded_state} 
Suppose Assumptions \ref{asm_stabil} \& \ref{asm_noise} hold. For given $\Tw$ and $T_r$, \alg controls the state such that $\|x_t\| = O((n+d)^{n+d})$ for $t\leq T_{r}$, with probability at least $1-3\delta$ and $\|x_t\| \! \leq\! (12\kappa^2\!+\! 2\kappa\sqrt{2})\gamma^{-1}\sigma_w \sqrt{2n\log(n(t\!-\!\Tw)/\delta)}$ for $T \!\geq\! t \!>\! T_{r}$, with probability at least $1-4\delta$. 
\end{lemma}

The proof is given in Appendix \ref{apx:bounded_state}, but here we provide a proof sketch. To bound the state for $t\leq T_r$, we show that deploying the same policy for $\tau_0$ time-steps in the first phase maintains a well-controlled state except $n+d$ time-steps, under the high probability event of $\hat{E}_t \cap \Tilde{E}_t$. Moreover, we show that this slow policy change prevents further state blow-ups due to non-contractive system dynamics in stabilizable systems. To bound the state for $t > T_r$, we show that, with the given choice of $\tau_0$, all the controllers during the stabilizing TS phase halves the magnitude of the state at the end of their control period. Thus, we prove that after $(n\!+\!d)\log(n\!+\!d)$ policy updates the state is well-controlled and brought to an equilibrium as shown in Lemma \ref{lem:bounded_state}. This result shows that the joint event $E_t = \hat{E}_t \cap \Tilde{E}_t \cap \Bar{E}_t$ holds with probability at least $1-4\delta$ for all $t\leq T$. 

Conditioned on this event, we individually analyze the regret terms individually~(Appendix \ref{apx:regret_anlz}). We show that with probability at least $1-\delta$, $R_{\Tw}^{\text{exp}}$ yields $\OO((n+d)^{n+d}\Tw)$ regret due to isotropic perturbations. $R_{T}^{\text{RLS}}$ and $R_{T}^{\text{mart}}$ are $\OO((n+d)^{n+d} \sqrt{T_r} + \poly(n,d) \sqrt{T-T_r})$ with probability at least $1-\delta$ due to standard arguments based on the event $E_T$. More importantly, conditioned on the event $E_T$, we prove that $R_{T}^{\text{gap}} = \OO((n\!+\!d)^{n+d} \sqrt{T_r} \!+\! \text{poly}(n,d) \sqrt{T\!-\!T_r})$ with probability at least $1-2\delta$, and $R_{T}^{\text{TS}} = \OO(n \Tw \!+\! \text{poly}(n,d) \sqrt{T\!-\!\Tw}) $ with probability at least $1-2\delta$, whose analyses require several novel fundamental results.

To bound on $R_{T}^{\text{gap}}$, we extend the results in \citet{abeille2018improved} to multidimensional stabilizable LQRs and incorporate the slow update rule and the early improved exploration. We show that while \alg enjoys well-controlled state with polynomial dimension dependency on regret due to slow policy updates, it also maintains the desirable $\OO(\sqrt{T})$ regret of frequent updates with only a constant $\tau_0$ scaling. As discussed in Section \ref{subsec:prior}, bounding $R_{T}^{\text{TS}}$ requires selecting optimistic models with constant probability, which has been an open problem in the literature for multidimensional systems. In this work, we provide a solution to this problem and show that TS indeed selects optimistic model parameters with a constant probability for multidimensional LQRs. The precise statement of this result and its proof outline are given in Section \ref{sec:overview}. Leveraging this result, we derive the upper bound on $R_{T}^{\text{TS}}$. Combining all these terms yields the regret upper bound of \alg given in Theorem \ref{reg:exp_s}.

%% file: sections/proof_overview.tex
In this section, we provide the precise statement that the probability of sampling an optimistic parameter is lower bounded by a fixed constant with high probability. Then we give the proof outline with the main steps.  The complete proof with the intermediate results are given in Appendix~\ref{apx:proof}. 

\begin{theorem}[Optimistic probability]\label{thm:optimistic_prob}
Let $\clf{F}_t^{\text{cnt}} \defeq \sigma(F_{t-1}, x_t)$ be the information available to the controller up to time $t$. Denote the optimistic set by $\S^{\opt} \!\defeq\! \cl{ \tt \!\in\! \R^{(n+d)\times n} \;\big| \; J(\tt) \!\leq\! J(\tts)}$.  If $T_w = cn^2(\sqrt{T}\log{T})^{1+o(1)}$ for a constant $c>0$, then under the event $E_T$ for large enough $T$, we have that $ p_t^{\opt} \defeq \Pr \cl{ \ttt_t \in \S^{\text{opt}} \cond \clf{F}_t^{\text{cnt}}, \hat{E}_t} \geq \frac{Q(1)}{1+o(1)}$ for any $T_r < t \leq T$. Furthermore, if the closed-loop matrix, $A_{c,*} = A_* + B_* K_*$, is non-singular, then the bound above still holds when $T_w = c(\log{T})^{1+o(1)}$ for a constant $c>0$.
\end{theorem}

\subsection{Surrogate Set Definition}
First, we define a surrogate subset $\S^{\text{surr}}$ to the optimistic set $\S^{\opt}$. The construction of $\S^{\text{surr}}$ is important as the geometry of $\S^{\opt}$ is complicated to study due to \eqref{DARE} that controls the spectrum of $P(\tt)$. \looseness=-1
\begin{lemma}[Surrogate set] \label{thm:subset}
Let $J(\tt, K) \!\!\defeq\!\! \tr\pr{ (Q\!+\!K^\tp R K)\Sigma(\tt, K) }$ be the expected average cost of controlling a system $\tt \!\in\! \mathcal{S} $ by a fixed stabilizing control policy $K \!\in\! \R^{d \times n}$ where $\Sigma(\tt, K) \!\defeq\! \lim_{t \goesto \infty} \E\br{x_t x_t^\tp}$ is the covariance of the state. The following surrogate set is a subset of $\S^{\opt}$:
\begin{align}\label{def:surrogate}
    \S^{\text{surr}} \!\defeq\! \cl{ \tt \!=\! (A,\,B)^\tp \!\in\! \R^{(n+d)\times n} \!\;\big| \!\; J(\tt, K(\tts)) \!\leq\! J(\tts, K(\tts)) \!=\! J(\tts)} \subset \S^{\opt}.
\end{align}
\end{lemma}
Note that $\Sigma(\tt, K)$ satisfies the Lyapunov equation $\Sigma(\tt, K) \!-\! \tt^\tp H_K \Sigma(\tt, K) H_K^\tp \tt \!=\! \sigma_w^2 I$, where $H_K^\tp \!\defeq\! [I,\, K^\tp]$, and $ \tt^\tp H_K \!=\! A+BK$, given that $K$ stabilizes the system $\tt$. We can analytically express $\Sigma(\tt, K)$ as a converging infinite sum $\Sigma(\tt, K) \!=\! \sigma_w^2 \sum_{t=0}^{\infty} (A\!+\!BK)^t (A^\tp\!+\!K^\tp B^\tp)^t$~\citep{kailath2000linear}. Using the properties of the trace operator, one can write $J(\tt, K(\tts)) \!=\! L(\tt^\tp H_*)$, where $L(A_c) \!\defeq\! \sigma_w^2 \sum_{t=0}^{\infty} \Norm[Q_*]{A_c^t}^2$ for any stable matrix $A_c$, $Q_* \! \defeq \! Q\!+\!K(\tts)^\tp R K(\tts)$, and $H_*^\tp \defeq [I, K(\tt_*)^\tp]$. Therefore, we can lower bound the probability of being optimistic as
\begin{align}
    p_t^{\opt} &\geq \Pr \cl{ \ttt_t \in \S^{\text{surr}} \cond \F_t^{\text{cnt}}, \hat{E}_t}  = \Pr \cl{ L(\ttt_t^\tp H_*) \leq  L(\tts^\tp H_*) \cond \clf{F}_t^{\text{cnt}}, \hat{E}_t} \nonumber \\
    &\geq \min_{\hat{\Theta} \in \mathcal{E}_t^{\text{RLS}}} \Pr_{t}\{   L(\hat{\Theta}^\tp H_* + \eta^\tp \beta_t V_t^{-\frac{1}{2}}H_*) \leq  L( {\Theta}_*^\tp H_*)\}  \label{worst_thetahat}\\
    & = \min_{\hat{\Theta} \in \mathcal{E}_t^{\text{RLS}}} \Pr_{t}\{   L(\hat{\Theta}^\tp H_* + \Xi \sqrt{F_t}) \leq  L( {\Theta}_*^\tp H_*)\} \label{whitening}
\end{align}
where $\Pr_{t}\{\cdot\} \!\defeq\! \Pr\{\cdot \,|\, \clf{F}_t^{\text{cnt}}\}, $ $F_t \!\defeq\! \beta_t^2 H_*^\tp V_t^{-1} H_*$ and $\Xi$ is a matrix of size $n\!\times\!n$ with iid $\mathcal{N}(0,1)$ entries. Here \eqref{worst_thetahat} considers the worst possible estimate within $\mathcal{E}_t^{\text{RLS}}$ and \eqref{whitening} is the whitening transformation.\looseness=-1

\subsection{Reformulation in Terms of Closed-Loop Matrix}
In the second step, we reformulate the probability of sampling optimistic parameters in terms of closed-loop system matrix $\Tilde{A}_c \!\defeq\! \ttt^\tp H_* \!=\! \Tilde{A} + \Tilde{B} K(\tts)$ of the sampled system $\ttt \!=\! (\Tilde{A},\Tilde{B})^\tp$ driven by the policy $K(\tts)$. Transitioning to the closed-loop formulation allows tighter bounds on the optimistic probability. To complete this reformulation, we need to construct an estimation confidence set for the closed-loop system matrix $\hat{A}_c\defeq \tth^\tp H_* = \hat{A} + \hat{B}K(\tts)$ of the RLS-estimated system $\tth = (\hat{A},\hat{B})^\tp$ and show that the constructed confidence set is a super set to $\mathcal{E}_t^{\text{RLS}}$.
\begin{lemma}[Closed-loop confidence]\label{thm:cl_conf}
Let $F_t(\delta) \!\defeq\! \beta_t^2(\delta) H_*^\tp V_t^{-1} H_*$. For any $t\geq 0$, define by
\begin{align}
    \mathcal{E}_t^{\text{cl}}(\delta) &\defeq \cl{ \hat{\Theta} \in \R^{(n+d)\times n} \;\big| \; \tr\br{(\tth^\tp H_* - \Theta_*^\tp H_*) F_{t}^{-1}(\delta) (\hat{\Theta}^\tp H_* - \Theta_*^\tp H_*)^\tp} \leq 1 }. \label{eqn:cl_conf}
\end{align}
the closed-loop confidence set. Then, for all times $t\!\geq\! 0$ and $\delta \!\in\! (0,1)$, we have that $\mathcal{E}_t^{\text{RLS}}(\delta)\! \subseteq\! \mathcal{E}_t^{\text{cl}}(\delta)$.
\end{lemma}
Note that the definition of $\mathcal{E}_t^{\text{cl}}(\delta)$ \textit{only} involves closed-loop matrices $\hat{A}_{c}\!\defeq\!\tth^\tp H_*$ and $A_{c,*}\!\defeq\!\tts^\tp H_*$. We can use the result of Lemma~\ref{thm:cl_conf} to reformulate the probability of sampling optimistic parameters, $\ttt = (\Tilde{A},\Tilde{B})$, as sampling optimistic closed-loop system matrices, $\Tilde{A}_c$. We bound $p_t^{\opt}$ from below as \looseness=-1
\begin{align}
    p_t^{\opt}&\geq \min_{\hat{\Theta} \in \mathcal{E}_t^{\text{cl}}} \Pr_{t}\{ L(\hat{\Theta}^\tp H_* + \Xi \sqrt{F_t}) \leq  L( A_{c,*}) \} \label{eqn:using_prop_8} \\
    &= \min_{\hat{A}_{c}\,:\, \norm[F_{t}^{-1}]{\hat{A}_{c}^\tp - {A}_{c,*}^\tp} \leq 1 } \Pr_{t}\{ L(\hat{A}_{c} + \Xi \sqrt{F_t}) \leq  L( A_{c,*})  \} \label{eqn:full_column} \\
    &= \min_{\hat{\Upsilon}\,:\, \norm[F]{\hat{\Upsilon} }\leq 1 } \Pr_{t}\{   L({A}_{c,*} + \hat{\Upsilon}  \sqrt{F_t}+ \Xi \sqrt{F_t}) \leq  L( A_{c,*})  \} \label{eqn:cl_form},
\end{align}
where \eqref{eqn:using_prop_8} is due to Lemma~\ref{thm:cl_conf} and \eqref{eqn:full_column} follows from the fact that $H_*$ has full column rank. Observe that, in equation \eqref{eqn:cl_form}, $\hat{\Upsilon}$ is a unit Frobenius norm matrix of size $n\times n$ and the term ${A}_{c,*} + \hat{\Upsilon}  \sqrt{F_t}$ accounts for the confidence ellipsoid for the estimated closed-loop matrix, $\hat{A}_{c}$. The event in \eqref{eqn:cl_form} corresponds to finding the closed-loop matrix, ${A}_{c,*} + (\Xi + \hat{\Upsilon}) \sqrt{F_t}$ of the TS sampled system in the sublevel manifold $\clf{M}_* \defeq \cl{A_c \in \M_n \;|\; L(A_c) \leq L(A_{c,*})}$ as illustrated in Figure~\ref{fig:manifolds}.

\subsection{Local Geometry of Optimistic Set under Perturbations}
Next, we further simplify the form of the probability in \eqref{eqn:cl_form} by exploiting the local geometric structure of the function $L: A_c \mapsto \sigma_w^2\sum_{t=0}^{\infty} \Norm[Q_*]{A_c^t}^2$ defined over the set of (Schur-)stable matrices, $\clf{M}_{\text{Schur}} \!\defeq\! \cl{A_c \!\in\! \M_n \; | \; \rho(A_c) \!<\! 1}$. The following lemma characterizes perturbative properties of $L$. \looseness=-1
\begin{lemma}[Perturbations]\label{thm:taylor}
The function $L: \clf{M}_{\text{Schur}} \to \R_{+}$ defined as $L(A_c) = \sigma_w^2 \sum_{t=0}^{\infty} \Norm[Q_*]{A_c^t}^2$ is smooth in its domain. For any $A_c \in \clf{M}_{\text{Schur}}$, there exists $\epsilon > 0$ such that for any perturbation $\norm[F]{G} \leq \epsilon $, the function $L$ admits a quadratic Taylor expansion as
\begin{align}
    L(A_c + G) = L(A_c) + \nabla L (A_c) \bullet G + \frac{1}{2} G \bullet \clf{H}_{A_c+sG}(G) \label{eqn:taylor}
\end{align}
for an $s\in[0,1]$ where $\clf{H}_{A_c}: \M_n \to \M_n$ is the Hessian operator evaluated at a point $A_c\in \clf{M}_{\text{Schur}}$. In particular, we have that $\nabla L (A_{c_*}) = 2P(\Theta_*) A_{c,*} \Sigma_*$. Furthermore, there exists a constant $r >0$ such that $\abs{G \bullet \clf{H}_{A_{c}+sG}(G)} \leq r \norm[F]{G}^2$ for any $s\in[0,1]$ and $\norm[F]{G}\leq \epsilon$.
\end{lemma}
Lemma \ref{thm:taylor} guarantees that if a perturbation is sufficiently small, the perturbed function can be locally expressed as a quadratic function of the perturbation. Since the set of stable matrices, $\clf{M}_{\text{Schur}}$, is globally non-convex and Taylor's theorem only holds in convex domains, we restrict the perturbations in a ball of radius $\epsilon >0$. The fact that there is a neighborhood of stable matrices around a matrix $A_c$ enables us to apply Taylor's theorem in this neighborhood. 

Given the optimal closed-loop system matrix $A_{c,*}$, let $\epsilon_* >0$ be chosen such that the expansion in \eqref{eqn:taylor} holds for perturbations $\norm[F]{G} \leq \epsilon_*$ around $A_{c,*}$. Denote the perturbation due to Thompson sampling and estimation error as $G_t = (\Xi +\hat{\Upsilon})\sqrt{F_t}$ and let $\norm[F]{G_t} \leq \epsilon_*$. Then, we can write
\begin{align}
    L(A_{c,*} + G_t) &= L(A_{c,*}) +\nabla L (A_{c,*})\bullet G_t + \frac{1}{2} G_t \bullet \clf{H}_{A_{c,*}+sG_t}(G_t) \nn\\
    &\leq L(A_{c,*}) + \nabla L (A_{c,*})\bullet G_t + \frac{r_*}{2}\norm[F]{G_t}^2 \label{eqn:quad_aprx}
\end{align}
where $r_* \!>\! 0$ is a constant due to Lemma~\ref{thm:taylor}. Using \eqref{eqn:quad_aprx}, we have the following lower bound on \eqref{eqn:cl_form},
\begin{align}
    p_t^{\opt} \!\geq\! \min_{\hat{\Upsilon}\,:\, \norm[F]{\hat{\Upsilon} }\leq 1 } \!\!\Pr_{t} \Big\{
		\frac{r_*}{2}\norm[F]{(\Xi \!+\!\hat{\Upsilon})F_t^{\frac{1}{2}}}^2 \!+\! \nabla L_*\!\bullet (\Xi \!+\!\hat{\Upsilon})F_t^{\frac{1}{2}}\! \leq 0, \text{ and }  \norm[F]{(\Xi \!+\!\hat{\Upsilon})F_t^{\frac{1}{2}}}\leq \epsilon_*
	\Big\}, \label{eqn:quad_prob}
\end{align}
where $\nabla L_* \!\defeq\!  \nabla L (A_{c,*})$. The event in \eqref{eqn:quad_prob} corresponds to finding ${A}_{c,*} + (\Xi + \hat{\Upsilon}) \sqrt{F_t}$ at the intersection of the stable ball $\clf{B}_* \!\defeq \!\cl{A_c \in \M_n \;|\; \norm[F]{A_c - A_{c,*}} \leq \epsilon_*}$ and the sublevel manifold $\clf{M}_{*}^{\text{qd}} \defeq \cl{A_c \in \M_n \;|\; \norm[F]{A_c - A_{c,*} + r_*^{-1}\nabla L_*} \leq \norm[F]{r_*^{-1} \nabla L_*} }$ as illustrated in Figure~\ref{fig:manifolds}.
\\ \\
The intersection $\clf{M}_{*}^{\text{qd}} \cap \clf{B}_* \subset \clf{M}_{*}$ serves as another surrogate to sublevel manifold $\clf{M}_{*}$. Switching to the new surrogate $\clf{M}_{*}^{\text{qd}}$ helps us overcome the issue of working with intractable and complicated geometry of $\clf{M}_{*}$ due to infinite sum in $L(A_{\text{c}})$. We can utilize techniques relating to Gaussian probabilities as the geometry of $\clf{M}_{*}^{\text{qd}}$ is described by a quadratic form.
\vspace{-0.5cm}
\begin{SCfigure}[1][htb]
\centering
\vspace{-0.4cm}
\hspace{-0.4cm}
\resizebox{0.51\linewidth}{!}{

\tikzset{every picture/.style={line width=0.75pt}} 

\begin{tikzpicture}[x=0.75pt,y=0.75pt,yscale=-1,xscale=1]

\draw  [fill={rgb, 255:red, 155; green, 155; blue, 155 }  ,fill opacity=1 ][line width=1.5]  (166.46,91.8) .. controls (200.02,211.39) and (200.33,201.68) .. (89.31,166.16) .. controls (-21.72,130.64) and (-30.64,438.37) .. (82.75,400.12) .. controls (196.14,361.86) and (196.2,359.92) .. (154.12,477.19) .. controls (112.05,594.45) and (415.69,598.92) .. (380.16,482.23) .. controls (344.64,365.53) and (346.51,365.54) .. (454.71,402.02) .. controls (562.9,438.5) and (573.66,131.75) .. (461.2,170.01) .. controls (348.74,208.27) and (349.71,207.3) .. (388.92,91.96) .. controls (428.13,-23.37) and (132.91,-27.8) .. (166.46,91.8) -- cycle ;
\draw  [fill={rgb, 255:red, 255; green, 0; blue, 0 }  ,fill opacity=0.4 ][line width=1.5]  (256.4,323.09) .. controls (231.17,249.19) and (273.04,168.18) .. (349.92,142.14) .. controls (426.79,116.11) and (509.57,154.91) .. (534.8,228.81) .. controls (560.03,302.71) and (518.16,383.73) .. (441.28,409.76) .. controls (364.41,435.8) and (281.63,397) .. (256.4,323.09) -- cycle ;
\draw  [fill={rgb, 255:red, 0; green, 0; blue, 255 }  ,fill opacity=0.4 ][dash pattern={on 1.69pt off 2.76pt}][line width=1.5]  (532.36,305.33) .. controls (491.66,305.12) and (459.76,270.68) .. (461.11,228.42) .. controls (462.45,186.16) and (496.53,152.07) .. (537.22,152.28) .. controls (577.92,152.49) and (609.82,186.93) .. (608.48,229.19) .. controls (607.13,271.46) and (573.05,305.55) .. (532.36,305.33) -- cycle ;
\draw [line width=1.5]    (534.79,228.81) -- (565.74,218.87) ;
\draw [shift={(569.55,217.65)}, rotate = 162.21] [fill={rgb, 255:red, 0; green, 0; blue, 0 }  ][line width=0.08]  [draw opacity=0] (11.61,-5.58) -- (0,0) -- (11.61,5.58) -- cycle    ;
\draw [line width=3]    (486.26,71.28) -- (583.32,386.33) ;
\draw [line width=1.5]  [dash pattern={on 3.75pt off 3pt}]  (534.79,228.81) .. controls (579.21,292.58) and (612.97,210.2) .. (611.11,294.68) ;
\draw [shift={(611,298.66)}, rotate = 271.89] [fill={rgb, 255:red, 0; green, 0; blue, 0 }  ][line width=0.08]  [draw opacity=0] (13.4,-6.43) -- (0,0) -- (13.4,6.44) -- (8.9,0) -- cycle    ;
\draw    (268.12,285.46) ;
\draw [shift={(268.12,285.46)}, rotate = 0] [color={rgb, 255:red, 0; green, 0; blue, 0 }  ][fill={rgb, 255:red, 0; green, 0; blue, 0 }  ][line width=0.75]      (0, 0) circle [x radius= 3.35, y radius= 3.35]   ;
\draw [shift={(268.12,285.46)}, rotate = 0] [color={rgb, 255:red, 0; green, 0; blue, 0 }  ][fill={rgb, 255:red, 0; green, 0; blue, 0 }  ][line width=0.75]      (0, 0) circle [x radius= 3.35, y radius= 3.35]   ;
\draw  [fill={rgb, 255:red, 255; green, 0; blue, 255 }  ,fill opacity=0.4 ][line width=1.5]  (534.8,228.81) .. controls (543.48,254.25) and (544.21,280.53) .. (538.31,305.12) .. controls (536.34,305.27) and (534.36,305.34) .. (532.36,305.33) .. controls (491.66,305.12) and (459.76,270.68) .. (461.11,228.42) .. controls (461.9,203.44) and (474.13,181.31) .. (492.41,167.44) .. controls (511.41,183.2) and (526.34,204.03) .. (534.8,228.81) -- cycle ;
\draw [line width=1.5]    (534.79,228.81) -- (399.39,274.67) ;
\draw [shift={(395.6,275.95)}, rotate = 341.29] [fill={rgb, 255:red, 0; green, 0; blue, 0 }  ][line width=0.08]  [draw opacity=0] (11.61,-5.58) -- (0,0) -- (11.61,5.58) -- cycle    ;
\draw [shift={(534.79,228.81)}, rotate = 161.29] [color={rgb, 255:red, 0; green, 0; blue, 0 }  ][fill={rgb, 255:red, 0; green, 0; blue, 0 }  ][line width=1.5]      (0, 0) circle [x radius= 4.36, y radius= 4.36]   ;
\draw [line width=1.5]  [dash pattern={on 3.75pt off 3pt}]  (511.33,259.79) .. controls (418.93,370.54) and (525.82,319.7) .. (522.15,422.49) ;
\draw [shift={(522,425.66)}, rotate = 273.21] [fill={rgb, 255:red, 0; green, 0; blue, 0 }  ][line width=0.08]  [draw opacity=0] (13.4,-6.43) -- (0,0) -- (13.4,6.44) -- (8.9,0) -- cycle    ;

\draw (184.87,60.29) node [anchor=west] [inner sep=0.75pt]  [font=\Huge] [align=left] {$\displaystyle \clf{M}_{*}$};
\draw (280.57,201.88) node [anchor=west] [inner sep=0.75pt]  [font=\Huge] [align=left] {$\displaystyle \clf{M}_{*}^{\text{qd}}$};
\draw (536.13,171.74) node [anchor=west] [inner sep=0.75pt]  [font=\Huge] [align=left] {$\displaystyle \mathcal{B}_{*}$};
\draw (588.52,314.86) node [anchor=west] [inner sep=0.75pt]  [font=\LARGE] [align=left] {$\displaystyle A_{c,*}$};
\draw (571.46,217.06) node [anchor=west] [inner sep=0.75pt]  [font=\large,rotate=-342.78] [align=left] {$\displaystyle \nabla L_{*}$};
\draw (396.74,293.13) node [anchor=west] [inner sep=0.75pt]  [font=\large,rotate=-342.78] [align=left] {$\displaystyle -r_*^{-1} \nabla L_{*}$};
\draw (448.55,46.62) node [anchor=west] [inner sep=0.75pt]  [font=\Huge] [align=left] {$\displaystyle T_{A_{c,*}}\clf{M}_{*}$};
\draw (448.49,443.49) node [anchor=west] [inner sep=0.75pt]  [font=\Huge] [align=left] {$\displaystyle \clf{M}_{*}^{\text{qd}} \cap \mathcal{B}_{*}$};
\draw (259.31,302.82) node [anchor=west] [inner sep=0.75pt]  [font=\LARGE] [align=left] {$\displaystyle O$};

\end{tikzpicture}
}
\hspace{-0.1cm}

\caption{ A visual representation of sublevel manifold $\clf{M}_*$. $O$ is the origin and $A_{c,*}$ is the optimal closed-loop system matrix. $T_{A_{c,*}}\clf{M}_*$ is the tangent space to the manifold $\clf{M}_*$ at the point $A_{c,*}$ and $\nabla L_*$  is the Jacobian of the function $L$ at $A_{c,*}$. $\clf{M}_{*}^{\text{qd}}$ is the sublevel manifold of the quadratic approximation to $L$ and $\clf{B}_*$ is a small ball of stable matrices around $A_{c,*}$. The intersection $\clf{M}_{*}^{\text{qd}} \cap \clf{B}_*$ is a subset of $\clf{M}_*$.}\label{fig:manifolds}

\end{SCfigure}
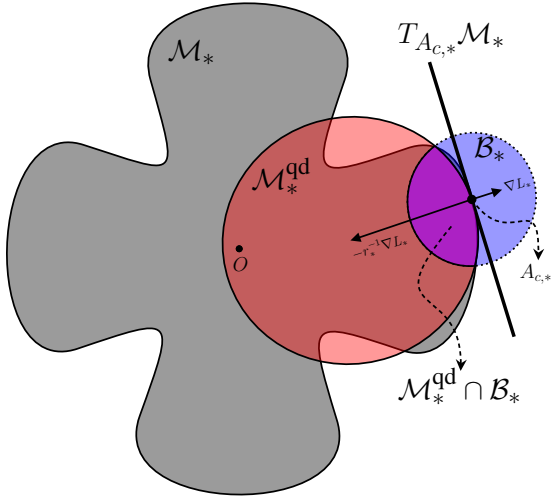

\subsection{Final Bound}
Equipped with the preceding results, we can bound the optimism probability tractably from below by the probability of a TS sampled closed-loop system matrix lying inside the intersection of two balls $\clf{M}_{*}^{\text{qd}} \cap \clf{B}_*$ as given in \eqref{eqn:quad_prob}. By bounding the weighted Frobenius norms in \eqref{eqn:quad_prob} from above by  $\lambda_{\max,t}$, the maximum eigenvalue of $F_t$, and normalizing the matrix $\nabla L_*\sqrt{F_t}$, we can write
\begin{align}
    p_t^{\opt} \! &\!\geq\! \min_{\norm[F]{\hat{\Upsilon} }\leq 1 } \!\Pr_{t}\! \cl{\frac{r_*}{2}\lambda_{\max,t}\norm[F]{\Xi \!+\!\hat{\Upsilon}}^2 \!+\! (\nabla L_*\sqrt{F_t})\bullet (\Xi \!+\!\hat{\Upsilon}) \leq  0, \text{ and } \lambda_{\max,t}\norm[F]{\Xi \!+\!\hat{\Upsilon}}^2\leq \epsilon_*^2 } \nonumber \\
	&\!=\! \min_{\norm[F]{\hat{\Upsilon} }\leq 1 } \! \Pr_{t}\! \cl{\frac{(\nabla L_*F_t^{\nicefrac{1}{2}})\!\bullet\! (\Xi \!+\!\hat{\Upsilon})}{\norm[F]{\nabla L_*F_t^{\nicefrac{1}{2}}}}\! \leq \! \frac{-\lambda_{\max,t}r_*\norm[F]{\Xi \!+\!\hat{\Upsilon}}^2}{2\norm[F]{\nabla L_*F_t^{\nicefrac{1}{2}}}}, \text{ and } \norm[F]{\Xi \!+\!\hat{\Upsilon}}^2\!\leq\! \frac{\epsilon_*^2}{\lambda_{\max,t}}}. \label{eqn:quad_prob2}
\end{align}
Observe that the inner product $(\nabla L_*F_t^{\nicefrac{1}{2}})\bullet \hat{\Upsilon} $ is maximized by $ \Upsilon_{\# }\!\defeq\! \frac{(\nabla L_*F_t^{\nicefrac{1}{2}})}{\norm[F]{\nabla L_*F_t^{\nicefrac{1}{2}}}}$ subject to $\norm[F]{\hat{\Upsilon} }\!\leq\! 1$. Since the probability distribution of $\norm[F]{\Xi \!+\!\hat{\Upsilon}}^2$ is invariant under orthogonal transformation of $\Xi$ and $\hat{\Upsilon}$, \eqref{eqn:quad_prob2} also attains its minimum at $\Upsilon_{\# }$. Thus, we can rewrite \eqref{eqn:quad_prob2} as \looseness=-1
\begin{align}
    p_t^{\opt} &\geq \Pr_{t}\cl{\frac{(\nabla L_*F_t^{\nicefrac{1}{2}})\!\bullet\! \Xi }{\norm[F]{\nabla L_*F_t^{\nicefrac{1}{2}}}} \!+\!1 \! \leq \! \frac{-\lambda_{\max,t}r_*}{2\norm[F]{\nabla L_*F_t^{\nicefrac{1}{2}}}} \Norm[F]{\Xi + \Upsilon_{\# }}^2, \text{ and } \Norm[F]{\Xi +\Upsilon_{\# }}^2\leq \frac{\epsilon_*^2}{\lambda_{\max,t}} } \nonumber \\
	&=\Pr_{t}\cl{\xi + 1 \leq  - \frac{\lambda_{\max,t}r_*}{2\norm[F]{\nabla L_*F_t^{\nicefrac{1}{2}}}}\pr{(\xi+1)^2 + X }, \text{ and } (\xi+1)^2 + X \leq \frac{\epsilon_*^2}{\lambda_{\max,t}} }, \label{eqn:quad_simple}
\end{align}
where $\xi \!\sampled\! \normal(0,1)$ and $X \!\sampled\! \chi^2_{n^2-1}$ are independent standard normal and chi-squared distributions, and \eqref{eqn:quad_simple} is derived by rotating $\Xi$ so that its first element is along the direction of $\nabla L_*F_t^{\nicefrac{1}{2}}$. We use the following lemma to characterize the eigenvalues of $F_t$ and control the lower bound~\eqref{eqn:quad_simple} on $p_t^{\opt}$.\looseness=-1 
\begin{lemma}[Bounded eigenvalues]\label{thm:bounded_F}
Suppose $\Tw \!=\! O((\sqrt{T})^{1+o(1)})$. Denote the minimum and maximum eigenvalues of $F_t$ by $\lambda_{\min,t}$ and $\lambda_{\max,t}$, respectively. Under the event $E_T$, for large enough $T$, we have that $\lambda_{\max,t} \leq C \frac{\log{T}}{\Tw}$ and $\frac{\lambda_{\max,t}}{\lambda_{\min,t}} \leq C \frac{T\log{T}}{\Tw}$ for any $T_r < t \leq T$ for a constant $C = \poly(n,d,\log(1/\delta))$.
\end{lemma}
 Lemma~\ref{thm:bounded_F} states that maximum eigenvalue and the condition number of $F_t$ are controlled inversely by the length of initial exploration phase $\Tw$ and proportionally by $\log{T}$ and $T\log{T}$ given that exploration time is bounded by a certain amount. The length of initial exploration $\Tw$ relative to the horizon $T$ is critical in guaranteeing asymptotically constant optimistic probability $p_t^{\opt}$. Although more lengthy initial exploration will lead to better convergence to constant optimistic probability, it also incurs higher asymptotic regret due to linear scaling of exploration regret with $T_w$.

Using the relation $\norm[F]{\nabla L_*F_t^{\frac{1}{2}}} \!\geq\!  \max(\sigma_{\min,*}\norm[F]{F_t^{\frac{1}{2}}},\, {\lambda_{\min,t}^{\frac{1}{2}}}\norm[F]{\nabla L_*} )$ where $\sigma_{\min,*}$ is the minimum singular value of $\nabla L_*$, we can further bound \eqref{eqn:quad_simple} from below. From Lemma \ref{thm:taylor}, we can write $\nabla L_* \!=\! 2 P(\Theta_*) A_{c,*} \Sigma_*$ where $P(\Theta_*)\!\psdg\! 0$ is the solution to the DARE in \eqref{DARE} and $\Sigma_*= \Sigma(\Theta_*, K_*) \psdg 0$ is the stationary state covariance matrix. Notice that the minimum singular value of $\nabla L_*$ is positive (\ie $\nabla L_*$ is full-rank) if and only if the closed-loop system matrix, $A_{c,*}$, is non-singular.\looseness=-1 

In general, $A_{c,*}$ can be singular. Assuming that $\Tw \!=\! O((\sqrt{T})^{1+o(1)})$, under the event $E_T$, we can use $\norm[F]{\nabla L_*F_t^{\frac{1}{2}}} \!\geq\!  \sqrt{\lambda_{\min,t}}\norm[F]{\nabla L_*}$ to obtain the following lower bound on $p_t^{\opt}$ for $T_r\!<\!t\!\leq\! T$:
\begin{align}
     p_t^{\opt} &\geq\Pr_{t}\bigg\{ \xi + 1 \leq  - \frac{\sqrt{\lambda_{\max,t}}}{2 \rho_*}\sqrt{\frac{\lambda_{\max,t}}{\lambda_{\min,t}}}\pr{(\xi+1)^2 + X }, \text{ and } (\xi+1)^2 + X \leq \frac{\epsilon_*^2}{\lambda_{\max,t}} \bigg\},
     \nonumber\\
     &\geq \Pr \bigg\{\xi + 1 \leq  - \frac{C}{2 \rho_*}\frac{\sqrt{T}\log{T}}{\Tw}\pr{(\xi+1)^2 + X }, \text{ and } (\xi+1)^2 + X \leq \frac{\epsilon_*^2 T_w}{C\log T} \bigg\}, \nn
\end{align}
where {$\rho_* \defeq \norm[F]{r_*^{-1}\nabla L_*}$}. Choosing the exploration time as $T_w = \omega(\sqrt{T}\log{T})$ makes the coefficients $\frac{\sqrt{T}\log{T}}{\Tw} = o(1)$ to be very small and $\frac{T_w}{\log T}$ to be very large, leading to constant lower bound on limiting optimistic probability $\liminf_{T \rightarrow \infty } p_T^{\opt} \geq \Pr\{\xi+1\leq 0\}\eqdef Q(1)$.

On the other hand, if  $A_{c,*}$ is non-singular, then we can use the alternative bound $\norm[F]{\nabla L_*\sqrt{F_t}} \!\geq\!  \sigma_{\min,*}\norm[F]{\sqrt{F_t}} \!\geq\! \sigma_{\min,*}\sqrt{\lambda_{\max,t}}$ to obtain the following lower bound for $T_r\!<\!t\!\leq\! T$:
\begin{align}
    p_t^{\opt} &\geq \Pr_{t}\cl{\xi + 1 \leq  - \frac{\sqrt{\lambda_{\max,t}}}{2\sigma_{\min,*} }\pr{(\xi+1)^2 + X }, \text{ and } (\xi+1)^2 + X \leq \frac{\epsilon_*^2}{\lambda_{\max,t}} }, \nonumber \\
    &\geq \Pr \cl{\xi + 1 \leq  - \frac{\sqrt{C}}{2 \sigma_{\min,*}}\sqrt{\frac{\log T}{ T_w}}\pr{(\xi+1)^2 + X }, \text{ and } (\xi+1)^2 + X \leq \frac{\epsilon_*^2 T_w}{C\log T} }. \nn
\end{align}
Similarly, choosing the exploration time as $T_w = \omega(\log{T})$ makes the coefficients $\sqrt{\frac{\log T}{ T_w}} = o(1)$ to be very small and $\frac{T_w}{\log T}=\omega(1)$ to be very large, leading to constant lower bound on limiting optimistic probability $\liminf_{T \rightarrow \infty } p_T^{\opt} \geq Q(1)$.

In both cases, the optimistic probability achieves a constant lower bound for large enough $T$ as $p_T^{\opt} \geq Q(1)(1 + o(1))^{-1}$. This result can be interpreted in a geometric way as follows. As the time passes, the estimates of the system become more accurate in the sense that the confidence region of the estimate shrinks very quickly as controlled by the eigenvalues of $F_t$. Similarly, the high-probability region of TS samples also shrink very fast controlled by the covariance matrix $F_t$. Therefore, for large enough $T$, the confidence region of the model estimate and the high-probability region of TS samples get significantly smaller compared to the surrogate optimistic set $\clf{M}_{*}^{\text{qd}}\cap \clf{B}_*$. This size difference effectively reduces the probability of finding a sampled system in $\clf{M}_{*}^{\text{qd}}\cap \clf{B}_*$ to the probability of finding a sampled system in the half-space separated by the tangent space $T_{A_{c,*}} \clf{M}_{*}$.\looseness=-1

%% file: sections/numerical.tex
\begin{table}[h!]
\centering
\caption{Regret and Maximum State Norm After 200 Time Steps in Boeing 747 Flight Control}
\label{table:regret_boeing}
\resizebox{!}{0.09\textwidth}{
 \begin{tabular}{c c c c| c c c }
 \textbf{Algorithm} &  \textbf{ \shortstack{Average \\ Regret}}  &  \textbf{Top 95\% } &   \textbf{Top 90\%} &  \textbf{ \shortstack{Average \\ $\max\|x\|_2$}} & \textbf{Top 95\% } &   \textbf{Top 90\%} \\
  \hline
 \alg & $\bf{4.58 \times 10^7}$   & $\bf{1.43 \times 10^5}$  & $\bf{9.49 \times 10^4}$ & $\bf{1.23\times10^3}$ & $\bf{1.07\times10^2}$ & $\bf{9.77\times10^1}$  \\
  
 StabL & $1.34 \times 10^4$   & $1.05 \times 10^3$  & $9.60 \times 10^3$ & $3.38 \times 10^1$   & $3.14 \times 10^1$  & $2.98 \times 10^1$  \\

 OFULQ  & $1.47 \times 10^{8}$ & $4.19 \times 10^6$  &  $9.89 \times 10^5$ & $1.62 \times 10^{3}$ & $5.21 \times 10^2$ & $2.78 \times 10^2$ \\

TS-LQR  & $5.63\times10^{11}$ &  $3.07\times10^7$ & $5.33 \times 10^6$ & $6.26\times10^4$  & $ 1.08\times 10^3$  & $6.39 \times 10^2$
\end{tabular}
}
\end{table}

Finally, we evaluate the performance of \alg in longitudinal flight control of Boeing 747 with linearized dynamics~\citep{747model}. We compare \alg with three adaptive control algorithms in literature that do not require an initial stabilizing policy: 
\begin{itemize}
    \item[(i)] OFULQ of \citet{abbasi2011lqr};
    \item[(ii)]TS-LQR of \citet{abeille2018improved};
    \item[(iii)]StabL of \citet{lale2022reinforcement}.
\end{itemize}
We perform 200 independent runs for 200 time-steps for each algorithm and report their average, top $95\%$ and top $90\%$ regret and maximum state norm performances. Note that, since optimistic control design is computationally intractable, we use projected gradient descent to \textit{heuristically} find optimistic models in OFULQ and StabL. For fair comparison, we also adopt slow policy updates in OFULQ and TS-LQR and report the best results of each algorithm. Further details are in Appendix \ref{apx:simulation}. The results are presented in Table \ref{table:regret_boeing}. Notice that \alg achieves the second best performance after StabL. As expected, StabL outperforms \alg since it performs much heavier computations to find the optimistic controller in the confidence set, whereas \alg samples optimistic parameters only with some fixed probability. However, TSAC compares favorably against both OFULQ and TS-LQR, making it the best performing computationally efficient algorithm.

%% file: sections/discussion.tex
We present the first efficient adaptive control algorithm, \alg, that attains optimal regret of $\Tilde{O}(\sqrt{T})$ in stabilizable LQRs without an initial stabilizing policy. We design \alg to quickly stabilize the system and avoid state blow-ups via careful policy updates. Building on these design choices, the main technical contribution of this work is to show that TS samples optimistic parameters with constant probability in all LQRs, thereby resolving the conjecture in \citet{abeille2018improved}.\looseness=-1

This result highlights that a simple sampling strategy provides effective exploration to recover low-cost achieving controllers in adaptive control of LQRs which yields order optimal regret. An important future direction is to investigate whether TS achieves optimal regret in partially observable LTI systems, e.g.~\citep{lale2020logarithmic,lale2021finite}. Moreover, to obtain constant probability of sampling optimistic parameters for general LQRs, \alg requires {$T_w = \omega(\sqrt{T}\log{T})$} time-steps of improved exploration (Theorem \ref{thm:optimistic_prob}), which causes the regret to be dominated by this phase. This long exploration is avoided in LQRs with non-singular optimal closed-loop matrix, which results in regret that scales polynominally in system dimensions (Theorem \ref{reg:exp_s}). It remains an open problem whether this polynomial dimension dependency in regret can be achieved via TS in general LQRs.

%% file: appendix/0organization.tex
In Appendix \ref{apx:notation}, we provide the notation tables for the paper. In Appendix \ref{apx:stabilization}, we provide the system identification and stabilization guarantees of \alg. In particular, we give the proof of Lemma \ref{lem:2norm_bound} and give the precise duration of the TS with improved exploration phase $\Tw$. In Appendix \ref{apx:bounded_state}, we show that under the joint event of $\hat{E}_t \cap \Tilde{E}_t$ the state stays bounded as described in Lemma \ref{lem:bounded_state} with high probability. In Appendix \ref{apx:regret_dec}, we provide the precise regret decomposition and discuss the individual terms in the regret upper bound. In Appendix \ref{apx:proof}, we provide the complete proof of Theorem \ref{thm:optimistic_prob}, as well as the intermediate results discussed in the main text. Appendix \ref{apx:regret_anlz} comprises the analysis of individual terms in the regret decomposition. In particular, Appendix \ref{subsec:reg_explore} studies $R_{T_w}^{\text{exp}}$, Appendix \ref{subsec:reg_rls} studies $R_{T}^{\text{RLS}}$, Appendix \ref{subsec:reg_mart} studies $R_{T}^{\text{mart}}$, Appendix \ref{subsec:reg_ts} considers $R_{T}^{\text{TS}} $,  Appendix \ref{subsec:reg_gap} bounds $R_{T}^{\text{gap}}$, and finally we combine these results to prove the regret upper bound of \alg in Appendix \ref{subsec:reg_all}. In Appendix \ref{apx:technical}, we give the technical theorems and lemmas used in the proofs. Finally, in Appendix \ref{apx:simulation}, we give the implementation details of all algorithms. Before proceeding the next section, we define the following high probability events which are standard in TS-based algorithms. First recall the RLS confidence ellipsoid given in Section \ref{subsec:learning}:
\begin{equation*}
    \mathcal{E}_t^{\text{RLS}}(\delta) = \{\Theta : \| \Theta-\hat{\Theta}_t\|_{V_t} \leq \beta_t(\delta) \},
\end{equation*}
for $\beta_t(\delta) = \sigma_w \sqrt{2 n \log ( (T \det(V_{t})^{1 / 2}) / (\delta \det(\mu I)^{1 / 2} ))} + \sqrt{\mu}S$. Further define \begin{equation*}
     \mathcal{E}_t^{\text{TS}}(\delta) = \{\Theta : \| \Theta-\hat{\Theta}_t\|_{V_t} \leq \upsilon_t(\delta) \},
\end{equation*}
for $\upsilon_t(\delta) = \beta_t(\delta)n\sqrt{(n+d)\log(n(n+d)/\delta)}$. Define the events
\begin{align}
    \hat{E}_t &= \{ \forall s \leq t, \tts \in \mathcal{E}_t^{\text{RLS}}(\delta) \} \\
    \Tilde{E}_t &= \{ \forall s \leq t, \ttt_s \in \mathcal{E}_t^{\text{TS}}(\delta) \}.
\end{align}
As described in Section \ref{sec:analysis}, $\hat{E}_t$ defines the event that RLS estimates $\tth_t$ concentrate around $\tts$ and $\tilde{E}_t$ defines the event that the sampled model parameter concentrates around $\tth_t$. From standard Gaussian tail bound and the self-normalized estimation error, we have that $\hat{E} \cap \Tilde{E}$ for all $t\leq T$, with probability at least $1-2\delta$. Here the time dependency dropped since $\hat{E} \coloneqq \hat{E}_T \subset \ldots \subset \hat{E}_1 $ and  $\Tilde{E} \coloneqq \tilde{E}_T \subset \ldots \subset \tilde{E}_1 $. These events will be key in providing all the technical results starting from stabilization guarantees to final regret upper bound.

%% file: appendix/notation.tex
This section contains two tables which list the notations used throughout the paper for improving readability. In particular, Table \ref{table:notation1} provides the system dependent notations and the useful notations for presenting the design of \alg. In Table \ref{table:notation2}, we present the notation used in deriving theoretical results, namely, the regret analysis and the lower bound on the probability of selecting optimistic parameters. Further details are also referenced to the related parts of the paper. 

\begin{table}
\centering
\caption{Useful Notations for the Design of \alg}

 \begin{tabular}{l | l } 
 \toprule
 \textbf{System Not.} & \textbf{Definition} \\ 
 \midrule
 $\Theta_*$ & Unknown discrete-time LTI system with dynamics of (\ref{output}); $[A_* \enskip B_*]^\tp$  \\ 
  $x_t$ & State of the system $\in  \mathbb{R}^{n}$ \\
 $u_t$ & Input to the system $\in  \mathbb{R}^{d}$  \\
 $w_t$ & Process noise as defined in Assumption \ref{asm_noise}; $\mathcal{N}(0, \sigma_w^2 I)$ \\
 $z_t$ & Stack of current state and input; $[x_t^\tp \enskip u_t^\tp]^\tp$\\
  $Q$, $R$ & Known cost matrices; $\|Q\|,\|R\| \!<\! \Bar{\alpha}$ and $\sigma_{\min}(Q), \sigma_{\min}(R) \!>\! \underline{\alpha} \!>\! 0$ \\
 $c_t$ & Quadratic cost at time $t$; $x_t^\tp Q x_t + u_t^\tp R u_t$ \\

 $\mathcal{S}$ & Set of $(\kappa, \gamma)\text{-stabilizable}$ and bounded systems that $\tts$ belongs (Assumption \ref{asm_stabil}) \\
 $P(\Theta)$ & Unique p.d. solution to DARE \eqref{DARE} for a stabilizable system $\Theta = [A \enskip B]^\tp$ \\
 
$K(\Theta)$ & Optimal controller for $\Theta$; $-(R+B^{\tp} P(\Theta) B)^{-1} B^{\tp} P(\Theta) A$ \\

$J(\tt)$ & Average expected cost of system $\tt$; $\sigma_w^2\Tr(P(\tt))$ \\
 $\kappa$ & Bound over all possible optimal controllers in $\mathcal{S}$; $\sup_{\tt\in \clf{S}}K(\tt)$ \\
 $D$ & Bound over all possible solutions to \eqref{DARE} in $\mathcal{S}$; $\Bar{\alpha} \gamma^{-1} \kappa^2(1+\kappa^2)$ \\
 
 \toprule
 \textbf{\alg Not.} &  \\ 
 \midrule
 $\tth$ &  Least squares estimate of $\tts$ using the history of inputs and states; $[\hat{A} \enskip \hat{B}]^\tp$  \\
 $\mu$ & Regularizer for least squares; set to $(1+\kappa^2)X_s^2$ \\
 $V_t$ & Regularized design matrix; $\mu I + \sum_{s=0}^{t–1} z_s z_s^\tp$ \\
 $\eta_t$ & Random matrix with iid standard normal entries used for sampling systems \\
 $\mathcal{R}_{\mathcal{S}}(\cdot)$ & Rejection sampling to make sure that sampled system belongs to $\mathcal{S}$ \\ 
 $\ttt$ & System obtained via TS;  $\mathcal{R}_{\mathcal{S}}(\tth_t + \beta_t(\delta)V_t^{-1/2}\eta_t)$\\
 $\nu_t$ & Improved exploration; $u_t = K(\ttt_t) x_t + \nu_t  $ for $\nu_t \sim \mathcal{N}(0, 2\kappa^2\sigma_w^2 I)$ \\
 \toprule
 \textbf{Quantities } &  \\ 
 \midrule
 $\delta$ & Fixed probability to define high probability events; $(0,1)$ \\
 $T$ & Time horizon \\
 $T_w$ & Duration of TS with improved exploration; defined in Theorem \ref{reg:exp_s} \\
 $X_s$ & Upper bound on state after stabilization; $\|x_t\| \leq X_s$ for $t>T_r$ w.h.p.   \\
 $S$ & Upper bound on the Frobenius norm of $\tts$ \\ 
 $\beta_t(\delta)$ & Size of the RLS confidence ellipsoid at time $t$; $\sigma_w \sqrt{2 n \log \left(\frac{\det(V_{t})^{1 / 2}}{\delta \det(\mu I)^{1 / 2} } \right) } \!+\! \sqrt{\mu}S$ \\
 $\upsilon_t(\delta)$ & Size of the sampling ellipsoid at time $t$; $\beta_t(\delta)n\sqrt{(n+d)\log(n(n+d)/\delta)}$ \\
 $\tau_0$ & Fixed duration for each sampled policy; $2\gamma^{-1} \log(2\kappa\sqrt{2})$ \\
 $T_0$ & Number of samples required to identify a stabilizing controller;~\eqref{T0def} \\
 $T_r$ & Time required to control the state w.h.p.; $ \Tw \!+\! (n\!+\!d)\tau_0\log(n+d)$ \\
 
 \bottomrule
\end{tabular}
\label{table:notation1}
\end{table}

\begin{table}
\centering
\caption{Useful Notations for the Analysis}
\resizebox{.53\paperheight}{!}{
 \begin{tabular}{l | l } 
 \toprule
 \textbf{Regret Analy.} & \\
 \midrule
 $R_T$ & Regret of \alg at until time $T$; $R_T = \sum\nolimits_{t=0}^T (c_t - J(\Theta_*))$ \\
 $\mathcal{F}_{t}$ & Filtration such that for all $t\geq 0$, $x_t, z_t$ are $\mathcal{F}_{t}$-measurable\\
 $\clf{F}_t^{\text{cnt}}$ & Information available to the controller up to time $t$; $\sigma(F_{t-1}, x_t)$ \\
 $R_{\Tw}^{\text{exp}}$ & Regret attained due to improved exploration (Appendix \ref{apx:regret_dec}) \\
 $R_{T}^{\text{RLS}}$ & Cost-to-go difference of the true and predicted next states (Appendix \ref{apx:regret_dec}) \\  
 $R_{T}^{\text{mart}}$ & Martingale with bounded difference (Appendix \ref{apx:regret_dec}) \\
 $R_{T}^{\text{TS}}$ & Difference in $J(\tts)$ and $J(\ttt)$ (Appendix \ref{apx:regret_dec}) \\
 $R_{T}^{\text{gap}}$ & Regret due to policy changes (Appendix \ref{apx:regret_dec}) \\
 $\mathcal{E}_t^{\text{RLS}}(\delta)$ & Regularized least squares confidence ellipsoid; $\{\Theta : \| \Theta-\hat{\Theta}_t\|_{V_t} \leq \beta_t(\delta) \}$\\
 $ \mathcal{E}_t^{\text{TS}}(\delta)$ & Confidence ellipsoid for sampled system; $\{\Theta : \| \Theta-\hat{\Theta}_t\|_{V_t} \leq \upsilon_t(\delta) \}$ \\
 $\hat{E}_t$ & Event of $\{ \forall s \leq t, \tts \in \mathcal{E}_t^{\text{RLS}}(\delta) \} $ \\
 $\Tilde{E}_t$ & Event of $\{ \forall s \leq t, \ttt_s \in \mathcal{E}_t^{\text{TS}}(\delta) \}$ \\
 $\Bar{E}_t$ & Event of $\{ \forall t \leq T_r, \|x_t\| \!\leq\! c'(n+d)^{n+d} \text{ and } \forall t>T_r, \|x_t\| \! \leq\! X_s \}$ \\
 $E_t$ & $E_t = \hat{E}_t \cap \Tilde{E}_t \cap \Bar{E}_t$ \\

 \toprule
 \textbf{Optimism Analy.} &  \\ 
 \midrule
 $\S^{\opt}$ & Optimistic set; $\cl{ \tt \!=\! (A,\,B)^\tp \!\in\! \R^{(n+d)\times n} \;\big| \; J(\tt) \!\leq\! J(\tts)}$\\
 $p_t^{\opt}$ & Probability of selecting optimistic system; $\Pr \cl{ \ttt_t \in \S^{\text{opt}} \cond \clf{F}_t^{\text{cnt}}, \hat{E}_t}$ \\
 $J(\tt, K) $ & Average expected cost of controlling $\tt$ with a stabilizing controller $K$\\
 $\Sigma(\tt, K)$ & Covariance matrix of the state in system $\tt$ under controller $K$ \\
 $H_*^\tp$ & Concatenation of identity and optimal controller $K(\tts)$; $ [I, K(\tt_*)^\tp]$ \\
 $Q_* $ & $Q\!+\!K(\tts)^\tp R K(\tts)$ \\
 $L(A_c)$ & Function that maps any stable matrix $A_c$ to $\sigma_w^2 \sum_{t=0}^{\infty} \Norm[Q_*]{A_c^t}^2$ \\
 $F_t$ & Confidence interval for estimated closed-loop system; $\beta_t^2 H_*^\tp V_t^{-1} H_*$ \\
 $\lambda_{\max,t}$ & Maximum eigenvalue of $F_t$ \\
 $\lambda_{\min,t}$ & Minimum eigenvalue of $F_t$\\
 $\Xi$ & Random matrix of size $n\times n$ with iid $\mathcal{N}(0,1)$ entries \\
 $A_{c,*}$ & Closed-loop system matrix of the $\tts$ driven by $K(\tts)$; $\tts^\tp H_*$ \\
 $\hat{A}_c$ & Closed-loop system matrix of the $\tth$ driven by $K(\tts)$; $\tth^\tp H_*$ \\
 $\Tilde{A}_c$ & Closed-loop system matrix of the $\ttt$ driven by $K(\tts)$; $\ttt^\tp H_*$ \\
 $\mathcal{E}_t^{\text{cl}}(\delta)$ & Closed-loop confidence set that is super set to $\mathcal{E}_t^{\text{RLS}}(\delta)$; \eqref{eqn:cl_conf} \\
 $\hat{\Upsilon} $ & Unit Fro. norm matrix s.t. $\hat{\Upsilon}  \sqrt{F_t}$ is the h.p. confidence ellipsoid on $A_{c,*}$\\
 $\clf{M}_n$ & Manifold of square matrices of dimension $n$; $\R^{n\times n}$ \\ 
  $\clf{M}_{\text{Schur}}$ & Manifold of (Schur-)stable matrices in $\clf{M}_n$; $\cl{A_c \in \M_n \; | \; \rho(A_c) \!<\! 1}$\\
 $\clf{M}_*$ & Sublevel manifold in $\clf{M}_{\text{Schur}}$ s.t. $\cl{A_c \in \clf{M}_{\text{Schur}} \;|\; L(A_c) \leq L(A_{c,*})}$\\
 $G_t$ & Perturbation around $A_{c,*}$; $(\Xi +\hat{\Upsilon})\sqrt{F_t}$ \\
 $\nabla L_*$ & Jacobian operator of $L(\cdot)$ evaluated at $A_c\in \clf{M}_{\text{Schur}}$\\
 $\sigma_{\min,*}$ & Minimum singular value of $\nabla L_*$ \\
 $\clf{H}_{A_c}$ &  Hessian operator of $L(\cdot)$ evaluated at $A_c\in \clf{M}_{\text{Schur}}$\\
 $\clf{B}_*$ & Stable ball for some constant $\epsilon_*$; $ \cl{A_c \in \M_n \;|\; \norm[F]{A_c - A_{c,*}} \leq \epsilon_*}$\\
  $\clf{M}_{*}^{\text{qd}}$ & Sublevel manifold; $\cl{A_c \!\in\! \M_n \;|\; \norm[F]{A_c \!-\! A_{c,*} \!+\! r_*^{-1}\nabla L_*} \!\leq\! \norm[F]{r_*^{-1} \nabla L_*} }$\\
 \bottomrule
\end{tabular}}
\label{table:notation2}
\end{table}

%% file: appendix/stabilization.tex
In this section, we show that improved exploration of \alg provides persistently exciting inputs, which will be used to enable reaching a stabilizing neighborhood around $\tts$. From Assumption \ref{asm_noise}, we have that $\mathbb{E}[x_{t+1}x_{t+1}^\tp ~|~ \F_{t}] \psdgeq \sigma_w^2 I$. Thus, with the input $u_t = K(\ttt_t)x_t + \nu_t$ for $\nu_t \sim \mathcal{N}(0,  2\kappa^2\sigma_w^2 I)$, we have that $\mathbb{E}[z_{t+1}z_{t+1}^\tp ~|~ \F_{t}] \psdgeq \frac{\sigma_w^2}{2} I$. Using Lemma \ref{lem:persistence}, we have that $V_t \psdgeq t \frac{\sigma_w^2}{40} I$ for $t \geq 200(n+d)\log\frac{12}{\delta}$ with probability at least $1-\delta$. Using the RLS estimate error bound given in Section \ref{subsec:learning}, \ie, under the event of $\hat{E}_t$ we have
\begin{equation}
    \|\tth_t-\tts\|_2 \leq {\frac{\beta_t}{\sqrt{\lambda_{\text{min}}(V_t)}}}
\end{equation}
with probability at least $1-\delta$. Plugging in the $\lambda_{\min}(V_t)$ in its place yields the first result. 

For the second result, we use Lemma 4.2 of \citet{lale2022reinforcement}. Recall that $D \!=\! \overline{\alpha} \gamma^{-1}\kappa^2(1+\kappa^2)$. Lemma 4.2 of \citet{lale2022reinforcement} states that for any $(\kappa,\gamma)$-stabilizable system $\tts$ and for any $\varepsilon \leq \min\{ \sqrt{(\sigma_w^2n)/(142D^7)},1/(54D^5\}$, such that $\|\Theta' - \tts \| \leq \varepsilon$, $K(\Theta')$ produces $(\kappa\sqrt{2}, \gamma/2)$-stable closed-loop dynamics on $\tts$ such that there exists $L$ and $H\succ 0$ such that $A_* + B_* K(\Theta') = H'LH'^{-1}$, with $\|L\| \leq 1-\gamma/2$ and $\|H'\| \|H'^{-1} \| \leq \kappa\sqrt{2} $. Under the event of $\hat{E} \cap \Tilde{E}$, we have $ \|\ttt_t-\tts\|_2 \leq \frac{\beta_t(\delta) + \upsilon_t(\delta)}{\sqrt{\lambda_{\text{min}}(V_T)}}$. Under the event of $\hat{E} \cap \Tilde{E}$, this yields $\|\ttt_t-\tts\|_2 \leq  \frac{7(\beta_{t}(\delta) + \upsilon_t(\delta))}{\sigma_w \sqrt{t}}$ with probability $1-\delta$. Combining this result with the required $\varepsilon$ for finding the stabilizing neighborhood, for TS with exploration duration of
\begin{equation} \label{T0def}
    \Tw \geq T_0 \coloneqq \frac{49(\beta_{T}(\delta) + \upsilon_T(\delta))^2}{\sigma_w \min\{ (\sigma_w^2n)/(142D^7),1/(54^2D^{10}\}},
\end{equation}
\alg achieves $(\kappa\sqrt{2}, \gamma/2)$-stable closed-loop dynamics on $\tts$, with probability at least $1-3\delta$.

%% file: appendix/bdd_state.tex
In this section, we show that under the joint event of $\hat{E} \cap \Tilde{E}$ and the stabilization guarantee of the previous section, the state is bounded at all times during \alg and it is well-controlled during the stabilizing TS phase, \ie provide the proof of Lemma \ref{lem:bounded_state}.  We first consider the evolution of state for $t\leq \Tw$. To bound the state for the first phase, we adapt the state bounding strategy given in Section 4.1 of \citet{abbasi2011lqr} for contractible systems to the stabilizable systems via the slow policy changes of \alg. To this end, define the following 
\begin{equation*}
\bar{\alpha}_{t} =\frac{18\kappa^3}{\gamma(8\kappa-1)} \bar{\eta}^{n+d}\left[G Z_{t}^{\frac{n+d}{n+d+1}} \beta_{t}(\delta)^{\frac{1}{2(n+d+1)}}  + (\|B_* \| \sigma_\nu + \sigma_w) \sqrt{2n \log \frac{nt}{\delta}} \right],
\end{equation*}
for
\[\begin{aligned} 
\bar{\eta} &\geq \sup _{\Theta \in \mathcal{S}}\left\|A_{*}+B_{*} K(\Theta)\right\|, \qquad Z_{T} =\max _{1 \leq t \leq T_r}\left\|z_{t}\right\|, \quad U=\frac{U_{0}}{H}\\ 
G &=2\left(\frac{2 S(n+d)^{n+d+1 / 2}}{\sqrt{U}}\right)^{1 /(n+d+1)}, \quad
U_{0} =\frac{1}{16^{n+d-2}\max\left(1, \enskip S^{2(n+d-2)}\right)} \end{aligned}\]
and where $H$ is any number satisfying 
\begin{equation*}
H>\max\left(16, \enskip \frac{4 S^{2} M^{2}}{(n+d) U_{0}}\right), \enskip \text{where} \quad
M=\sup _{Y \geq 1} \frac{\left(\sigma_w \sqrt{n(n+d) \log \left(\frac{1+T Y / \lambda}{\delta}\right)}+\lambda^{1 / 2} S\right)}{Y}.
\end{equation*}
Under the joint event of $\hat{E}_t \cap \Tilde{E}_t$, \citet{abbasi2011lqr} show that the norm of the state is well-controlled except $n+d$ times at most in any horizon $T_r$. Denoting the set of time-steps that the state is not well-controlled by $\mathcal{T}_{t}$, the following lemma formalizes this argument:

\begin{lemma}[Lemma 18 of \citet{abbasi2011lqr}] \label{wellcontrol}
We have that for any $0 \leq t \leq T$,
\begin{equation*}
\max _{s \leq t, s \notin T_{t}}\left\|(\tts - \ttt_{s})^{\tp} z_{s}\right\| \leq G Z_{t}^{\frac{n+d}{n+d+1}} (\beta_{t}(\delta) + \upsilon(\delta))^{\frac{1}{2(n+d+1)}}.
\end{equation*}
\end{lemma}

Notice that this lemma is updated for TS. Moreover, it does not depend neither on the contractibility of the underlying system on the standard basis nor on the stabilizability. Equipped with this result, we write the closed loop system as 
\begin{equation*}
x_{t+1}=\Gamma_{t} x_{t}+r_{t}
\end{equation*}
where
\begin{equation} \label{state_update}
\Gamma_{t}=\left\{\begin{array}{ll}{\tilde{A}_{t-1}+\tilde{B}_{t-1} K(\tilde{\Theta}_{t-1})} & {t \notin \mathcal{T}_{\Tw}} \\ {A_{*}+B_{*} K(\tilde{\Theta}_{t-1})} & {t \in \mathcal{T}_{\Tw}}\end{array}\right. \,\text{and} \enskip r_{t}=\left\{\begin{array}{ll}{(\tts - \ttt_{t-1})^{\tp} z_{t}+B_*\nu_t+w_{t}} & {t \notin \mathcal{T}_{\Tw}} \\ {B_*\nu_t + w_{t}} & {t \in \mathcal{T}_{\Tw}}\end{array}\right.
\end{equation} 

Starting from $x_0 = 0$, we obtain the following roll out for the state,

\begin{align} 
x_{t} &=\Gamma_{t-1} x_{t-1}+r_{t-1}=\Gamma_{t-1}\left(\Gamma_{t-2} x_{t-2}+r_{t-2}\right)+r_{t} \nonumber \\ &=\Gamma_{t-1} \Gamma_{t-2} \Gamma_{t-3} x_{t-3}+\Gamma_{t-1} \Gamma_{t-2} r_{t-2}+\Gamma_{t-1} r_{t-1}+r_{t} \nonumber \\ &= \Gamma_{t-1} \Gamma_{t-2} \ldots \Gamma_{t-(t-1)} r_{1} + \cdots + \Gamma_{t-1} \Gamma_{t-2} r_{t-2} + \Gamma_{t-1} r_{t-1} + r_{t} \nonumber \\ &=\sum_{k=1}^{t}\left(\prod_{s=k}^{t-1} \Gamma_{s}\right) r_{k} \label{rollout} 
\end{align}

Recall that the sampled model is an element of $\mathcal{S}$ due to rejection sampling, thus, it is $(\kappa, \gamma)$-stabilizable by its optimal controller (Assumption \ref{asm_stabil}):
\begin{equation}
  1-\gamma \geq \max _{t \leq T} \rho\left(\tilde{A}_{t}+\tilde{B}_{t} K(\tilde{\Theta}_{t}) \right)  . \label{etabar}
\end{equation}

Notice that multiplication of the closed-loop system matrices are not guaranteed to be contractive without a similarity transformation. Therefore, unlike \citet{abbasi2011lqr} that bounds the rollout terms via contractive mappings due to their assumption of contractive systems, we need to make sure that the policy changes does not cause unexpected growth in the magnitude of the state. The slow policy update schedule, \ie, using all the sampled controllers for fixed $\tau_0$ time-steps, allows us to prevent such undesirable outcomes, In particular, by setting $\tau_0 = 2\gamma^{-1} \log(2\kappa\sqrt{2})$, we have that \begin{equation} \label{explore_state_bound}
    \|x_t\| \leq \frac{18\kappa^3 \bar{\eta}^{n+d} }{\gamma(8\kappa-1)} \left( \max _{1 \leq k \leq t}\left\|r_{k}\right\| \right)
\end{equation}

Moreover, we have that $\left\|r_{k}\right\| \leq\left\|(\tts - \ttt_{k-1})^{\tp} z_{k}\right\|+\left\|B_* \nu_k + w_{k}\right\|$ when $k \notin \mathcal{T}_{T},$ and $\left\|r_{k}\right\|=\left\|B_*\nu_k + w_{k}\right\|,$ otherwise. Hence, 

\[
\max _{k\leq t}\left\|r_{k}\right\| \leq \max _{k\leq t, k \notin \mathcal{T}_{t}}\left\|(\tts - \ttt_{k-1})^{\tp} z_{k}\right\|+\max _{k\leq t}\left\|B_*\nu_k + w_{k}\right\|
\]

The first term is bounded by the Lemma \ref{wellcontrol}. The second term involves summation of independent $\|B_* \| \sigma_\nu $ and $\sigma_w$ Gaussian vectors. Using standard Gaussian tail inequalities, for all $k\leq t$, we have $\left\|B_*\nu_k + w_{k}\right\| \leq (\|B_*\| \sigma_\nu + \sigma_w) \sqrt{2n \log \frac{nt}{\delta}}$ with probability at least $1-\delta$. Therefore, on the joint event of $\hat{E} \cap \Tilde{E}$, 

\begin{equation}
      \|x_t\| \leq \frac{18\kappa^3 \bar{\eta}^{n+d} }{\gamma(8\kappa-1)} \left[G Z_{t}^{\frac{n+d}{n+d+1}} (\beta_{t}(\delta) + \upsilon(\delta))^{\frac{1}{2(n+d+1)}}  + (\|B_* \| \sigma_\nu + \sigma_w) \sqrt{2n \log \frac{nt}{\delta}} \right]
\end{equation}
for $t \leq \Tw$ with probability $1-\delta$. Notice that this bound depends on $Z_t$ and $\beta_t(\delta)$ which in turn depends on $x_t$. Using Lemma 5 of \cite{abbasi2011lqr}, one can obtain the following bound 
\begin{equation}\label{stateboundexplore}
\|x_t \| \leq c' (n+d)^{n+d} .
\end{equation}
for some constant $c'$ for all $t \leq \Tw$ with probability $1-3\delta$, which gives the first advertised result.

To bound the state for $t > \Tw$, we show that, with the given choice of $\tau_0$, all the controllers during the stabilizing TS phase halves the magnitude of the state at the end of their control period. In particular, during the stabilizing TS phase, the closed-loop system dynamics can be written as $ x_{t+1} = (A_{*}+B_{*} K(\tilde{\Theta}_{t}) ) x_t + w_t = \tt_*^\tp H_{K(\ttt_t)} + w_t$. From the choice of $\Tw$ for the stabilizable systems, we have that $\tt_*^\tp H_{K(\ttt_t)} $ is $(\kappa \sqrt{2}, \gamma/2)$-strongly stable. Thus, we have $\rho(\tt_*^\tp H_{K(\ttt_t)}) \leq 1- \gamma/2$ for all $t > \Tw$ and $\|H_t\| \|H_t^{-1}\|\leq \kappa \sqrt{2}$ for $H_t \succ 0$, such that $\|L_t\| \leq 1- \gamma/2$ for $\tt_*^\tp H_{K(\ttt_t)} = H_t L_t H_t^{-1}$. Then for $T>t>\Tw$, if the same policy, $\tt_*^\tp H_{K(\ttt)}$ is applied starting from the state $x_{\Tw}$, we have the following state roll-out on the event of  $\hat{E}_t \cap \Tilde{E}_t$
\begin{align}
\| x_t \| &= \bigg \|\prod_{i=\Tw+1}^{t} \!\!\!\tt_*^\tp H_{K(\ttt)} x_{\Tw} +  \sum_{i=\Tw+1}^t \left(\prod_{s=i}^{t-1} \tt_*^\tp H_{K(\ttt)}\right) w_i  \bigg\| \\
&\leq \kappa\sqrt{2} (1-\gamma/2)^{t-\Tw} \|x_{\Tw}\| + \max_{\Tw < i \leq T}\left\| w_i  \right \| \left( \sum_{i=\Tw+1}^t \kappa\sqrt{2} (1-\gamma/2)^{t-i+1}  \right) \\
&\leq  \kappa\sqrt{2} (1-\gamma/2)^{t-\Tw} \| x_{\Tw}\| + \frac{2\kappa\sigma_w\sqrt{2}}{\gamma}  \sqrt{2n\log(n(t-\Tw)/\delta)} \label{policy_change_stabil}
\end{align}
with probability at least $1-\delta$. Since $\tau_0 = 2\gamma^{-1} \log(2\kappa\sqrt{2})$, we have $\kappa\sqrt{2}(1-\gamma/2)^{\tau_0} \leq 1/2$. Therefore, at the end of each controller period the effect of previous state is halved. Using this fact, at the $i$th policy change after $\Tw$, we get 
\begin{align*}
    \|x_{t_i}\| &\leq 2^{-i} \|x_{\Tw}\| + \sum_{j=0}^{i-1} 2^{-j} \frac{2\kappa\sigma_w\sqrt{2}}{\gamma}  \sqrt{2n\log(n(t-\Tw)/\delta)} \\
    &\leq 2^{-i} \|x_{\Tw}\| + \frac{4\kappa\sigma_w\sqrt{2}}{\gamma}  \sqrt{2n\log(n(t-\Tw)/\delta)}
\end{align*}
For all $i> (n+d)\log(n+d) - \log(\frac{2\kappa\sigma_w\sqrt{2}}{\gamma}  \sqrt{2n\log(n(t-\Tw)/\delta)} )$, at policy change $i$, we get
\begin{align*}
    \|x_{t_i}\| \leq \frac{6\kappa\sigma_w\sqrt{2}}{\gamma}  \sqrt{2n\log(n(t-\Tw)/\delta)}.
\end{align*}
Finally, from (\ref{policy_change_stabil}), we have that 
\begin{equation}
    \|x_t \| \leq \frac{(12\kappa^2+ 2\kappa\sqrt{2})\sigma_w}{\gamma}  \sqrt{2n\log(n(t-\Tw)/\delta)},
\end{equation}
with probability $1-4\delta$ for all $t>T_{r} \coloneqq \Tw + \left((n+d)\log(n+d)\right)\tau_0$. Based on this result, let $X_s = \frac{(12\kappa^2+ 2\kappa\sqrt{2})\sigma_w}{\gamma}  \sqrt{2n\log(n(T-\Tw)/\delta)}$. We define our final good event,
\begin{equation}
  \Bar{E}_t = \{ \forall t \leq T_r, \|x_t\| \leq c'(n+d)^{n+d} \text{ and } \forall t>T_r, \|x_t\| \! \leq\! X_s \}.  
\end{equation}
Notice that the joint event $E_t = \hat{E}_t \cap \Tilde{E}_t \cap \Bar{E}_t$ holds with probability at least $1-4\delta$. This event will be the key conditioning in the regret decomposition and the analysis.

%% file: appendix/optimism_proof.tex
In this section, we give the proof of the main technical contribution of this work, showing that TS samples optimistic model parameters with constant probability (Theorem \ref{thm:optimistic_prob}). The proof follows the outline provided in Section \ref{sec:overview}. We first provide the proofs of each lemma in Section \ref{sec:overview}. In particular Lemma \ref{thm:subset} is proven in Appendix \ref{apx:thm:subset}, Lemma \ref{thm:cl_conf} is studied in Appendix \ref{apx:thm:cl_conf}, Lemma \ref{thm:taylor} in Appendix \ref{apx:thm:taylor}, and Lemma \ref{thm:bounded_F} in \ref{apx:thm:bounded_F}. Finally, we combine these results to prove Theorem \ref{thm:optimistic_prob} in Appendix \ref{apx:thm:optimistic_prob}. 

\subsection{Proof of Lemma~\ref{thm:subset}} \label{apx:thm:subset}
Given a stabilizable system $\tt = (A,B)^\tp$, and a stabilizing linear feedback controller $K$, we can find the LQR cost as follows
\begin{align}
    J(\tt, K) &= \lim _{T \rightarrow \infty} \frac{1}{T} \E\br{\sum\nolimits_{t=1}^{T} x_{t}^{\tp} Q x_{t}+u_{t}^{\tp} R u_{t}},\\
    &= \lim _{T \rightarrow \infty} \frac{1}{T} \E\br{ \sum\nolimits_{t=1}^{T} \tr\pr{{( Q + K^\tp R K) x_{t}}x_{t}^{\tp}}}, \label{eqn:eqn1}\\
    &= \lim _{T \rightarrow \infty}  \tr\pr{( Q + K^\tp R K) \frac{1}{T}  \sum\nolimits_{t=1}^{T}\E\br { x_{t}x_{t}^{\tp}}}, \\
    &= \tr\pr{ (Q+K^\tp R K)\Sigma(\tt, K)}
\end{align}
where $\Sigma(\tt, K) \defeq \lim_{T \goesto \infty} \frac{1}{T} \sum\nolimits_{t=1}^{T}\E\br { x_{t}x_{t}^{\tp}}$ is the stationary state covariance of the closed-loop system. In \eqref{eqn:eqn1}, we used the feedback control policy relation $u_t = K x_t$ and trace trick for inner products of vectors. Note that the closed-loop system evolves as
\begin{align}
    x_{t+1} = (A+BK)x_{t} + w_{t}. 
\end{align}
The covariance of the state at time $t$ can be written as a recursive relation
\begin{align}
    \E\br{x_{t+1} x_{t+1}^\tp} &= \E\br{((A+BK)x_{t} + w_{t}) ((A+BK)x_{t} + w_{t})^\tp} \\
    &= (A+BK)\E\br{x_{t}x_{t}^\tp}(A+BK)^\tp + \E\br{w_t w_t^\tp} \label{eqn:eqn2}\\
    &= (A+BK)\E\br{x_{t}x_{t}^\tp}(A+BK)^\tp + \sigma_w^2 I
\end{align}
where \eqref{eqn:eqn2} is because $\E\br{w_t} = 0$ and $w_t$ and $x_t$ are independent. Since $\rho(A+BK)< 1$, the above iteration converges to a finite fixed-point. Furthermore, we have the following relation
\begin{align}
    \frac{1}{T}\sum\nolimits_{t=1}^{T} \E\br{x_{t+1} x_{t+1}^\tp}
    = (A+BK) \frac{1}{T}\sum\nolimits_{t=1}^{T} \E\br{x_{t}x_{t}^\tp}(A+BK)^\tp + \sigma_w^2 I
\end{align}
Denoting by $\Sigma_{T}(\tt, K)\defeq \frac{1}{T} \sum\nolimits_{t=1}^{T}\E\br { x_{t}x_{t}^{\tp}}$ the finite averaged state covariance, we have the following
\begin{align}
   \Sigma_{T}(\tt, K) + \frac{\E\br{x_{T+1}x_{T+1}^\tp} - \E\br{x_1 x_1^\tp}}{T}
    = (A+BK) \Sigma_{T}(\tt, K)(A+BK)^\tp + \sigma_w^2 I
\end{align}
Taking the limit of both sides as $T\rightarrow \infty$ and noting that $\E\br{x_{T+1}x_{T+1}^\tp}$ has a finite value at the limit, we obtain the following Lyapunov equation
\begin{align}
    \Sigma(\tt, K)
    = (A+BK) \Sigma(\tt, K)(A+BK)^\tp + \sigma_w^2 I
\end{align}
whose solution is given by the following convergent infinite sum 
\begin{align}
     \Sigma(\tt, K) = \sum_{t=0}^{\infty} (A+BK)^{t} \sigma_w^2 I \pr{(A+BK)^\tp}^t
\end{align}
It is well known that the optimal control policy of infinite-horizon LQR systems can be achieved by stationary linear feedback controllers \cite{bertsekas1995dynamic}. Therefore, we can find the optimal LQR cost of a stabilizable system by minimizing its closed-loop cost among all stabilizing stationary linear feedback controllers. 

Suppose $\tt \in \clf{S}^{\text{surr}}$, \ie, $J(\tt,K(\tts)) \leq J(\tts, K(\tts))$. Then, the optimal LQR cost of $\tt$ is given as
\begin{align}
    J(\tt) = J(\tt, K(\tt))  &= \min_{K \in \R^{d\times n}} J(\tt, K) \\
    &\leq J(\tt, K(\tts)) \stackrel{(a)}{\leq} J(\tts, K(\tts)) = J(\tts)
\end{align}
where $(a)$ is due to $\tt \in \clf{S}^{\text{surr}}$.
Thus, $\tt \in \clf{S}^{\opt}$. \QED

\subsection{Proof of Lemma~\ref{thm:cl_conf}}\label{apx:thm:cl_conf}
The following lemma will be used as the backbone for Lemma~\ref{thm:cl_conf}. 
\begin{lemma} \label{thm:compare_ellp}
Let $V_1,V_2 \in \R^{n\times n}$ be symmetric positive semi-definite matrices. Define two ellipsoids as
\begin{align}
    \clf{E}_1 &\defeq \cl{\tt \in \R^{n\times m} \;\big|\; \tr\pr{\tt ^\tp V_1 \tt } \leq 1  } \quad \text{and} \quad \clf{E}_2 \defeq \cl{\tt \in \R^{n\times m} \;\big|\; \tr\pr{\tt ^\tp V_2 \tt } \leq 1  }
\end{align}
Then, $\clf{E}_1 \subseteq \clf{E}_2$ if and only if $ V_1  \psdgeq V_2 $.
\end{lemma}
\begin{proof}
For the forward direction, assume $V_1 - V_2$ has a negative eigenvalue, \ie, there exist $\lambda < 0$ and a unit vector $\theta \in \R^n / \{0\}$ such that $(V_1 - V_2)\theta = \lambda \theta$. 
Construct $\tt = [\theta, \theta,\dots, \theta] \in \R^{n\times m}$. Observe that $\tr\pr{\tt^\tp V_1 \tt} = m \theta^\tp V_1 \theta$ and $\tr\pr{\tt^\tp V_2 \tt} = m \theta^\tp V_2 \theta$. Therefore, we have the relationship $\tr\pr{\tt^\tp V_2 \tt} = \tr\pr{\tt^\tp V_1 \tt} - m\lambda$. 

If $V_1\theta=0$, then $\tr(\tt^\tp V_1 \tt) = 0 \leq 1$ and therefore for any scalar $\alpha>0$, $\alpha \tt \in \clf{E}_1$. On the other hand, $\tr\pr{\tt^\tp V_2 \tt} = - m\lambda> 0$ and therefore, one can find a scalar $\alpha >0$ such that $\tr\pr{(\alpha\tt)^\tp V_2 (\alpha\tt)} =- m\lambda \alpha^2 > 1$, \ie $\alpha \tt \notin \clf{E}_2$. If $V_1\theta \neq 0$, then define $\tt^\prime =  \frac{1}{\sqrt{m \theta^\tp V_1 \theta}} \tt$ and observe that $\tr\pr{{\tt^\prime}^{\tp} V_1 \tt^\prime} = 1$, \ie,  $\tt^\prime \in \clf{E}_1$. On the other hand, $\tr\pr{{\tt^\prime}^{\tp} V_2 \tt^\prime}= 1 - \frac{\lambda }{\theta^\tp V_1 \theta} > 1$, \ie, $\tt^\prime \notin \clf{E}_2$. Therefore, we have that if $\clf{E}_1 \subseteq \clf{E}_2$ then $V_1 \psdgeq V_2$.

For the reverse direction, assume that $V_1 \psdgeq V_2$ and $\tt \in \clf{E}_1$. Then, $\tr\pr{\tt^\tp (V_1 - V_2) \tt} \geq 0$ and $\tr\pr{\tt^\tp V_2 \tt} \leq  \tr\pr{\tt^\tp V_1 \tt} \leq 1$. Therefore, $\tt \in \clf{E}_2$. 
\end{proof}

\paragraph{\textit{Proof of Lemma~\ref{thm:cl_conf}.}}
Let us rewrite the the ellipsoids. For the time being, we will drop $\delta$ dependence for simplicity.
\begin{align}
    \clf{E}_t^{\text{RLS}} &= \cl{ \tth \in \R^{(n+d)\times n} \;\big|\; \tr\pr{(\tth  - \tts)^\tp  \beta_t^{-1}V_t (\tth- \tts) } \leq 1 }, \\
    \clf{E}_t^{\text{cl}} &= \cl{ \tth \in \R^{(n+d)\times n} \;\big|\; \tr\pr{(\tth  - \tts)^\tp  H_* F_{t}^{-1}H_*^\tp (\tth- \tts) } \leq 1 }. 
\end{align}
In order to prove the lemma, it is necessary and sufficient to show $ \beta_t^{-1}V_t \!\psdgeq\! H_* F_{t}^{-1}H_*^\tp$ by Lemma~\ref{thm:compare_ellp}. Eliminating $b_t$ terms from both sides and multiplying by $V_t^{-\half}$ from left and right, we obtain the equivalent condition,
\begin{align}
    I \psdgeq V_t^{-\half}H_* (H_*^\tp V_t^{-1} H_*)^{-1} H_*^\tp V_t^{-\half} =  V_t^{-\half}H_* (H_*^\tp V_t^{-1} H_*)^{-\half} (H_*^\tp V_t^{-1} H_*)^{-\half}H_*^\tp V_t^{-\half} ,   
\end{align} 
In other words, we have that $\clf{E}_t^{\text{RLS}} \subseteq \clf{E}_t^{\text{cl}}$ if and only if $\norm[2]{(H_*^\tp V_t^{-1} H_*)^{-\half}H_*^\tp V_t^{-\half}} \leq 1$. Notice that 
\begin{align}
    \norm[2]{(H_*^\tp V_t^{-1} H_*)^{-\half}H_*^\tp V_t^{-\half}}^2 &= \sigma_1\pr{(H_*^\tp V_t^{-1} H_*)^{-\half}H_*^\tp V_t^{-\half}}^2, \\ &=\lambda_{\max}\pr{V_t^{-\half}H_* (H_*^\tp V_t^{-1} H_*)^{-1} H_*^\tp V_t^{-\half}}, \\
    &= \lambda_{\max}\pr{(H_*^\tp V_t^{-1} H_*)^{-\half}H_*^\tp V_t^{-1} H_*  (H_*^\tp V_t^{-1} H_*)^{-\half} }, \\
    &= \lambda_{\max}\pr{I} = 1,
\end{align}
where we used the fact that $\sigma_1(A) = \sqrt{\lambda_{\max}(A^\tp A)} =\sqrt{\lambda_{\max}(A A^\tp)}$. This is true for any time $t$ and $\delta$ and therefore completes the proof.\QED

\subsection{Proof of Lemma~\ref{thm:taylor}}\label{apx:thm:taylor}
The following lemma guarantees existence of a stable neighborhood around any stable matrix. 
\begin{lemma}\label{thm:stable_ball}
Let $A_c \in \clf{M}_{\text{Schur}}$, \ie, $\rho(A_c) < 1$. Then, there exists $\epsilon > 0$ such that for any $\Delta \in \M_n$ with $\norm[F]{\Delta} \leq 1$, we have that $A_c + \epsilon \Delta \in \clf{M}_{\text{Schur}}$, \ie, $\rho(A_c + \epsilon \Delta) < 1$.
\end{lemma}
\begin{proof}
Per Gelfand's formula, we have that for any $\delta >0$, there exists $N_{\delta} \in \N$ such that 
\begin{align}
    \rho(A_c) \leq \norm[F]{A_c^k}^{1/k} < \rho(A_c) + \delta \label{eqn:gelfand}
\end{align}
for any $k\geq N_{\delta}$. Since the mapping $A_c \mapsto \norm[F]{A_c^k}^{1/k}$ is smooth for any $k\in \N$, we can write the following expansion by Taylor's theorem for any $t\in \R$
\begin{align}
    \norm[F]{(A_c + t \Delta)^k}^{1/k} = \norm[F]{A_c^k}^{1/k} + t \ddx{t}  \norm[F]{(A_c + t \Delta)^k}^{1/k} \Big|_{\lambda t }  
\end{align}
where $\lambda \in [0,1]$. For a given $t\in \R$, there exists a constant $M_{k,t} > 0 $ such that for any $\norm[F]{\Delta}\leq 1$, we have that $\abs{\ddx{t}  \norm[F]{(A_c + t \Delta)^k}^{1/k} } \leq M_{k,t}$ by Taylor's theorem. Then, we can write the following upper bound
\begin{align}
    \norm[F]{(A_c + t \Delta)^k}^{1/k} \leq \norm[F]{A_c^k}^{1/k} + \abs{t} M_{t,k}  \label{eqn:perturbed_bound}
\end{align}
Using the relation~\eqref{eqn:gelfand} and the upper bound~\eqref{eqn:perturbed_bound}, we have that for any $\delta >0$, $t>0$, and $\norm[F]{\Delta}\leq 1$, there exists $N_{\delta} \in \N$ and $M_{t,N_{\delta}} >0$ such that
\begin{align}
    \rho(A_c + t \Delta) \leq \norm[F]{(A_c + t \Delta)^{N_{\delta}}}^{1/N_{\delta}} &\leq \norm[F]{A_c^{N_{\delta}}}^{1/N_{\delta}} + t  M_{t,N_{\delta}}  \\
    &< \rho(A_c) + \delta + t  M_{t,N_{\delta}} \label{eqn:eqn4}
\end{align}
Fix a $\delta >0$ such that $\rho(A_c) + \delta < 1$ and fix a $t>0$. Then, we can find $0<\epsilon\leq t $ such that $\rho(A_c) + \delta + \epsilon  M_{t,N_{\delta}} < 1$ and thus
\begin{align}
    \rho(A_c + \epsilon \Delta)  < \rho(A_c) + \delta + \epsilon  M_{t,N_{\delta}} < 1
\end{align}
for any $\norm[F]{\Delta}\leq 1$ by \eqref{eqn:eqn4}.
\end{proof}

\paragraph{\textit{Proof of Lemma~\ref{thm:taylor}.}} For any $A_c \in \M_{\text{Schur}}$, there exists a constant $\epsilon >0$, such that for any $\norm[F]{\Delta}\leq 1$, we have that $A_c + \epsilon \Delta \in \M_{\text{Schur}}$ by Lemma~\ref{thm:stable_ball}. To see smoothness of $L$, we write $A_t \defeq A_c + t\Delta$ and $L(A_t) = \tr(Q_* \Sigma_t)$ for any $\abs{t}\leq \epsilon$ and $\norm[F]{\Delta}\leq 1$ where $\Sigma_t$ solves the following Lyapunov equation
\begin{align}
    \Sigma_t - A_t\Sigma_t A_t^\tp = \sigma_w^2 I \text{ and } \Sigma_0 - A_c\Sigma_0 A_c^\tp = \sigma_w^2 I \label{eqn:lyap_t} 
\end{align}
Note that, $\rho(A_t) <1$ for any $\abs{t} \leq \epsilon$ and therefore both equations in \eqref{eqn:lyap_t} have unique solutions for any $\abs{t} \leq \epsilon$. The Jacobian $\nabla L (A_c) \in \M_n$ satisfies $\nabla L(A_c) \bullet  \Delta = \ddx{t} L(A_t) \big|_{t=0} = \tr(Q_* \dot{\Sigma}_0)$ for any $\norm[F]{\Delta}\leq 1$ where $\dot{\Sigma}_t$ is the derivative of $\Sigma_t$ and satisfies the following Lyapunov equation
\begin{align}
    \dot{\Sigma}_t - A_t\dot{\Sigma}_t A_t^\tp = \Delta \Sigma_t A_t^\tp + A_t \Sigma_t \Delta^\tp  \text{ and } \dot{\Sigma}_0 - A_c\dot{\Sigma}_0 A_c^\tp = \Delta \Sigma_0 A_c^\tp + A_c \Sigma_0 \Delta^\tp\label{eqn:dlyap_t}
\end{align}
Similarly, both equations in \eqref{eqn:dlyap_t} have unique solutions for any $\abs{t} \leq \epsilon$ and therefore $\nabla L(A_c)$ exists for any $A_c$. To find the Jacobian, we have that $\dot{\Sigma}_0 = \sum_{k=0}^{\infty} A_c^k \pr{\Delta \Sigma_0 A_c^\tp + A_c \Sigma_0 \Delta^\tp}  (A_c^\tp)^k$ and 
\begin{align}
    \tr(Q_* \dot{\Sigma}_0) &=\tr\pr{Q_* \sum_{k=0}^{\infty} A_c^k \pr{\Delta \Sigma_0 A_c^\tp + A_c \Sigma_0 \Delta^\tp}  (A_c^\tp)^k } \\
    &= 2 \tr\pr{  \sum_{k=0}^{\infty} (A_c^\tp)^k Q_* A_c^k   A_c \Sigma_0 \Delta ^\tp }  = 2\sum_{k=0}^{\infty} (A_c^\tp)^k Q_* A_c^k   A_c \Sigma_0 \bullet \Delta
\end{align}
Therefore, $\nabla L(A_c) = 2\sum_{k=0}^{\infty} (A_c^\tp)^k Q_* A_c^k   A_c \Sigma_0$. In particular, in the case of $A_{c,*}$, we have that $ \sum_{k=0}^{\infty} (A_{c,*}^\tp)^k Q_* A_{c,*}^k = P_* $, the solution to the Riccati equation, and thus $\nabla L(A_{c,*}) = 2P_* A_{c,*} \Sigma_*$. Repeating the same process, one can see that $L(A_t)$ is infinitely differentiable and thus we conclude $L$ is a smooth function.

Denote by $\clf{B}_{\epsilon} \defeq \cl{ A \in \M_n \,|\, \norm[F]{A - A_c} \leq \epsilon} \subset \M_{\text{Schur}}$ the ball of radius $\epsilon>0$ around $A_c \in \M_{\text{Schur}}$. Consider the function $L$ restricted to the domain $\clf{B}_{\epsilon} $. Since $\clf{B}_{\epsilon} $ is a convex set, we can apply Taylor's theorem to $L$ around $A_c$ in this domain to obtain
\begin{align}
    L(A_c + \epsilon \Delta) = L(A_c) + \nabla L (A_c) \bullet \epsilon\Delta + \frac{1}{2} \epsilon \Delta \bullet \clf{H}_{A_c+s\Delta}(\epsilon\Delta)
\end{align}
for $\norm[F]{\Delta}\leq 1$ and for some $s\in[0,\epsilon]$. Here, $\clf{H}_{A_c}: \M_n \to \M_n$ is the Hessian operator evaluated at a point $A_c\in \clf{M}_{\text{Schur}}$ and satisfies the following relationship
\begin{align}
    \Delta \bullet \clf{H}_{A_c}(\Delta) = \frac{\diff^2}{\diff t^2} L(A_c + t\Delta) \Big|_{t=0}
\end{align}
for any $\norm[F]{\Delta}\leq 1$. Finally, there exists a constant $r > 0$, such that for any $G\in\M_n$, we have that $ \abs{G \bullet \clf{H}_{A_c+s\Delta}(G)} \leq r\norm[F]{G}^2$ for any $s \in [0,\epsilon]$  and $\norm[F]{\Delta}\leq 1$ by Taylor's theorem .\QED

\subsection{Proof of Lemma~\ref{thm:bounded_F}}\label{apx:thm:bounded_F}
In this section, we will assume that Assumptions \ref{asm_stabil} \& \ref{asm_noise} hold. First, we need to show the boundedness of the stacked state and control input vector, $z_t$.
\begin{lemma}\label{thm:bounded_z}
Define the terms 
\begin{align}
    Z^{\prime}_{\Tw} &\defeq (1+\kappa)c^\prime (n+d)^{n+d} + \kappa\sigma_w \sqrt{4d\log(d \Tw /\delta)} \\
    Z^{\prime\prime}_{T} &\defeq (1+\kappa) (12\kappa^2\!+\! 2\kappa\sqrt{2})\gamma^{-1}\sigma_w \sqrt{2n\log(n(T\!-\!\Tw)/\delta)}
\end{align}
Then, the following holds w.p. at least $1-4\delta$,
\begin{align}
    \norm{z_t} \leq  
    \begin{cases}
        Z^{\prime}_{\Tw}, & \text{for } t\leq T_r  \\
        Z^{\prime\prime}_{T}, & \text{for } T_r< t\leq T
    \end{cases} \label{Vstep1}
\end{align}
\end{lemma}
\begin{proof}
From Lemma~\ref{lem:bounded_state}, we know that $\norm{x_t} \leq c^\prime (n+d)^{n+d}$ with $c^\prime>0$ a constant for $t\leq T_r$ and $\norm{x_t} \leq (12\kappa^2\!+\! 2\kappa\sqrt{2})\gamma^{-1}\sigma_w \sqrt{2n\log(n(t\!-\!\Tw)/\delta)}$ for all $T_r < s\leq T$ w.p. at least $1-4\delta$. Furthermore, under the event of $E_t$, we have that $\norm{u_t} \leq \kappa \norm{x_t} + \norm{v_t} \leq \kappa \norm{x_t} + \kappa\sigma_w \sqrt{4d\log(d \Tw /\delta)}$ for all $0\leq t \leq \Tw$.
Observing that $\norm{z_t} = \sqrt{\norm{x_t}^2 + \norm{u_t}^2} \leq \norm{x_t} + \norm{u_t}$, one can reach the desired result by substituting the appropriate bounds on $\norm{x_t}$ and $\norm{u_t}$ and considering the maximal case achieved when $t=T$.
\end{proof}
The following lemma will be used to bound $V_t$.
\begin{lemma}\label{thm:bounded_V}
Let $V_t = \mu I + \sum\nolimits_{s=0}^{t-1} z_s z_s^\tp$. On the event of $E_T = \hat{E} \cap \Tilde{E} \cap \Bar{E}$, we have
\begin{align}
    \lambda_{\max}(V_t) &\leq 
    \begin{cases}
        \mu +  t {Z^{\prime}}_{\Tw}^2, & \text{for } t\leq T_r  \\
        \mu +  T_r {Z^{\prime}}_{\Tw}^2 + (t-T_r) {Z^{\prime\prime}_{T}}^2, & \text{for } T_r< t\leq T
    \end{cases} \\
    \text{and} \quad \lambda_{\min}(V_t) &\geq 
    \begin{cases}
        \mu +  t \frac{\sigma_w^2}{40}, & \text{for } 200(n+d)\log{\frac{12}{\delta}} \leq t\leq \Tw  \\
        \mu +   \Tw \frac{\sigma_w^2}{40}, & \text{for } \Tw < t\leq T
    \end{cases}
\end{align}
\end{lemma}
\begin{proof}
Recall that on the event $E_T$, the RLS estimates, TS sampled systems are concentrated and the state is bounded, \textit{i.e.}, Lemma \ref{lem:bounded_state}. Conditioned on this event, we will start with bounding $\lambda_{\max}(V_t)$. For any time $0\leq t \leq T$, triangle inequality gives $\lambda_{\max}(V_t) = \norm[2]{\mu I + \sum\nolimits_{s=0}^{t-1} z_s z_s^\tp} \leq \mu + \sum\nolimits_{s=0}^{t-1} \norm{z_s}^2$. Using the bounds on $\norm{z_t}$ given in Lemma~\ref{thm:bounded_z}, we can write $\lambda_{\max}(V_t) \leq \mu + t  {Z^{\prime}}_{\Tw}^2$ for $t\leq T_r$ and $\lambda_{\max}(V_t) \leq \mu +  T_r {Z^{\prime}}_{\Tw}^2 + (t-T_r) {Z^{\prime\prime}_{T}}^2$ for $T_r< t\leq T$. For the lower bound, note that we have that $\mathbb{E}[z_{t+1}z_{t+1}^\tp ~|~ \F_{t}] \psdgeq \frac{\sigma_w^2}{2} I$. Using Lemma \ref{lem:persistence}, on the event $E_T$, we have that $V_t \psdgeq \mu I + t \frac{\sigma_w^2}{40} I$ for $200(n+d)\log\frac{12}{\delta} \leq t \leq \Tw$. Since $V_{t+1} = V_t + z_t z_t^\tp$, we have that $V_{t}\psdgeq V_{\Tw} \psdgeq \mu I + \Tw \frac{\sigma_w^2}{40} I$ for $\Tw < t\leq T$.
\end{proof}
Finally, we will use the following lemma to bound $\beta_t(\delta) = \sigma_w \sqrt{2 n \log \pr{\frac{\det(V_{t})^{1 / 2}}{\delta \det(\mu I)^{1 / 2}} }} + \sqrt{\mu}S$
\begin{lemma}\label{thm:beta_bound}
On the event of $E_T$, we have the following upper bound on $\beta_T(\delta)$:
\begin{align}
    \beta_T(\delta) \leq 4\sigma_w^2 n \log\pr{\frac{1}{\delta}} + 2\sigma_w^2 n(n+d) \log\pr{1 + \frac{T_r {Z^{\prime}}_{\Tw}^2 + (T-T_r) {Z^{\prime\prime}_{T}}^2}{(n+d)\mu} } + 2\mu S^2
\end{align}
\end{lemma}
\begin{proof}
Following a similar approach pursued in Lemma 10 of \cite{abbasi2011lqr}, we can bound the log-determinant of $V_t$ as
\begin{align}
    \log{\frac{\det(V_T)}{\det(\mu I)} } \leq (n+d) \log\pr{1 + \frac{T_r {Z^{\prime}}_{\Tw}^2 + (T-T_r) {Z^{\prime\prime}_{T}}^2}{(n+d)\mu} } \nn
\end{align}
by Lemma~\ref{thm:bounded_V}. This leads to the following upper bound on $\beta_t(\delta)$
\begin{align}
    \beta_T(\delta)^2 &\leq \pr{\sigma_w \sqrt{2 n \log\pr{\frac{1}{\delta}} + n(n+d)\log\pr{1 + \frac{T_r {Z^{\prime}}_{\Tw}^2 + (T-T_r) {Z^{\prime\prime}_{T}}^2}{(n+d)\mu} }} + \sqrt{\mu}S}^2 \nn \\
    &\leq 4\sigma_w^2 n \log\pr{\frac{1}{\delta}} + 2\sigma_w^2 n(n+d) \log\pr{1 + \frac{T_r {Z^{\prime}}_{\Tw}^2 + (T-T_r) {Z^{\prime\prime}_{T}}^2}{(n+d)\mu} } + 2\mu S^2. \nn
\end{align}
\end{proof}
\paragraph{\textit{Proof of Lemma~\ref{thm:bounded_F}.}} We will first show the desired bounds on $\lambda_{\min}(F_t)$ and $\lambda_{\max}(F_t)$. Recall that the event $E_T$ holds with probability at least $1-4\delta$. Noting that $H_*^\tp H_* = I + K_*^\tp K_* $, it is clear that $F_t \psdgeq \beta_t^2 \lambda_{\min}(V_t^{-1}) H_*^\tp H_* \psdgeq \frac{\beta_t^2}{\lambda_{\max}(V_t)} I$. Thus, from Lemma \ref{thm:bounded_V}, for $T_r < t\leq T$, we have that $ \lambda_{\min,t} \geq  \frac{\beta_t^2}{\lambda_{\max}(V_t)}  \geq \frac{\beta_t^2}{\mu +  T_r {Z^{\prime}}_{\Tw}^2 + (t-T_r) {Z^{\prime\prime}_{T}}^2}$. 
\\ \\
On the other hand, $F_t \psdleq \beta_t^2 \lambda_{\max}(V_t^{-1}) H_*^\tp H_* \psdleq \frac{\beta_t^2(1+\kappa^2)}{\lambda_{\min}(V_t)} I$. Again using Lemma \ref{thm:bounded_V}, for $T_r < t\leq T$, we have that $\lambda_{\max,t}\leq  \frac{(1+\kappa^2)\beta_t^2}{\lambda_{\min}(V_t)} \leq \frac{(1+\kappa^2)\beta_t^2}{\mu + \Tw \frac{\sigma_w^2}{40}}$. Since $t\mapsto \beta_t$ is increasing, $t\mapsto\lambda_{\max,t}$ is increasing as well. The condition number $\kappa_t\defeq  \frac{\lambda_{\max,t}}{\lambda_{\min,t}}  \leq \frac{\mu +  T_r {Z^{\prime}}_{\Tw}^2 + (t-T_r) {Z^{\prime\prime}_{T}}^2}{(1+\kappa^2)^{-1}(\mu + \Tw \frac{\sigma_w^2}{40})}$ is increasing for $T_r < t\leq T$.
\\ \\
If $\Tw = O(\sqrt{T}^{1+o(1)})$, then we have that $\lambda_{\max}(V_T)\leq O(\poly(n,d,\log(1/\delta))T\log{T})$ and $\beta_T(\delta) \leq O(\poly(n,d,\log(1/\delta))\log{T})$. Thus, there are positive constants $C = \poly(n,d,\log(1/\delta))$ and $c = \poly(n,d,\log(1/\delta))$ such that $\lambda_{\max,T} \leq C \frac{\log{T}}{\Tw}$ and $\kappa_{t} = \frac{\lambda_{\max,T}}{\lambda_{\min,t}} \leq c \frac{T\log{T}}{\Tw}$ for $T_r < t\leq T$ for large enough $T$. Choosing the larger between $C$ and $c$ yields the desired result. 
\QED

\subsection{Proof of Theorem~\ref{thm:optimistic_prob}}\label{apx:thm:optimistic_prob}


Defining by $p_t^{\opt} \defeq \Pr \cl{ \ttt_t \in \S^{\text{opt}} \cond \clf{F}_t^{\text{cnt}}, \hat{E}_t}$ the optimistic probability, and by $\Pr_{t}\{\cdot\} \defeq \Pr\{\cdot \,|\, \clf{F}_t^{\text{cnt}}\}$ conditional probability measure,  we can write
\begin{align}
    p_t^{\opt} &\geq \Pr \cl{ \ttt_t \in \S^{\text{surr}} \cond \F_t^{\text{cnt}}, \hat{E}_t}  \label{step1} \\
    &= \Pr \cl{ L(\ttt_t^\tp H_*) \leq  L(\tts^\tp H_*) \cond \clf{F}_t^{\text{cnt}}, \hat{E}_t} \\
    &\geq \min_{\hat{\Theta} \in \mathcal{E}_t^{\text{RLS}}} \Pr_{t}\{   L(\hat{\Theta}^\tp H_* + \eta^\tp \beta_t V_t^{-\frac{1}{2}}H_*) \leq  L( {\Theta}_*^\tp H_*)\}  \label{step2}  \\
    &= \min_{\hat{\Theta} \in \mathcal{E}_t^{\text{RLS}}} \Pr_{t}\{   L(\hat{\Theta}^\tp H_* + \Xi \sqrt{F_t}) \leq  L( {\Theta}_*^\tp H_*)\}   \label{step3}
\end{align}
where \eqref{step1} is by Lemma~\ref{thm:subset}, \eqref{step2} is a worst-case estimation bound within high-probability confidence region, and \eqref{step3} is because $\eta^\tp \beta_t V_t^{-\frac{1}{2}}H_*$ and  $ \Xi \sqrt{F_t}$ have the same distributions with $\eta\in\R^{(n+d)\times n}$ and $\Xi\in \R^{n\times n}$ being i.i.d. standard normal random matrices. 
\\ \\
The bound in \eqref{step3} can be further lower bounded by minimizing over a larger confidence set as
\begin{align}
    p_t^{\opt}&\geq \min_{\hat{\Theta} \in \mathcal{E}_t^{\text{cl}}} \Pr_{t}\{ L(\hat{\Theta}^\tp H_* + \Xi \sqrt{F_t}) \leq  L( A_{c,*}) \} \label{step4} \\
    &= \min_{\hat{\Upsilon}\,:\, \norm[F]{\hat{\Upsilon} }\leq 1 } \Pr_{t}\{   L({A}_{c,*} + (\Xi+\hat{\Upsilon} ) \sqrt{F_t}) \leq  L( A_{c,*})  \} \label{step7},
\end{align}
where \eqref{step4} is by Lemma~\eqref{thm:cl_conf} and \eqref{step7} is because $H_*$ is full column rank and therefore we can minimize over closed-loop matrices instead of open-loop system parameters. 
\\ \\
Denoting by $G_t = (\Xi+\hat{\Upsilon} ) \sqrt{F_t}$ the perturbation due to estimation and sampling, Lemma~\ref{thm:taylor} suggests that there exists constants $\epsilon_* > 0$ and $r_* >0$ such that 
\begin{align}
    L(A_{c,*} + G_t) &= L(A_{c,*}) +\nabla L_*\bullet G_t + \frac{1}{2} G_t \bullet \clf{H}_{A_{c,*}+sG_t}(G_t) \\
    &\leq L(A_{c,*}) + \nabla L_*\bullet G_t + \frac{r_*}{2}\norm[F]{G_t}^2 \label{step8}
\end{align}
whenever $\norm[F]{G_t}\leq \epsilon_*$. Substituting \eqref{step8} into \eqref{step7} leads to the following lower bound
\begin{align}
    p_t^{\opt} &\geq \min_{\hat{\Upsilon}\,:\, \norm[F]{\hat{\Upsilon} }\leq 1 } \Pr_{t}\{   L({A}_{c,*} + G_t) \leq  L( A_{c,*})  \} \\
    &\geq \min_{\hat{\Upsilon}\,:\, \norm[F]{\hat{\Upsilon} }\leq 1 } \Pr_{t}\cl{
    \begin{array}{c}
		L(A_{c,*}) + \nabla L_*\bullet G_t + \frac{r_*}{2}\norm[F]{G_t}^2  \leq  L(A_{c,*}), \\
	    \text{ and }\quad \norm[F]{G_t}\leq \epsilon_*
	\end{array}} \\
    &= \min_{\hat{\Upsilon}\,:\, \norm[F]{\hat{\Upsilon} }\leq 1 } \Pr_{t}\cl{
    \begin{array}{c}
		\frac{r_*}{2}\norm[F]{(\Xi +\hat{\Upsilon})\sqrt{F_t}}^2 + \nabla L_*\bullet (\Xi +\hat{\Upsilon})\sqrt{F_t} \leq  0, \\
	    \text{ and }\quad \norm[F]{(\Xi +\hat{\Upsilon})\sqrt{F_t}}\leq \epsilon_*
	\end{array}} \label{step9}
\end{align}
Noting that $\norm[F]{(\Xi +\hat{\Upsilon})\sqrt{F_t}} \leq \sqrt{\lambda_{\max,t}}\norm[F]{\Xi +\hat{\Upsilon}}$ where $\lambda_{\max,t} \defeq \lambda_{\max}(F_t)$, we can further relax the lower bound \eqref{step9} as
\begin{align}
    p_t^{\opt} &\geq \min_{\hat{\Upsilon}\,:\, \norm[F]{\hat{\Upsilon} }\leq 1 } \Pr_{t}\cl{
    \begin{array}{c}
		\frac{\lambda_{\max,t}r_*}{2}\norm[F]{\Xi +\hat{\Upsilon}}^2 + (\nabla L_* \sqrt{F_t})\bullet (\Xi +\hat{\Upsilon}) \leq  0, \\
	    \text{ and }\quad \sqrt{\lambda_{\max,t}} \norm[F]{\Xi +\hat{\Upsilon}}\leq \epsilon_*
	\end{array}} \label{step10}\\
	&= \min_{\hat{\Upsilon}\,:\, \norm[F]{\hat{\Upsilon} }\leq 1 } \Pr_{t}\cl{
    \begin{array}{c}
		\Norm[F]{\Xi +\hat{\Upsilon} + \frac{\nabla L_* \sqrt{F_t}}{\lambda_{\max,t}r_*}}^2   \leq  \Norm[F]{\frac{\nabla L_* \sqrt{F_t}}{\lambda_{\max,t}r_*}}^2, \\
	    \text{ and }\quad  \norm[F]{\Xi +\hat{\Upsilon}}^2\leq \frac{\epsilon_*^2}{\lambda_{\max,t}}
	\end{array}} \label{step11}
\end{align}
where \eqref{step11} is obtained by completion of squares. Let $\clf{U}:\M_n\to \M_n$ be an orthogonal transformation such that $\clf{U}\pr{\hat{\Upsilon}+ \frac{\nabla L_* \sqrt{F_t}}{\lambda_{\max,t}r_*}} = \Norm[F]{\hat{\Upsilon} + \frac{\nabla L_* \sqrt{F_t}}{\lambda_{\max,t}r_*}} E_{11}$ where $E_{11}\in \M_n$ has $1$ in its $(1,1)$ entry and zeros elsewhere. Since Frobenius norm and the probability density of $\Xi$ are invariant under orthogonal transformations, \eqref{step11} can be rewritten as
\begin{align}
    p_t^{\opt} &\geq \min_{\hat{\Upsilon}\,:\, \norm[F]{\hat{\Upsilon} }\leq 1 } \Pr_{t}\cl{
    \begin{array}{c}
		\Norm[F]{\clf{U}\pr{\Xi +\hat{\Upsilon} + \frac{\nabla L_* \sqrt{F_t}}{\lambda_{\max,t}r_*}}}^2   \leq  \Norm[F]{\frac{\nabla L_* \sqrt{F_t}}{\lambda_{\max,t}r_*}}^2, \\
	    \text{ and }\quad  \norm[F]{\clf{U}(\Xi +\hat{\Upsilon})}^2\leq \frac{\epsilon_*^2}{\lambda_{\max,t}}
	\end{array}} \\
	&= \min_{\hat{\Upsilon}\,:\, \norm[F]{\hat{\Upsilon} }\leq 1 } \Pr_{t}\cl{
    \begin{array}{c}
		\Norm[F]{\Xi +\Norm[F]{\hat{\Upsilon} + \frac{\nabla L_* \sqrt{F_t}}{\lambda_{\max,t}r_*}} E_{11}}^2   \leq  \Norm[F]{\frac{\nabla L_* \sqrt{F_t}}{\lambda_{\max,t}r_*}}^2, \\
	    \text{ and }\quad  \norm[F]{\Xi +\clf{U}(\hat{\Upsilon})}^2\leq \frac{\epsilon_*^2}{\lambda_{\max,t}}
	\end{array}} \\
	&= \min_{\hat{\Upsilon}\,:\, \norm[F]{\hat{\Upsilon} }\leq 1 } \Pr_{t}\cl{
    \begin{array}{c}
		\pr{\Xi_{11} +\Norm[F]{\hat{\Upsilon} + \frac{\nabla L_* \sqrt{F_t}}{\lambda_{\max,t}r_*}}}^2 + \sum_{i,j\neq 1,1} \Xi_{ij}^2   \leq  \Norm[F]{\frac{\nabla L_* \sqrt{F_t}}{\lambda_{\max,t}r_*}}^2, \\
	    \text{ and }\quad  \norm[F]{\Xi +\clf{U}(\hat{\Upsilon})}^2\leq \frac{\epsilon_*^2}{\lambda_{\max,t}}
	\end{array}} \label{step12}
\end{align}
Notice that the probability in \eqref{step12} is described by the intersection of two balls whose centers are far apart by $\frac{\norm[F]{\nabla L_* \sqrt{F_t}}}{\lambda_{\max,t}r_*} $ and hence the intersection has a fixed shape. Choosing $\hat{\Upsilon}$ along the direction of $\norm[F]{\nabla L_* \sqrt{F_t}}$ moves the center of the first ball furthest possible from the origin which leads to the intersection of the balls to move furthest away from the origin as well. Therefore, the probability in \eqref{step12} attains its minimum at $\hat{\Upsilon}_{\#} \defeq \frac{\nabla L_* \sqrt{F_t}}{\norm[F]{\nabla L_* \sqrt{F_t}}}$ and \eqref{step12} can be equivalently expressed by
\begin{align}
    p_t^{\opt} &\geq \Pr_{t}\cl{
    \begin{array}{c}
		\pr{\Xi_{11} +1+\frac{\norm[F]{\nabla L_* \sqrt{F_t}}}{\lambda_{\max,t}r_*}}^2 + \sum_{i,j\neq 1,1} \Xi_{ij}^2   \leq  \frac{\norm[F]{\nabla L_* \sqrt{F_t}}^2}{\lambda_{\max,t}^2r_*^2}, \\
	    \text{ and }\quad  \norm[F]{\Xi +E_{11}}^2\leq \frac{\epsilon_*^2}{\lambda_{\max,t}}
	\end{array}} \label{step13}
	\\
	&= \Pr_{t}\cl{
    \begin{array}{c}
		\pr{\xi +1+\frac{\norm[F]{\nabla L_* \sqrt{F_t}}}{\lambda_{\max,t}r_*}}^2 + X   \leq  \frac{\norm[F]{\nabla L_* \sqrt{F_t}}^2}{\lambda_{\max,t}^2 r_*^2} , \\
	    \text{ and }\quad  (\xi+1)^2 + X \leq \frac{\epsilon_*^2}{\lambda_{\max,t}}
	\end{array}} \label{step14}
\end{align}
where $\xi \sampled \normal(0,1)$ and $X \sampled \chi^2_{n^2-1}$ are independent normal and chi-squared random variables, respectively. Denoting by $a_t \defeq \frac{\norm[F]{\nabla L_* \sqrt{F_t}}}{\lambda_{\max,t}r_*}$ and $b_t = \frac{\epsilon_*}{\sqrt{\lambda_{\max,t}}}$ the radii of the balls, we can rewrite \eqref{step14} as 
\begin{align}
    p_t^{\opt} &\geq \Pr_{t}\cl{\pr{\xi +1+a_t}^2 + X   \leq  a_t^2 ,\text{ and } (\xi+1)^2 + X \leq b_t^2} \label{step15} \\
    &=\Pr_{t}\cl{
    \begin{array}{c}
		\abs{\xi +1+a_t}    \leq  \sqrt{a_t^2-X} ,\text{ and }  \abs{\xi+1} \leq \sqrt{b_t^2-X},  \\
	    \text{ and } X\leq \min(a_t^2,b_t^2)
	\end{array}} \label{step16} \\
	&= \int_{0}^{\min(a_t^2,b_t^2)} \Pr_{t}\cl{\abs{\xi +1+a_t}    \leq  \sqrt{a_t^2-x} ,\text{ and } \abs{\xi+1} \leq \sqrt{b_t^2-x}} f_{n^2-1}(x)\diff x \label{step17} \\
	&= \int_{0}^{\min(a_t^2,b_t^2)} \Pr_{t}\cl{\begin{array}{c}
		1+a_t - \sqrt{a_t^2-x}\leq \xi \leq  1+a_t +\sqrt{a_t^2-x},  \\
		\text{ and }
	   1 - \sqrt{b_t^2-x}\leq \xi \leq  1 +\sqrt{b_t^2-x},
	\end{array}} f_{n^2-1}(x)\diff x\label{step18}
\end{align}
where $f_{k}(x) \defeq \pr{2^{\frac{k}{2}} \Gamma(\frac{k}{2})}^{-1} x^{\frac{k}{2}-1} e^{-\frac{x}{2}}$ is the probability density function of the chi-squared distribution with $k\in \N$ degrees of freedom. \eqref{step17} is derived from law of total probability. Notice that the probability inside the integral in \eqref{step18} is determined by the intersection of two intervals. This probability will have a non-zero value only for a fixed interval of $x$ depending on the relation between $a_t$ and $b_t$. We will investigate three cases:
\paragraph{$\bm{i. \; 0 \leq b_t\leq \sqrt{2}a_t:}$}There is a non-empty intersection if and only if $0\leq x \leq b_t^2 \pr{ 1-\frac{b_t^2}{4a_t^2}}$ and the integral \eqref{step18} becomes
\begin{align}
    p_t^{\opt} &\geq \int_{0}^{b_t^2 \pr{ 1-\frac{b_t^2}{4a_t^2}}} \Pr_{t}\cl{1+a_t - \sqrt{a_t^2-x}\leq \xi \leq 1 +\sqrt{b_t^2-x}} f_{n^2-1}(x)\diff x \\
    &= \int_{0}^{b_t^2 \pr{ 1-\frac{b_t^2}{4a_t^2}}} \br{Q\pr{1+a_t - \sqrt{a_t^2-x}} - Q\pr{1 +\sqrt{b_t^2-x}}} f_{n^2-1}(x)\diff x \label{step19}
\end{align}
where $Q$ is the Gaussian $Q$-function. Notice that for fixed values of $b_t$, \eqref{step19} is monotonically increasing with respect to $a_t$ and vice versa.
\paragraph{$\bm{ii. \; \sqrt{2}a_t \leq b_t\leq 2a_t:}$} There is a non-empty intersection if and only if $0\leq x \leq a_t^2 $ and the integral \eqref{step18} becomes
\begin{align}
    p_t^{\opt} &\geq \int_{0}^{a_t^2} \Pr_{t}\cl{1+a_t - \sqrt{a_t^2-x}\leq \xi \leq 1 +\sqrt{b_t^2-x}} f_{n^2-1}(x)\diff x \\
    &= \int_{0}^{a_t^2} \br{Q\pr{1+a_t - \sqrt{a_t^2-x}} - Q\pr{1 +\sqrt{b_t^2-x}}} f_{n^2-1}(x)\diff x\label{step20}
\end{align}
Notice that for fixed values of $b_t$, \eqref{step20} is monotonically increasing with respect to $a_t$ and vice versa.
\paragraph{$\bm{iii. \; 2a_t  \leq b_t:}$} There is a non-empty intersection if and only if $0\leq x \leq a_t^2 $ and the integral \eqref{step18} becomes
\begin{align}
    p_t^{\opt} &\geq \int_{0}^{a_t^2} \Pr_{t}\cl{1+a_t - \sqrt{a_t^2-x}\leq \xi \leq 1+a_t + \sqrt{a_t^2-x}} f_{n^2-1}(x)\diff x \\
    &= \int_{0}^{a_t^2} \br{Q\pr{1+a_t - \sqrt{a_t^2-x}} - Q\pr{1+a_t + \sqrt{a_t^2-x}}} f_{n^2-1}(x)\diff x \label{step21}
\end{align}
Notice that for fixed values of $b_t$, \eqref{step21} is monotonically increasing with respect to $a_t$ and vice versa.
\\ \\
As seen from all three case, the integral in \eqref{step18} is monotonically increasing with respect to both $a_t$, and $b_t$ regardless of their relative relation. Therefore, we will consider tight lower bounds of $a_t = \frac{\norm[F]{\nabla L_* \sqrt{F_t}}}{\lambda_{\max,t}r_*}$ so that the relation $b_t \geq 2a_t$ holds for large enough $t\geq 0$. Noting that $\nabla L_* = 2P_* A_{c,*} \Sigma_*$ by Lemma~\ref{thm:taylor} and $P_*\psdg 0$, $\Sigma_*\psdg 0$, we will consider two cases.
\paragraph{$\bm{1. \text{ Singular } A_{c,*}: }$} In this case, the Jacobian matrix $\nabla L_*$ becomes singular as well. Then, we can bound $a_t$ from below as $a_t = \frac{\norm[F]{\nabla L_* \sqrt{F_t}}}{\lambda_{\max,t}r_*} \geq \sqrt{\lambda_{\min,t}} \frac{\norm[F]{\nabla L_*}}{\lambda_{\max,t}r_*} = \sqrt{\frac{\lambda_{\min,t}}{\lambda_{\max,t}}} \frac{\norm[F]{r_*^{-1}\nabla L_* }}{\sqrt{\lambda_{\max,t}}}$. Furthermore, choosing $\Tw = O((\sqrt{T})^{1+o(1)})$, we can use upper bounds for $\frac{\lambda_{\max,t}}{\lambda_{\min,t}}$ and $\lambda_{\max,t}$ from Lemma~\ref{thm:bounded_F} to write down, $a_t \geq \frac{\Tw}{\sqrt{T}\log{T}}\frac{\norm[F]{\nabla L_* }}{Cr_*} \eqdef a_{1,T}$ and $b_t \geq \sqrt{\frac{\Tw}{\log{T}}}\frac{\epsilon_*}{\sqrt{C}}\eqdef b_{1,T}$ for all $T_r < t \leq T$ under the event $E_T$ for large enough $T$. Therefore, replacing $a_t$ and $b_t$ with $a_{1,T}$ and $b_{1,T}$ in \eqref{step18} gives a lower bound to \eqref{step18}. Noting that the ratio $\frac{b_{1,T}}{a_{1,T}} =\sqrt{\frac{T\log{T}}{\Tw}} \frac{\epsilon_* r_*\sqrt{C}}{\norm[F]{\nabla L_*}}$ can be made to be greater than or equal to $2$ by an appropriate choice of $\Tw$ leading to the case $(iii)$ bound
\begin{align}
    p_t^{\opt} &\geq  \int_{0}^{a_{1,T}^2} \br{Q\pr{1+a_{1,T} - \sqrt{a_{1,T}^2-x}} - Q\pr{1+a_{1,T} + \sqrt{a_{1,T}^2-x}}} f_{n^2-1}(x)\diff x  \label{step22}
\end{align}
for all $T_r < t\leq T$ for large enough $T$. 

\paragraph{$\bm{2. \text{ Nonsingular } A_{c,*}: }$} In this case, the Jacobian matrix $\nabla L_*$ becomes nonsingular as well. Then, we can bound $a_t$ from below as $a_t = \frac{\norm[F]{\nabla L_* \sqrt{F_t}}}{\lambda_{\max,t}r_*} \geq \sigma_{\min,*} \frac{\norm[F]{\sqrt{F_t}}}{\lambda_{\max,t}r_*} \geq \frac{\sigma_{\min,*}}{r_*\sqrt{\lambda_{\max,t}}}$. Choosing $\Tw = O((\sqrt{T})^{1+o(1)})$, we can use the upper bound for $\lambda_{\max,t}$ from Lemma~\ref{thm:bounded_F} to write the lower bound, $a_t \geq \sqrt{\frac{\Tw}{\log{T}}}\frac{\min(\sigma_{\min,*},\, \epsilon_* r_* /2)}{\sqrt{C}r_*} \eqdef a_{2,T}$ and $b_t \geq \sqrt{\frac{\Tw}{\log{T}}}\frac{\epsilon_*}{\sqrt{C}}\eqdef b_{2,T}$ for all $T_r<t\leq T$ under the event $E_T$ for large enough $T$. Therefore, replacing $a_t$ and $b_t$ with $a_{2,T}$ and $b_{2,T}$ in \eqref{step18} gives a lower bound to \eqref{step18} for $T_r < t\leq T$. Noting that the ratio $\frac{\beta_{2,T}}{a_{2,T}} =\frac{\epsilon_* r_*}{\min(\sigma_{\min,*},\, \epsilon_* r_* /2)} = \max\pr{\frac{\epsilon_* r_*}{\sigma_{\min,*}},\, 2} \geq 2$, we can use the case $(iii)$ bound
\begin{align}
    p_t^{\opt} &\geq  \int_{0}^{a_{2,T}^2} \br{Q\pr{1+a_{2,T} - \sqrt{a_{2,T}^2-x}} - Q\pr{1+a_{2,T} + \sqrt{a_{2,T}^2-x}}} f_{n^2-1}(x)\diff x \label{step23}
\end{align}
for all $T_r <  t\leq T$ for large enough $T$. 
\\ \\
In both cases, our focus will be on the following probability with a parameters $a >0$, and $k \in \N$
\begin{align}
    p_{k}(a) &\defeq  \int_{0}^{a^2} \br{Q(1+a - \sqrt{a^2-x}) - Q(1+a + \sqrt{a^2-x})} f_{k}(x)\diff x \label{eqn:general_integral}
\end{align}
The following lemma summarizes some of the important properties of the function $a \mapsto p_{k}(a)$.
\begin{lemma}\label{thm:asymptotics}
The non-negative real valued function $a \mapsto p_{k}(a)$ is monotonically increasing with respect to $a \geq 0$. Furthermore, we have that $\frac{1}{p_{k}(a)}\leq \frac{1}{Q(1)} \pr{1 + \frac{Ck}{ a^{1/2}}}$ for  $a\geq ck$ for problem independent constants $c,C>0$.
\end{lemma}
\begin{proof}
Notice that for a fixed value of $0\leq x \leq a^2$, the functions $a \mapsto 1+ a - \sqrt{a^2 -x}$ and $a \mapsto 1+ a + \sqrt{a^2 -x}$ are monotonically decreasing and monotonically increasing, respectively. As $Q$-function is monotonically decreasing, the function $a \mapsto Q(1+a - \sqrt{a^2-x}) - Q(1+a + \sqrt{a^2-x})$ is monotonically increasing for fixed $0\leq x \leq a^2$. Therefore, the function $a \mapsto p_{k}(a)$ is also monotonically increasing. 
\\ \\
In order to obtain the desired asymptotic bound, let $\epsilon \in (0,1)$ and we can write
\begin{align}
    p_{k}(a) &=  \int_{0}^{a^2} \br{Q(1+a - \sqrt{a^2-x}) - Q(1+a + \sqrt{a^2-x})} f_{k}(x)\diff x \nn \\
    &\geq \int_{0}^{\epsilon a^2} \br{Q(1+a - \sqrt{a^2-x}) - Q(1+a + \sqrt{a^2-x})} f_{k}(x)\diff x \nn \\
    &\geq \int_{0}^{\epsilon a^2} \min_{0\leq x^\prime \leq \epsilon a^2}\br{Q(1+a - \sqrt{a^2-x^\prime}) - Q(1+a + \sqrt{a^2-x^\prime})} f_{k}(x)\diff x \nn \\
    &= \br{Q(1+a(1-\sqrt{1-\epsilon})) - Q(1+a(1+\sqrt{1-\epsilon}))} F_{k}(\epsilon a^2) \nn 
\end{align}
where $F_{k}(x) \defeq 1-\frac{\Gamma(\nicefrac{k}{2},\,\nicefrac{x}{2})}{\Gamma(\nicefrac{k}{2})}$ is the cumulative distribution function of chi-square distribution and $(s,x)\mapsto\Gamma(s,x)\defeq \int_{x}^{\infty}t^{s-1}e^{-t}\diff t$ and $s\mapsto\Gamma(s)\defeq \int_{o}^{\infty}t^{s-1}e^{-t}\diff t$ are upper incomplete Gamma and ordinary Gamma functions respectively. Notice that the functions $ (s,x)\mapsto\Gamma(s,x)$  and $x\mapsto Q(x)$ are monotonically decreasing with increasing $x>0$. Therefore, for large enough $\epsilon a^2 \gg 1$ and large enough $a\gg 1$,  we can claim that $\Gamma(\nicefrac{k}{2},\,\nicefrac{\epsilon a^2}{2}) \ll 1$ and $Q(1+a) \ll 1$ are small enough. Furthermore, for small enough $\epsilon \ll 1$, we can use Taylor expansion to see that $1 - \sqrt{1 -  \epsilon} = \frac{\epsilon}{2}\sum_{k=0}^{\infty} \frac{\epsilon^k}{2^k}  (2k-1)!! \leq c_1 \epsilon$ for a problem independent constant $c_1 > 0$.  Then, for small enough $\epsilon \ll 1$, we have that 
\begin{align}
    p_{k}(a) &\geq \br{Q(1+a(1-\sqrt{1-\epsilon})) - Q(1+a(1+\sqrt{1-\epsilon}))}   \pr{1-\frac{\Gamma(k/2,\,\epsilon a^2/2)}{\Gamma(k/2)}} \nn \\
    &\geq \br{Q(1+c_1 \epsilon a) - Q(1+a)}   \pr{1-\frac{\Gamma(k/2,\,\epsilon a^2/2)}{\Gamma(k/2)}} \nn
\end{align}
Furthermore, for small enough $\epsilon a \ll 1$, we have that $Q(1+c_1\epsilon a) \geq Q(1) - c_2 \epsilon a$ by Taylor's theorem where $c_2$ is a problem independent constant. Using these bounds, we can bound the inverse of $p_{k}(a)$ from above for small enough $\epsilon \ll 1$, small enough $\epsilon a \ll 1$, large enough $a \gg 1$ and large enough $\epsilon a^2 \gg 1$ as
\begin{align}
    \frac{1}{p_{k}(a)} &\leq \frac{1}{Q(1) - c_2 \epsilon a  -Q(1+a) } \frac{1}{1-\frac{\Gamma(k/2,\,\epsilon a^2/2)}{\Gamma(k/2)}} \nn \\
    &=  \frac{1}{Q(1)} \frac{1}{\pr{1 - c_2 \epsilon a  -Q(1+a)}\pr{1-\frac{\Gamma(k/2,\,\epsilon a^2/2)}{\Gamma(k/2)}} } \nn \\
    &\leq \frac{1}{Q(1)} \br{1 + 2C\pr{ c_2 \epsilon a + Q(1+a) + \frac{\Gamma(k/2,\,\epsilon a^2/2)}{\Gamma(k/2)}}} \label{step24}
\end{align}
where we used the Taylor expansion $\frac{1}{1-x} = \sum_{k=0}^{\infty} x^k \leq 1+Cx$ for small enough $x\ll 1$ with $C>0$ being a problem independent constant.
\\ \\ 
The assumption $\epsilon a^2 \gg 1$ can be used to write the asymptotic expansion of incomplete Gamma function $\Gamma(k/2,\,\epsilon a^2/2) = (\epsilon a^2/2)^{k/2-1} e^{-\epsilon a^2/2}\br{1 + O\pr{(\epsilon a^2/2)^{-1} }}$. Noting that the Q function is always bounded as $Q(1+a) \leq \frac{e^{-\frac{(1+a)^2}{2}}}{\sqrt{2\pi} (1+a)}$, we claim that choosing $\epsilon = \frac{k}{2ea^{1+1/2}}$, for $\alpha \geq c^{\prime \prime } k$ with a constant $c^{\prime \prime }>0$ guarantees that $\epsilon a = \frac{k}{2e} a^{-1/2} \ll 1$ and $\epsilon a^2 = \frac{k}{2e} a^{1-1/2} \gg 1$. Therefore, the upper bound \eqref{step24} is valid for $\alpha \geq c^{\prime \prime } k$. Furthermore, the term $\epsilon a$ decays slower than both $Q(1+a)$ and $\frac{\Gamma(k/2,\,\epsilon a^2/2)}{\Gamma(k/2)}$ and thus $\epsilon a$ dominates as
\begin{align}
    \frac{1}{p_{k}(a)} &\leq \frac{1}{Q(1)} \pr{1 + \frac{Ck}{2e }a^{-1/2}} \nn
\end{align}
for a problem independent constant $C>0$.
\end{proof}
Based on Lemma~\ref{thm:asymptotics}, the integrals in \eqref{step22} and \eqref{step23} are asymptotically constant if both $a_{1,T}$ and $a_{2,T}$ are asymptotically large enough. This can be achieved if $a_{1,T} =\frac{\Tw}{\sqrt{T}\log{T}}\frac{\norm[F]{\nabla L_* }}{Cr_*}  = \omega(1)$ for singular $A_{c,*}$ and $a_{2,T} =\sqrt{\frac{\Tw}{\log{T}}}\frac{\min(\sigma_{\min,*},\, \epsilon_* r_* /2)}{\sqrt{C}r_*}= \omega(1)$ for non-singular $A_{c,*}$. In other words, choosing $\Tw = n^2\omega(\sqrt{T}\log{T})$ for singular $A_{c,*}$ and $ \Tw = n^2 \omega(\log{T})$ for non-singular $A_{c,*}$ yields the desired bound
\begin{align}
    p_t^{\opt} \geq \frac{Q(1)}{1+ o(1)}. \nn
\end{align}
for $T_r<t\leq T$ for large enough $T$. Combined with the upper $\Tw = O((\sqrt{T})^{1+o(1)})$, the proposed choices of $\Tw$ satisfy the asymptotic conditions.

%% file: appendix/regret_decomposition.tex
Denote the optimal expected average cost of an LQR system $\Theta$ with process noise covariance $W$ by $J_*(\Theta, W) = \tr(P(\Theta)W)$. Note that during the initial exploration period, we have that $u_t = \Bar{u}_t + \nu_t$ for $t\leq T_w$ and after the initial exploration, we have that $u_t = \Bar{u}_t$ for $t > T_w$ where we denote by $\Bar{u}_t \defeq K(\Tilde{\Theta}_t) x_t$ the optimal control action assuming the system $\Tilde{\Theta}_t$. Since initial exploration period injects independent random perturbations through the optimal control input, $\Bar{u}_t$, for sampled system, $\Tilde{\Theta}_t$, the state dynamics can be reformulated in order to take the external perturbations into account by adding it to the process noise:
\begin{align}
    x_{t+1} = A_* x_{t} + B_* \Bar{u}_t + \zeta_t,
\end{align}
where $\Bar{u}_t = K(\Tilde{\Theta}_t) x_t$, $\zeta_t = B_* \nu_t +w_t$ for $t \leq T_w$, and $\zeta_t = w_t$ for $t > T_w$. We can write the regret explicitly as
\begin{align}
    R_T &= \sum\nolimits_{t=0}^{T}\cl{x_t^\tp Q x_t + u_t^\tp R u_t - J_*(\Theta_*, \sigma_w^2 I)} = R_{T_w}^{\text{exp}} + R_T^{\text{noexp}},
\end{align}
where 
\begin{align}
    R_{T_w}^{\text{exp}}&\defeq \sum\nolimits_{t=0}^{T_w} \pr{2\Bar{u}_t ^\tp R \nu_t + \nu_t^\tp R \nu_t},\quad \& \quad   R_T^{\text{noexp}}\defeq \sum\nolimits_{t=0}^{T}\cl{x_t^\tp Q x_t + \Bar{u}_t^\tp R \Bar{u}_t - J_*(\Theta_*, \sigma_w^2 I)} \nn
\end{align}
Since $E_s \subset E_t$ for any $0\leq s \leq t$, we have that
\begin{align}
   R_T^{\text{noexp}}\indic_{E_T} &= \sum_{t=0}^{T}\cl{x_t^\tp Q x_t \!+\! \Bar{u}_t^\tp R \Bar{u}_t -J_*(\Theta_*, \sigma_w^2 I)}\indic_{E_T} \nonumber\\
   &\leq \sum_{t=0}^{T}\cl{x_t^\tp Q x_t + \Bar{u}_t^\tp R \Bar{u}_t - J_*(\Theta_*, \sigma_w^2 I)}\indic_{E_t}, \label{reg_noexp}\\
   R_{T_w}^{\text{exp}}\indic_{E_T} &= \sum\nolimits_{t=0}^{T_w} \pr{2\Bar{u}_t^\tp R \nu_t + \nu_t^\tp R \nu_t}\indic_{E_T} \leq \sum\nolimits_{t=0}^{T_w} \pr{2\Bar{u}_t^\tp R \nu_t + \nu_t^\tp R \nu_t}\indic_{E_t}.
\end{align}
From Bellman optimality equations \citep{bertsekas1995dynamic}, we obtain 
\begin{align*}
    J_*&(\Tilde{\Theta}_t, \Cov [\zeta_t]) + x_t^\tp P(\Tilde{\Theta}_t)x_t \\
    &= \min_{u} \cl{x_t^\tp Q x_t + u^\tp R u + \E \br{(\Tilde{A}_{t}x_t + \Tilde{B}_{t}u  + \zeta_t)^\tp P(\Tilde{\Theta}_t) (\Tilde{A}_{t}x_t + \Tilde{B}_{t}u  + \zeta_t) \, \big| \, \clf{F}_t } }, \\
    &= x_t^\tp Q x_t + \bar{u}_t^\tp R \bar{u}_t + \E \br{(\Tilde{A}_{t}x_t + \Tilde{B}_{t}\bar{u}_t  + \zeta_t)^\tp P(\Tilde{\Theta}_t) (\Tilde{A}_{t}x_t + \Tilde{B}_{t}\bar{u}_t  + \zeta_t) \, \big| \, \clf{F}_t }, \\
    &= x_t^\tp Q x_t + \bar{u}_t^\tp R \bar{u}_t + \E \br{(\Tilde{A}_{t}x_t + \Tilde{B}_{t}\bar{u}_t )^\tp P(\Tilde{\Theta}_t) (\Tilde{A}_{t}x_t + \Tilde{B}_{t}\bar{u}_t ) \, | \, \clf{F}_t } + \E \br{\zeta_t^\tp P(\Tilde{\Theta}_t) \zeta_t  \, \big| \, \clf{F}_t }, \\
    &= x_t^\tp Q x_t + \bar{u}_t^\tp R \bar{u}_t + \E \br{(\Tilde{A}_{t}x_t + \Tilde{B}_{t}\bar{u}_t )^\tp P(\Tilde{\Theta}_t) (\Tilde{A}_{t}x_t + \Tilde{B}_{t}\bar{u}_t ) \, \big| \, \clf{F}_t } \\
     &\quad  + \E \br{x_{t+1}^\tp P(\Tilde{\Theta}_t) x_{t+1} \, | \, \clf{F}_t } - \E \br{(A_{*}x_t + B_{*}\bar{u}_t )^\tp P(\Tilde{\Theta}_t) (A_{*}x_t + B_{*}\bar{u}_t ) \, \big| \, \clf{F}_t }, \\
    &= x_t^\tp Q x_t + \bar{u}_t^\tp R \bar{u}_t + \E \br{x_{t+1}^\tp P(\Tilde{\Theta}_t) x_{t+1} \, \big| \, \clf{F}_t } \\
     &\quad  +(\Tilde{A}_{t}x_t + \Tilde{B}_{t}\bar{u}_t )^\tp P(\Tilde{\Theta}_t) (\Tilde{A}_{t}x_t + \Tilde{B}_{t}\bar{u}_t ) - (A_{*}x_t + B_{*}\bar{u}_t )^\tp P(\Tilde{\Theta}_t) (A_{*}x_t + B_{*}\bar{u}_t ), \\
     &= x_t^\tp Q x_t + \bar{u}_t^\tp R \bar{u}_t + \E \br{x_{t+1}^\tp P(\Tilde{\Theta}_t) x_{t+1} \, \big| \, \clf{F}_t } + \bar{z}_t^\tp  \Tilde{\Theta}_{t} P(\Tilde{\Theta}_t) \Tilde{\Theta}_{t}^\tp \bar{z}_t - \bar{z}_t^\tp  \Theta_{*} P(\Tilde{\Theta}_t) \Theta_{*}^\tp \bar{z}_t,
\end{align*}
where $\Bar{z}_t^\tp = [x_t^\tp, \Bar{u}_t^\tp]$. Rearranging the terms and subtracting the optimal expected average cost of the true system, we obtain the following for each term in \eqref{reg_noexp},
\begin{align*}
    &\cl{x_t^\tp Q x_t + \Bar{u}_t^\tp R \Bar{u}_t - J_*(\Theta_*, \sigma_w^2 I)}\indic_{E_t} \\
    &= \cl{ J_*(\Tilde{\Theta}_t, \Cov [\zeta_t]) -J_*(\Theta_*, \sigma_w^2 I)}\indic_{E_t} + \cl{ \bar{z}_t^\tp \Theta_* P(\Tilde{\Theta}_t) \Theta_*^\tp \bar{z}_t - \bar{z}_t^\tp \Tilde{\Theta}_t P(\Tilde{\Theta}_t) \Tilde{\Theta}_t^\tp \bar{z}_t}\indic_{E_t}, \\
    &\quad + x_t^\tp P(\Tilde{\Theta}_t)x_t\indic_{E_{t}}  - \E\br{x_{t+1}^\tp P(\Tilde{\Theta}_{t})x_{t+1} \indic_{E_{t}} \,\big|\,\clf{F}_t }.
\end{align*}
Note that, $\indic_{E_t}\indic_{E_{t+1}} = \indic_{E_{t+1}}$ since $E_{t+1} \subset E_{t}$. Since $P(\Tilde{\Theta}_t) \succ 0$, we obtain 
\begin{align*}
    &\E\br{x_{t+1}^\tp P(\Tilde{\Theta}_{t})x_{t+1} \indic_{E_{t}} \,\big|\,\clf{F}_t } \\
    &= \E\br{x_{t+1}^\tp P(\Tilde{\Theta}_{t})x_{t+1} \indic_{E_{t}} \pr{\indic_{E_{t+1}} + \indic_{E^c_{t+1}}} \,\big|\,\clf{F}_t },\\
    &= \E\br{x_{t+1}^\tp P(\Tilde{\Theta}_{t})x_{t+1} \indic_{E_{t+1}} \,\big|\,\clf{F}_t } \!+\! \E\br{x_{t+1}^\tp P(\Tilde{\Theta}_{t})x_{t+1} \indic_{E_{t}} \indic_{E^c_{t+1}} \,\big|\,\clf{F}_t },\\
    &\geq \E\br{x_{t+1}^\tp P(\Tilde{\Theta}_{t})x_{t+1} \indic_{E_{t+1}} \,\big|\,\clf{F}_t }, \\
    &=  \E\br{x_{t+1}^\tp \pr{P(\Tilde{\Theta}_{t}) - P(\Tilde{\Theta}_{t+1})}x_{t+1} \indic_{E_{t+1}} \,\big|\,\clf{F}_t } + \E\br{x_{t+1}^\tp P(\Tilde{\Theta}_{t+1})x_{t+1} \indic_{E_{t+1}} \,\big|\,\clf{F}_t }.
\end{align*}
Therefore, 
\begin{align}
    \cl{x_t^\tp Q x_t + \Bar{u}_t^\tp R \Bar{u}_t - J_*(\Theta_*, \sigma_w^2 I)}&\indic_{E_t} \leq \cl{ J_*(\Tilde{\Theta}_t, \Cov [\zeta_t]) -J_*(\Theta_*, \sigma_w^2 I)}\indic_{E_t}, \nn \\
    &+ \cl{ \bar{z}_t^\tp \Theta_* P(\Tilde{\Theta}_t) \Theta_*^\tp \bar{z}_t - \bar{z}_t^\tp \Tilde{\Theta}_t P(\Tilde{\Theta}_t) \Tilde{\Theta}_t^\tp \bar{z}_t}\indic_{E_t}, \nonumber \\
    &+ \cl{x_t^\tp P(\Tilde{\Theta}_t)x_t \indic_{E_t}  - \E\br{x_{t+1}^\tp P(\Tilde{\Theta}_{t+1})x_{t+1} \indic_{E_{t+1}} \, \big|\,\clf{F}_t } }, \nonumber \\
    &+ \E\br{x_{t+1}^\tp \pr{ P(\Tilde{\Theta}_{t+1}) - P(\Tilde{\Theta}_{t})}x_{t+1} \indic_{E_{t+1}} \,\big|\,\clf{F}_t } \label{eq:regret_diff}
\end{align}
Notice that $\Cov[\zeta_t] = \sigma_{\nu}^2 B_* B_* ^\tp + \sigma_w^2 I$ for $t\leq T_w$ and  $\Cov[\zeta_t] = \sigma_w^2 I$ for $t> T_w$ and therefore 
\begin{align}
    J_*(\Tilde{\Theta}_t, \Cov [\zeta_t])  = \Tr(P(\Tilde{\Theta}_t) \Cov [\zeta_t]) = 
    \begin{cases} 
      \sigma_{\nu}^2\Tr(P(\Tilde{\Theta}_t) B_* B_* ^\tp) + \sigma_{w}^2\Tr(P(\Tilde{\Theta}_t)) & t\leq T_w \\
      \sigma_{w}^2\Tr(P(\Tilde{\Theta}_t)) & t> T_w
   \end{cases}
\end{align}
Summing the terms in \eqref{eq:regret_diff} upto time T and adding the $R_{T_w}^{\text{exp}}$ term, we obtain 
\begin{align}
     R_T\indic_{E_T} = R_{T_w}^{\text{exp}}\indic_{E_T} +  R_T^{\text{noexp}}\indic_{E_T} &\leq R_{T_w}^{\text{exp,1}} + R_{T_w}^{\text{exp,2}} + R_{T}^{\text{TS}} + R_{T}^{\text{RLS}} + R_{T}^{\text{mart}} + R_{T}^{\text{gap}}
\end{align}
where
\begin{align}
    R_{T_w}^{\text{exp,1}} &= \sum\nolimits_{t=0}^{T_w} \pr{2\Bar{u}_t^\tp R \nu_t + \nu_t^\tp R \nu_t}\indic_{E_t}, \\
    R_{T_w}^{\text{exp,2}} &= \sum\nolimits_{t=0}^{T_w} \sigma_{\nu}^2\Tr(P(\Tilde{\Theta}_t) B_* B_* ^\tp)\indic_{E_t}, \\
    R_{T}^{\text{TS}} &= \sum\nolimits_{t=0}^{T}\cl{ J_*(\Tilde{\Theta}_t, \sigma_w^2 I) -J_*(\Theta_*, \sigma_w^2 I)}\indic_{E_t},  \label{def:R_TS}\\
    R_{T}^{\text{RLS}} &= \sum\nolimits_{t=0}^{T}\cl{ \bar{z}_t^\tp \Theta_* P(\Tilde{\Theta}_t) \Theta_*^\tp \bar{z}_t - \bar{z}_t^\tp \Tilde{\Theta}_t P(\Tilde{\Theta}_t) \Tilde{\Theta}_t^\tp \bar{z}_t}\indic_{E_t}, \label{def:R_RLS} \\
    R_{T}^{\text{mart}} &= \sum\nolimits_{t=0}^{T}\cl{x_t^\tp P(\Tilde{\Theta}_t)x_t \indic_{E_t}  - \E\br{x_{t+1}^\tp P(\Tilde{\Theta}_{t+1})x_{t+1} \indic_{E_{t+1}} \, \big|\,\clf{F}_t } }, \label{def:R_mart} \\
    R_{T}^{\text{gap}} &=  \sum\nolimits_{t=0}^{T} \E\br{x_{t+1}^\tp \pr{ P(\Tilde{\Theta}_{t+1}) - P(\Tilde{\Theta}_{t})}x_{t+1} \indic_{E_{t+1}} \,\big|\,\clf{F}_t }. \label{def:R_gap}
\end{align}

In the next section, we will give upper bounds to each term.

%% file: appendix/regret_analysis.tex
In this section, we bound each term in regret decomposition individually. In particular, $R_{\Tw}^{\text{exp}}$ is studied in Appendix \ref{subsec:reg_explore}, $R_{T}^{\text{RLS}}$ is studied in Appendix~\ref{subsec:reg_rls}, $R_{T}^{\text{mart}}$ in Appendix~\ref{subsec:reg_mart}, $R_{T}^{\text{TS}}$ in Appendix~\ref{subsec:reg_ts}, and $R_{T}^{\text{gap}}$ in Appendix~\ref{subsec:reg_gap}. Finally, in Appendix~\ref{subsec:reg_all}, we combine these results to obtain the regret upper bound of \alg as stated in Theorem \ref{reg:exp_s}.

\subsection{Bounding \texorpdfstring{$R_{T_w}^{\text{exp,1}}$}{Rexp1} and \texorpdfstring{$R_{T_w}^{\text{exp,2}}$}{Rexp2}} \label{subsec:reg_explore}
The following gives an upper bound on the regret attained due to isotropic perturbations in the TS with improved exploration phase of \alg. 

\begin{lemma}[Direct Effect of Improved Exploration on Regret] \label{lem:exp_effect}
The following holds with probability at least $1-\delta$, 
\begin{equation}
    R_{T_w}^{\text{exp,1}} = \sum_{t=0}^{T_w} \cl{2 \Bar{u}_t^\tp R \nu_t  + \nu_t^\tp R \nu_t} \indic_{E_t} \leq d \sigma_\nu \sqrt{B_\delta} + d\norm{R}\sigma_\nu^2 \pr{T_w + \sqrt{T_w} \log \frac{4d T_w}{\delta}\sqrt{\log \frac{4}{\delta}} } \nn
\end{equation}
where 
\begin{equation*}
B_\delta =  8 \left(1 + T_w \kappa^2 \|R\|^2 (n+d)^{2(n+d)} \right) \log \left( \frac{4d}{\delta} \left(1 + T_w \kappa^2 \|R\|^2 (n+d)^{2(n+d)} \right)^{1/2} \right). 
\end{equation*}
Furthermore, we have $ R_{T_w}^{\text{exp,2}} \leq \sigma_{\nu}^2 D \,\norm[F]{B_*}^2 T_w$.
\end{lemma}

\begin{proof}
 First we will study $R_{T_w}^{\text{exp,1}}$. Let $q_{t}^\tp =  \bar{u}_{t}^\tp R \indic_{E_t}$. The first term can be written as 
\[
2\sum_{t=0}^{T_w} \sum_{i=1}^{d} q_{t, i} \nu_{t, i} = 2\sum_{i=1}^{d}  \sum_{t=0}^{T_w} q_{t, i} \nu_{t, i}
\]
Let $M_{t, i}=\sum_{k=0}^{t} q_{k, i} \nu_{k, i} .$ By Theorem \ref{selfnormalized} on some event $G_{\delta, i}$ that holds with probability at least $1-\delta /(2 d),$ for any $t \geq 0$,  

\begin{align*}
	M_{t, i}^2 &\leq 2 \sigma_\nu^2  \left(1 + \sum_{k=0}^{t} q_{k,i}^2\right) \log \left(\frac{2d}{\delta} \left(1 + \sum_{k=0}^{t} q_{k,i}^2\right)^{1/2} \right)   \\
\end{align*}

 Note that $\| q_k \| = \|R  \bar{u}_{t} \| \indic_{E_t}  \leq \kappa \|R\| (n+d)^{n+d} $, thus $q_{k,i} \leq \kappa \|R\| (n+d)^{n+d} $. Using union bound we get, for probability at least $1-\frac{\delta}{2}$, 

\begin{align}
	&\sum_{t=0}^{T_w} 2 \nu_t^\tp R\Bar{u}_t  \indic_{R_t} \leq  \nonumber \\
	&d \sqrt{ 8 \sigma_\nu^2  \left(1 + T_w \kappa^2 \|R\|^2 (n+d)^{2(n+d)} \right) \log \left( \frac{4d}{\delta} \left(1 + T_w \kappa^2 \|R\|^2 (n+d)^{2(n+d)} \right)^{1/2} \right) }  \label{explore_reg_1}
\end{align}

Let $W = \sigma_\nu  \sqrt{2d\log \frac{4dT_w}{\delta}}$. Define $\Psi_{t}=\nu_{t}^{\tp} R\nu_{t}-\mathbb{E}\left[\nu_{t}^{\tp} R\nu_{t} | \mathcal{F}_{t-1}\right]$ and its truncated version $\tilde{\Psi}_{t}=\Psi_{t} \mathbb{I}_{\left\{\Psi_{t} \leq 2 D W^{2}\right\}}$. 

\begin{align*}
	\Pr \bigg( \sum_{t=1}^{T_w} \Psi_t &>  2\|R\| W^2 \sqrt{2T_w \log \frac{4}{\delta}}\bigg) \leq \\
	 &\Pr \left( \max_{1 \leq t \leq T_w} \Psi_t >  2\|R\| W^2 \right) + \Pr \left( \sum_{t=1}^{T_w} \tilde{\Psi}_t >  2\|R\| W^2 \sqrt{2T_w \log \frac{4}{\delta}}\right) 
\end{align*}

Using Lemma \ref{subgauss lemma} with union bound and Theorem \ref{Azuma}, summation of terms on the right hand side is bounded by $\delta/2$. Thus, with probability at least $1-\delta/2$,

\begin{equation} 
	\sum_{t=0}^{T_w} \nu_t^\tp R \nu_t \leq d T_w \sigma_\nu^2 \|R\| + 2\|R\| W^2 \sqrt{2T_w \log \frac{4}{\delta}} .\label{explore_reg_2}
\end{equation}

Combining \eqref{explore_reg_1} and \eqref{explore_reg_2} gives the statement of lemma for the regret of external exploration noise. Next, we consider $R_{T_w}^{\text{exp,2}}$. Due to rejection sampling $\mathcal{R}_{\mathcal{S}}(\cdot)$, a new model sample is redrawn until it lies on the set $\clf{S}$ at every TS step, \textit{i.e.}, $\Tilde{\Theta}_t \in \clf{S}$ for every time step $t\geq 0$. By Assumption \ref{asm_stabil}, we have $\norm[F]{P(\Tilde{\Theta}_t)} \leq D =\Bar{\alpha} \gamma^{-1} \kappa^2(1+\kappa^2)  $. Thus, we have
\begin{align}
    R_{T_w}^{\text{exp,2}} &= \sum_{t=0}^{T_w} \sigma_{\nu}^2\Tr(P(\Tilde{\Theta}_t) B_* B_* ^\tp)\indic_{E_t}, \nn  \\
    &\leq \sum_{t=0}^{T_w} \sigma_{\nu}^2 \norm[F]{P(\Tilde{\Theta}_t)} \norm[F]{B_*}^2 \indic_{E_t}, \nn \\
    &\leq \sigma_{\nu}^2 D \,\norm[F]{B_*}^2 \sum_{t=0}^{T_w}  \indic_{E_t} \leq \sigma_{\nu}^2 D \,\norm[F]{B_*}^2 T_w. \label{exp_reg_2_Tw}
\end{align}
\end{proof}

\subsection{Bounding \texorpdfstring{$R_{T}^{\text{RLS}}$}{RRLS}} \label{subsec:reg_rls}

Bounding this term is achieved by manipulating the similar bounds in \citet{abeille2017thompson,abbasi2011lqr} to our setting and TS algorithm. We first have the following result from regularized least squares estimate.

\begin{lemma} \label{lemma_RLS_H0}
On the event of $E_T$, for $X_s = \frac{(12\kappa^2+ 2\kappa\sqrt{2})\sigma_w}{\gamma}  \sqrt{2n\log(n(T-\Tw)/\delta)}$, we have,
\begin{align*}
    \sum_{t=0}^T \|(\tts - \ttt_t)^\tp \!z_t \|^2 &\leq \!
    2(\beta_T(\delta) + \upsilon_T(\delta))^2  \Bigg(  \left(1+ \frac{(1+\kappa^2)(n+d)^{2(n+d)}}{\mu}\right)^{\tau_0+1} \!\!\!\!\!\!\!\!\! \log \frac{\det(V_{T_{r}})}{\det(\mu I)} \\ &\qquad\qquad\qquad\qquad\qquad+ \left(1+ \frac{(1+\kappa^2)X_s^2}{\mu}\right)^{\tau_0+1} \!\!\!\!\!\!\!\log \frac{\det(V_{T})}{\det(V_{T_{r}})} \Bigg). 
\end{align*}
\end{lemma}
\begin{proof}
 Let $ \tau \leq t$ be the last time step before $t$, when the policy was updated. Using Cauchy-Schwarz inequality, we have:
 \begin{equation} \label{eq:rls_cauchy}
     \sum_{t=0}^T \|(\tts - \ttt_t)^\tp \!z_t \|^2 \leq  \sum_{t=0}^{T}\|V_{t}^{\frac{1}{2}}(\ttt_{t}-\tts) \|^2 \|z_{t} \|^2_{V_t^{-1}} \leq  \sum_{t=0}^{T} \frac{\det(V_t)}{\det(V_\tau)} \|V_{\tau}^{\frac{1}{2}}(\ttt_{\tau}-\tts) \|^2  \|z_{t} \|^2_{V_t^{-1}}.
 \end{equation}
 Note that $t - \tau \leq \tau_0$ due to policy update rule. Moreover, we have 
 \begin{equation*}
     \det(V_t) = \det(V_\tau) \prod_{i=0}^{t-\tau} (1+ \|z_t\|^2_{V_{t-i}^{-1}}) \leq \det(V_\tau) \left(1+ \frac{\|z_t\|^2}{\mu} \right)^{\tau_0}.
 \end{equation*}
 Combining this with \eqref{eq:rls_cauchy}, on the event of $E_T$, for $t\leq T_r$, we have: 
 \begin{align}
     \sum_{t=0}^{T_r} \|(\tts - \ttt_t)^\tp \!z_t \|^2 &\leq  \sum_{t=0}^{T_r} \left(1+ \frac{(1+\kappa^2)(n+d)^{2(n+d)}}{\mu}\right)^{\tau_0} \|V_\tau^{1/2} (\ttt_{\tau}-\tts) \|^2 \|z_t\|^2_{V_t^{-1}} \\ 
    &\leq  \sum_{t=0}^{T_r} \left(1+ \frac{(1+\kappa^2)(n+d)^{2(n+d)}}{\mu}\right)^{\tau_0} (\beta_T(\delta) + \upsilon_T(\delta))^2   \|z_t\|^2_{V_t^{-1}}, \label{rls_squared_upsilon} \\
    &\!\!\! \hspace{-7em} \leq \frac{2(1+\kappa^2)(n+d)^{2(n+d)}}{\mu} \left(1+ \frac{(1+\kappa^2)(n+d)^{2(n+d)}}{\mu}\right)^{\tau_0}\!\!\!\!\! (\beta_T(\delta) + \upsilon_T(\delta))^2  \log \left(\frac{\det(V_{T_r})}{\det(\mu I)} \right) \label{rls_boundlast}
\end{align}
 where in \eqref{rls_squared_upsilon} we used the fact that on the event of $E_T$, using triangle inequality, we have $\|\ttt_{\tau}-\tts\|_{V_\tau} \leq \| \ttt_{\tau}-\tth_{\tau} \|_{V_\tau} + \|\tth_{\tau} - \tts\|_{V_{\tau}} \leq \upsilon_\tau(\delta) + \beta_{\tau}(\delta) \leq \upsilon_T(\delta) + \beta_{T}(\delta)$ and in \eqref{rls_boundlast} we used used the upper bound of $\| z_t\|_{V_t^{-1}}$ to utilize Lemma 10 of \citet{abbasi2011lqr}. Similarly, on the even of $E_t$, for $t>T_r$, we get:
 \begin{align*}
     \sum_{t=T_r+1}^{T} \|(\tts - \ttt_t)^\tp \!z_t \|^2  &\leq \frac{2(1+\kappa^2)X_s^2}{\mu} \left(1+ \frac{(1+\kappa^2)X_s^2}{\mu}\right)^{\tau_0} \!\!\!\!\! (\beta_T(\delta) + \upsilon_T(\delta))^2  \log \left(\frac{\det(V_{T})}{\det(V_{T_r})} \right) 
\end{align*}
\end{proof}

\begin{lemma}[Bounding $R_{T}^{\text{RLS}}$ for \alg]\label{R_RLS}
Let $R_{T}^{\text{RLS}}$ be as defined by \eqref{def:R_RLS}. Under the event of $E_T$, setting $\mu = (1+\kappa^2)X_s^2$, we have
	\begin{align*}
	 \left|R_{T}^{\text{RLS}}\right| = \tilde{O} \left((n+d)^{(\tau_0+2)(n+d) + 1.5} \sqrt{n} \sqrt{T_{r}} + (n+d) n \sqrt{T-T_{r}} \right).
	\end{align*}
\end{lemma}

\begin{proof}
 
 \begin{align} 
 \left|R_{T}^{\text{RLS}}\right| &\leq \sum_{t=0}^{T}\left|\left\|P(\ttt_t)^{\frac{1}{2}} \ttt_{t}^{\tp} z_{t}\right\|^{2}-\left\|P(\ttt_t)^{\frac{1}{2}} \tts^{\tp} z_{t}\right\|^{2}\right| \label{R_RLS_first_tri} \\
 =& \sum_{t=0}^{T_{r}}\left|\left\|P(\ttt_t)^{\frac{1}{2}} \ttt_{t}^{\tp} z_{t}\right\|^{2}-\left\|P(\ttt_t)^{\frac{1}{2}} \tts^{\tp} z_{t}\right\|^{2}\right| + \sum_{t=T_{r}}^{T}\left|\left\|P(\ttt_t)^{\frac{1}{2}} \ttt_{t}^{\tp} z_{t}\right\|^{2}-\left\|P(\ttt_t)^{\frac{1}{2}} \tts^{\tp} z_{t}\right\|^{2}\right| \nonumber \\ 
\leq& \bigg(\!\sum_{t=0}^{T_{r}}\left(\left\|P(\ttt_t)^{\frac{1}{2}} \ttt_{t}^{\tp} z_{t}\right\|\!-\!\left\|P(\ttt_t)^{\frac{1}{2}} \tts^{\tp} z_{t}\right\|\right)^{2\!\!}\bigg)^{\frac{1}{2}} \!\!\! \bigg(\!\sum_{t=0}^{T_{r}}\left(\left\|P(\ttt_t)^{\frac{1}{2}} \ttt_{t}^{\tp} z_{t}\right\|\!+\!\left\|P(\ttt_t)^{\frac{1}{2}} \tts^{\tp} z_{t}\right\|\right)^{2}\!\!\bigg)^{\frac{1}{2}} \nonumber \\
+&  \bigg(\!\sum_{t=T_{r}}^{T}\!\left(\left\|P(\ttt_t)^{\frac{1}{2}} \ttt_{t}^{\tp} z_{t}\right\|\!-\!\left\|P(\ttt_t)^{\frac{1}{2}} \tts^{\tp} z_{t}\right\|\right)^{2\!\!}\bigg)^{\frac{1}{2}} \!\!\! \bigg(\!\sum_{t=T_{r}}^{T}\!\left(\left\|P(\ttt_t)^{\frac{1}{2}} \ttt_{t}^{\tp} z_{t}\right\|\!+\!\left\|P(\ttt_t)^{\frac{1}{2}} \tts^{\tp} z_{t}\right\|\right)^{2}\!\!\bigg)^{\frac{1}{2}} \label{R_RLS_cauchy} \\
\leq& \left(\sum_{t=0}^{T_{r}}\left\|P(\ttt_t)^{\frac{1}{2}}\left(\ttt_{t}-\tts\right)^{\tp} z_{t}\right\|^{2}\right)^{\frac{1}{2}} \!\!\! \left(\sum_{t=0}^{T_{r}}\left(\left\|P(\ttt_t)^{\frac{1}{2}} \ttt_{t}^{\tp} z_{t}\right\|+\left\|P(\ttt_t)^{\frac{1}{2}} \tts^{\tp} z_{t}\right\|\right)^{2}\right)^{\frac{1}{2}} \nonumber \\
&+ \left(\sum_{t=T_{r}}^{T}\left\|P(\ttt_t)^{\frac{1}{2}}\left(\ttt_{t}-\tts\right)^{\tp} z_{t}\right\|^{2}\right)^{\frac{1}{2}} \!\!\! \left(\sum_{t=T_{r}}^{T}\left(\left\|P(\ttt_t)^{\frac{1}{2}} \ttt_{t}^{\tp} z_{t}\right\|+\left\|P(\ttt_t)^{\frac{1}{2}} \tts^{\tp} z_{t}\right\|\right)^{2}\right)^{\frac{1}{2}} \label{R_RLS_second_tri} 
\end{align}
where \eqref{R_RLS_first_tri} and \eqref{R_RLS_second_tri} follow from triangle inequality, and \eqref{R_RLS_cauchy} follows from Cauchy Schwarz inequality. Note that for $t\leq T_r$, we have $\|z_t\|^2 \leq (1+\kappa^2)(n+d)^{2(n+d)}$ and for $t > T_r$ we have  $\|z_t\|^2 \leq (1+\kappa^2)X_s^2$. Moreover, since $\ttt$ belongs to $\mathcal{S}$ by construction of the rejection sampling, we get
\begin{align}
    &\left|R_{T}^{\text{RLS}}\right| \leq \left(D\sum_{t=0}^{T_{r}}\left\|\left(\ttt_{t}-\tts\right)^{\tp} z_{t}\right\|^{2}\right)^{\frac{1}{2}} \sqrt{4T_{r} D (1+\kappa^2) S^2 (n+d)^{2(n+d)}} \nonumber \\ 
    &\quad\quad \quad \quad \quad  + \left(D\sum_{t=T_{r}}^{T}\left\|\left(\ttt_{t}-\tts\right)^{\tp} z_{t}\right\|^{2}\right)^{\frac{1}{2}} \sqrt{4(T-T_{r})D(1+\kappa^2)S^2 X_s^2} \nonumber \\
    &\!\leq\! \frac{\sqrt{8T_r}DS(1\!+\!\kappa^2)(n\!+\!d)^{2(n+d)} (\beta_T(\delta) \!+\! \upsilon_T(\delta))}{\sqrt{\mu}} \bigg(1\!+\! \frac{(1\!+\!\kappa^2)(n\!+\!d)^{2(n+d)}}{\mu}\bigg)^{\frac{\tau_0}{2}}\!\!\!  \sqrt{\log \big(\frac{\det(V_{T_r})}{\det(\mu I)} \big) }  \nonumber   \\ 
    &\quad \!+\!\frac{\sqrt{8(T\!-\!T_r)}DS(\!1\!+\kappa^2)X_s^2\left(\beta_T(\delta) + \upsilon_T(\delta)\right)}{\sqrt{\mu}} \left(1\!+\! \frac{(1\!+\!\kappa^2)X_s^2}{\mu}\right)^{\!\!\frac{\tau_0}{2}} \!\!\!\!\!  \sqrt{\log \left(\frac{\det(V_{T})}{\det(V_{T_r})} \right) } \label{regret_rls_final}
\end{align}
From Lemma 10 of \citet{abbasi2011lqr}, we have that $\log (\frac{\det(V_{T_r})}{\det(\mu I)}) \!\leq\! (n\!+\!d)\log(1+\frac{T_r(1+\kappa^2)(n+d)^{2(n+d)}}{\mu(n+d)})$ and $\log (\frac{\det(V_{T})}{\det(V_{T_r})} ) \leq (n+d)\log ( 1 + \frac{T_{r}(1+\kappa^2)(n+d)^{2(n+d)} + (T-T_{r})X_s^2}{\mu(n+d)} ) $. 
\\ \\
After inserting these quantities into  \eqref{regret_rls_final}, we have the dimension dependency of $(n+d)^{2(n+d)} \times \sqrt{n(n+d)} \times (n+d)^{(n+d)\tau_0} \times  (n+d)$ on the first term where $\sqrt{n(n+d)}$ is due to $\beta_T(\delta) + \upsilon_T(\delta)$. For the second term, for large enough $T$, we have the dimension dependency of $n \times \sqrt{n(n+d)} \times n^{(\tau_0/2)} \times \sqrt{n+d}$, where $n$ comes from $X_s^2$. Thus, we achieve the following bound for $\left|R_{T}^{\text{RLS}}\right|$:
\begin{equation*}
    \left|R_{T}^{\text{RLS}}\right| = \tilde{O} \left((n+d)^{(\tau_0+2)(n+d) + 1.5} \sqrt{n} \sqrt{T_{r}} + (n+d) n^{1.5+\tau_0/2} \sqrt{T-T_{r}} \right).
\end{equation*}
With the choice of $\mu = (1+\kappa^2)X_s^2$, the dependency of $n^{(\tau_0/2)}$ on the second term can be converted to a scalar multiplier of $\sqrt{2}^{\tau_0}$ and reduces the dependency of $X_s^2$ to $X_s$, which gives the advertised bound. 
\end{proof}

\subsection{Bounding \texorpdfstring{$R_{T}^{\text{mart}}$}{Rmart}} \label{subsec:reg_mart}

Notice that this term is very similar to corresponding term in \citet{abbasi2011lqr,abeille2018improved}, besides the difference of early improved exploration. Following the same analysis, while including the effect of improved exploration gives the upper bound on $R_{T}^{\text{mart}}$. A similar analysis is also conducted in \citet{lale2022reinforcement}, yet we provide it for completeness. 

\begin{lemma}[Bounding $R_{T}^{\text{mart}}$ ] \label{R1_zeta_stabil}
Let $R_{T}^{\text{mart}}$ be as defined by \eqref{def:R_mart}. Under the event of $E_T$, with probability at least $1-\delta$, for $t> T_{r}$, we have 
\begin{align*}
    R_{T}^{\text{mart}} &\leq k_{s,1} (n+d)^{n+d} (\sigma_w + \| B_*\|\sigma_\nu)n \sqrt{T_{r}} \log((n+d) T_{r}/\delta) \\
    &\quad + \frac{k_{s,2}(12\kappa^2+ 2\kappa\sqrt{2})}{\gamma}  \sigma_w^2 n\sqrt{n} \sqrt{T-\Tw} \log(n(T-\Tw)/\delta) \\
    &\quad +  k_{s,3} n\sigma_w^2 \sqrt{T-\Tw}\log(nT/\delta) + k_{s,4} n(\sigma_w+\|B_*\|\sigma_\nu)^2\sqrt{\Tw}\log(nT/\delta),
\end{align*}
for some problem dependent coefficients $k_{s,1}, k_{s,2}, k_{s,3}, k_{s,4}$. 
\end{lemma}

\begin{proof}
 Let $f_{t} = A_*x_{t} + B_*u_{t}$. One can decompose $R_{T}^{\text{mart}}$ as 
 \[
 R_1 = x_0^\tp P(\ttt_0) x_0 - x_{T+1}^\tp P(\ttt_{T+1}) x_{T+1} + \sum_{t=1}^{T} x_t^\tp P(\ttt_{t}) x_t - \mathbb{E}\left[x_{t}^\tp P(\ttt_{t}) x_{t}\big| \mathcal{F}_{t-2}\right] 
 \]
 Since $P(\ttt_0)$ is positive semidefinite and $x_0 = 0$, the first two terms are bounded above by zero. Recall that $\zeta_t = B_* \nu_t +w_t$ for $t \leq T_w$, and $\zeta_t = w_t$ for $t > T_w$. The second term is decomposed as follows 
 \begin{align*}
      \sum_{t=1}^{T} x_t^\tp P(\ttt_{t}) x_t &- \mathbb{E}\left[x_{t}^\tp P(\ttt_{t}) x_{t}\big| \mathcal{F}_{t-2}\right]  \\
      &= \sum_{t=1}^{T} f_{t-1}^\tp P(\ttt_{t}) \zeta_{t-1} + \sum_{t=1}^{T} \left( \zeta_{t-1}^\tp P(\ttt_{t}) \zeta_{t-1} - \mathbb{E}\left[\zeta_{t-1}^\tp P(\ttt_{t}) \zeta_{t-1}\big| \mathcal{F}_{t-2}\right] \right)
 \end{align*}
 
 Let $R_{1,1} =  \sum_{t=1}^{T} f_{t-1}^\tp P(\ttt_{t}) \zeta_{t-1} $, $R_{1,2} = \sum_{t=1}^{T} \left( \zeta_{t-1}^\tp P(\ttt_{t}) \zeta_{t-1} - \mathbb{E}\left[\zeta_{t-1}^\tp P(\ttt_t) \zeta_{t-1}\big| \mathcal{F}_{t-2}\right] \right)$, and $v_{t-1}^\tp =  f_{t-1}^\tp P(\ttt_{t}) $. Then one can write $R_{1,1}$. Let  can be written as 
 \[
 R_{1,1} = \sum_{t=1}^{T} \sum_{i=1}^{n} v_{t-1, i} \zeta_{t-1, i} = \sum_{i=1}^{n}  \sum_{t=1}^{T} v_{t-1, i} \zeta_{t-1, i}.
 \]
 Let $M_{t, i}=\sum_{k=1}^{t} v_{k-1, i} \zeta_{k-1, i} .$ By Theorem \ref{selfnormalized} on some event $G_{\delta, i}$ that holds with probability at least $1-\delta /(2 n),$ for any $t \geq 0$,  
\begin{align*}
M_{t, i}^2 &\leq 2 (\sigma_w^2 + \|B_*\|^2 \sigma_\nu^2 ) \left(1 + \sum_{k=1}^{T_{r}} v_{k-1,i}^2\right) \log \left(\frac{2n}{\delta} \left(1 + \sum_{k=1}^{T_{r}} v_{k-1,i}^2\right)^{1/2} \right) \\
&\qquad + 
2 \sigma_w^2 \left(1 + \sum_{k=T_{r}+1}^{t} v_{k-1,i}^2\right) \log \left(\frac{2n}{\delta} \left(1 + \sum_{k=T_{r}+1}^{t} v_{k-1,i}^2\right)^{1/2} \right) \quad \text{for } t > T_{r}.
\end{align*}
Notice that \alg stops additional isotropic perturbation after $t=\Tw$, and the state starts decaying until $t=T_{r}$. For simplicity of presentation we treat the time between $\Tw$ and $T_{r}$ as TS with improved exploration while sacrificing the tightness of the result. On $E_T$, $\| v_k \| \leq DS (n+d)^{n+d} \sqrt{1+\kappa^2} $ for $k \leq T_{r}$ and $\| v_k \| \leq
\frac{(12\kappa^2+ 2\kappa\sqrt{2})DS\sigma_w\sqrt{1+\kappa^2}}{\gamma}  \sqrt{2n\log(n(t-\Tw)/\delta)}$ for $k > T_{r}$. Thus, $v_{k,i} \leq DS (n+d)^{n+d} \sqrt{1+\kappa^2}$ and $v_{k,i} \leq \frac{(12\kappa^2+ 2\kappa\sqrt{2})DS\sigma_w\sqrt{1+\kappa^2}}{\gamma}  \sqrt{2n\log(n(t-\Tw)/\delta)}$ respectively for $k\leq T_{r}$ and $k>T_{r}$ . Using union bound we get, for probability at least $1-\frac{\delta}{2}$,  for $t > T_{r}$,
\begin{align*}
    R_{1,1} &\leq n \sqrt{ 2 (\sigma_w^2 + \|B_*\|^2 \sigma_\nu^2 )  \left(1 + T_{r} D^2 S^2 (n+d)^{2(n+d)}  (1+\kappa^2) \right) } \times \\
    &\qquad \qquad \qquad \sqrt{\log \left( \frac{4n}{\delta} \left(1 + T_{r} D^2 S^2 (n+d)^{2(n+d)}  (1+\kappa^2) \right)^{1/2}  \right) }   \\
&+n \sqrt{ 2 \sigma_w^2  \left(1 + \frac{2(t-T_{r})(12\kappa^2+ 2\kappa\sqrt{2})^2D^2S^2n\sigma_w^2(1+\kappa^2)}{\gamma^2} \log(n(T-\Tw)/\delta) \right) } \times \\
&\qquad \sqrt{ \log \left( \frac{4n}{\delta} \left(1 + \frac{2(t-T_{r})(12\kappa^2+ 2\kappa\sqrt{2})^2D^2S^2n\sigma_w^2(1+\kappa^2)}{\gamma^2} \log(n(T-\Tw)/\delta) \right)  \right) }.
\end{align*}
Let $\mathcal{W}_{exp} \!=\! (\sigma_w+\|B_*\|\sigma_\nu) \sqrt{2n\log \frac{4nT}{\delta}}$ and $\mathcal{W}_{noexp} \!=\!  \sigma_w \sqrt{2n\log \frac{4nT}{\delta}}$. Define $\Psi_{t}\!=\!\zeta_{t-1}^{\tp} P(\ttt_{t}) \zeta_{t-1}-\mathbb{E}\left[\zeta_{t-1}^{\tp} P(\ttt_{t}) \zeta_{t-1} | \mathcal{F}_{t-2}\right]$ and its truncated version $\tilde{\Psi}_{t}=\Psi_{t} \mathbb{I}_{\left\{\Psi_{t} \leq 2 D W_{exp}^{2}\right\}}$ for $t \leq \Tw$ and $\tilde{\Psi}_{t}=\Psi_{t} \mathbb{I}_{\left\{\Psi_{t} \leq 2 D W_{noexp}^{2}\right\}}$ for $t > \Tw$ . Notice that $R_{1,2} =  \sum_{t=1}^{T} \Psi_{t} $. 
\begin{align*}
	&\Pr \left( \sum_{t=1}^{\Tw} \Psi_t > 2DW_{exp}^2 \sqrt{2\Tw \log \frac{4}{\delta}}\right) + \Pr \left( \sum_{t=\Tw+1}^{T} \Psi_t > 2DW_{noexp}^2 \sqrt{2(T-\Tw) \log \frac{4}{\delta}}\right) \\
	&\leq \Pr \left( \max_{1 \leq t \leq \Tw} \Psi_t > 2DW_{exp}^2 \right) + \Pr \left( \max_{\Tw + 1 \leq t \leq T} \Psi_t > 2DW_{noexp}^2 \right) \\
	&+\Pr \left( \sum_{t=1}^{\Tw} \tilde{\Psi}_t > 2DW_{exp}^2 \sqrt{2\Tw \log \frac{4}{\delta}}\right) +  \Pr \left( \sum_{t=\Tw+1}^{T} \tilde{\Psi}_t > 2DW_{noexp}^2 \sqrt{2(T-\Tw) \log \frac{4}{\delta}}\right)
\end{align*}
By Lemma \ref{subgauss lemma} with union bound and Theorem \ref{Azuma}, summation of terms on the right hand side is bounded by $\delta/2$. Thus, with probability at least $1-\delta/2$, for $t > \Tw$,
\[
 R_{1,2} \leq  4nD\sigma_w^2   \sqrt{2(t-\Tw) \log \frac{4}{\delta}}\log \frac{4nT}{\delta} + 4nD(\sigma_w+\|B_*\|\sigma_\nu)^2\sqrt{2\Tw \log \frac{4}{\delta}}\log \frac{4nT}{\delta}.
 \]
 Combining $R_{1,1}$ and $R_{1,2}$ gives the statement. 
\end{proof}

\subsection{Bounding \texorpdfstring{$R_{T}^{\text{TS}}$}{RTS}} \label{subsec:reg_ts}
\begin{lemma}[Bounding $R_{T}^{\text{TS}}$ for \alg]
Let $R_{T}^{\text{TS}}$ be as defined by \eqref{def:R_TS}. Under the event of $E_T$, we have that 
\begin{align*}
	 \left|R_{T}^{\text{TS}}\right| \leq \tilde{O}\pr{\sqrt{n}\Tw + \poly(n,d,\log(1/\delta))\sqrt{T-\Tw}}.
	\end{align*}
\end{lemma}
with probability at least $1-2\delta$ if $\Tw = \omega(\sqrt{T} \log{T})$ for singular $A_{c,*}$ and $\Tw = \omega(\log{T})$ for non-singular $A_{c,*}$.
\begin{proof}
We decompose $R_{T}^{\text{TS}}$ into two pieces as
\begin{align}
     R_{T}^{\text{TS}} &= \underbrace{\sum_{t=0}^{T_w}\cl{ J_*(\Tilde{\Theta}_t, \sigma_w^2 I) -J_*(\Theta_*, \sigma_w^2 I)}\indic_{E_t}}_{R_{T_w}^{\text{TS,exp}}} + \underbrace{{\sum_{t=T_w + 1}^{T}\cl{ J_*(\Tilde{\Theta}_t, \sigma_w^2 I) -J_*(\Theta_*, \sigma_w^2 I)}\indic_{E_t}}}_{R_{T}^{\text{TS,noexp}}} \nn
\end{align}
Since every sampled system is in set $\clf{S}$, we have that $\norm[F]{P(\Tilde{\Theta}_t)} \leq D$ and therefore
\begin{align}
    R_{T_w}^{\text{TS,exp}} &\leq \sum_{t=0}^{T_w}\abs{ J_*(\Tilde{\Theta}_t, \sigma_w^2 I) -J_*(\Theta_*, \sigma_w^2 I)}\indic_{E_t} \nn \\
    &\leq \sum_{t=0}^{T_w}\pr{\abs{ J_*(\Tilde{\Theta}_t, \sigma_w^2 I)} +\abs{J_*(\Theta_*, \sigma_w^2 I)}} \label{TSstep1} \\
    &\leq \sqrt{n}\sigma_w^2 \sum_{t=0}^{T_w}\pr{ \norm[F]{P(\Tilde{\Theta}_t)} + \norm[F]{P(\Theta_*)} } \leq 2\sqrt{n} \sigma_w^2 D T_w \label{TSstep0}
\end{align}
where we used the relation $\tr(P) \leq \sqrt{n} \norm[F]{P}$ in \eqref{TSstep1}. Considering the number of times a new TS sample is drawn, the second term in $R_{T}^{\text{TS}}$ can be written as
\begin{align}
    R_{K}^{\text{TS,noexp}} &= {\sum_{k=0}^{K} \tau_0 \cl{ J_*(\Tilde{\Theta}_{t_k}, \sigma_w^2 I) -J_*(\Theta_*, \sigma_w^2 I)}\indic_{E_{t_k}}} \nn
\end{align}
where $t_k = T_w +1 + k\tau_0$ and $K= \left\lceil\frac{T-\Tw}{\tau_0}\right\rceil$. Denoting the information available to the controller up to time $t\geq 0$ via $\clf{F}_{t}^{\text{cnt}} \defeq \sigma\pr{\clf{F}_{t-1}, x_t}$, $R_{K}^{\text{TS,noexp}}$ can be further decomposed into two pieces as
\begin{align}
    R_{K}^{\text{TS,noexp}}  &= \underbrace{\sum_{k=0}^{K} \tau_0  \cl{ J_*(\Tilde{\Theta}_{t_k}, \sigma_w^2 I) -\E\br{J_*(\Tilde{\Theta}_{t_k}, \sigma_w^2 I) \, \big| \, \clf{F}_{t_k}^{\text{cnt}}, E_{t_k}}}\indic_{E_{t}}}_{R_{K}^{\text{TS,1}}} , \nn \\
    &\quad \quad +\underbrace{\sum_{k=0}^{K} \tau_0 \cl{ \E\br{J_*(\Tilde{\Theta}_{t_k}, \sigma_w^2 I) \, \big| \, \clf{F}_{t}^{\text{cnt}}, E_{t_k}} -J_*(\Theta_*, \sigma_w^2 I)}\indic_{E_{t_k}}}_{R_{K}^{\text{TS,2}}} \nn
\end{align}
We will investigate each term in order under the event of $E_T$. \\
\textbf{Bounding $\bm{R_{K}^{\text{TS,1}}}$. } Notice that $\{R_{K}^{\text{TS,1}}\}_{K\geq 0} $ is a martingale sequence with $\abs{R_{K}^{\text{TS,1}} - R_{K-1}^{\text{TS,1}}}\leq 2\tau_0\sigma_w^2\sqrt{n}D$. Therefore it can be bounded by Azuma's inequality (Lemma~\ref{Azuma}) w.p. at least $1-\delta$ as
\begin{align}
    R_{K}^{\text{TS,1}} \leq \sigma_w^2 D\sqrt{8n \tau_0^2 K\log(2/\delta)}\leq  \sigma_w^2 D\sqrt{8n \tau_0 (T-\Tw)\log(2/\delta)} \label{TSstep2}
\end{align}
\textbf{Bounding $\bm{R_{K}^{\text{TS,2}}}$. } Denoting by $\clf{S}^{\text{opt}} \defeq \cl{\Theta \in \R^{(n+d)\times n} \, \big| \, J_*(\Theta, \sigma_w^2 I) \leq J_*(\Theta_*, \sigma_w^2 I)}$ the set of optimistic parameters and defining $R_k^{TS,2} \defeq \cl{ \E\br{J_*(\Tilde{\Theta}_{t_k}, \sigma_w^2 I) \, \big| \, \clf{F}_{t_k}^{\text{cnt}}, E_{t_k}} -J_*(\Theta_*, \sigma_w^2 I)}\indic_{E_{t_k}}$.  Notice that, for any $\Theta \in \clf{S}^{\text{opt}}$, we can write
\begin{align}
    R_k^{TS,2} &\leq \cl{ \E\br{J_*(\Tilde{\Theta}_{t_k}, \sigma_w^2 I) \, \big| \, \clf{F}_{t_k}^{\text{cnt}}, E_{t_k}} -J_*(\Theta, \sigma_w^2 I)}\indic_{E_{t_k}} \nn \\
    &\leq \abs{J_*(\Theta, \sigma_w^2 I) - \E\br{J_*(\Tilde{\Theta}_{t_k}, \sigma_w^2 I) \, \big| \, \clf{F}_{t_k}^{\text{cnt}}, E_{t_k}} }\indic_{E_{t_k}} \nn
\end{align}
As the above bound holds for any $\Theta \in \clf{S}^{\text{opt}}$, we can replace the right hand side with an expectation over the optimistic set $\clf{S}^{\text{opt}}$. Specifically, we choose an i.i.d. copy of $\Tilde{\Theta}_{t_k}$, that is, we choose a random variable $\Tilde{\Theta}_{t_k}^\prime$ which has the same distribution as $\Tilde{\Theta}_{t_k}$ and independent from it.  Then, we have that
\begin{align}
    R_k^{TS,2} &\leq \E \br{\abs{J_*(\Tilde{\Theta}_{t_k}^\prime , \sigma_w^2 I) - \E\br{J_*(\Tilde{\Theta}_{t_k}, \sigma_w^2 I) \, \big| \, \clf{F}_{t_k}^{\text{cnt}}, E_{t_k}} }\indic_{E_{t_k}} \, \big| \, \clf{F}_{t_k}^{\text{cnt}}, E_{t_k}, \Tilde{\Theta}_{t_k}^\prime \in \clf{S}^{\text{opt}}} \nn \\
    &= \frac{\E \br{\abs{J_*(\Tilde{\Theta}_{t_k}^\prime, \sigma_w^2 I) - \E\br{J_*(\Tilde{\Theta}_{t_k}, \sigma_w^2 I) \, \big| \, \clf{F}_{t_k}^{\text{cnt}}, E_{t_k}} }\indic_{E_{t_k}} \indic_{\Tilde{\Theta}_{t_k}^\prime \in \clf{S}^{\text{opt}}} \, \big| \, \clf{F}_{t_k}^{\text{cnt}}, E_{t_k}}}{ \Pr\pr{\Tilde{\Theta}_{t_k}^\prime \in \clf{S}^{\text{opt}} \, \big| \,  \clf{F}_{t_k}^{\text{cnt}}, \hat E_{t_k}} } \nn 
\end{align}
Denoting by $p_t^{\opt} = \Pr\pr{\Tilde{\Theta}_{t}^\prime \in \clf{S}^{\text{opt}} \, \big| \,  \clf{F}_{t}^{\text{cnt}}, \hat E_{t}}$ the probability of drawing cost optimistic TS samples, we can write further bounds on $R_k^{TS,2}$ as
\begin{align}
    R_k^{TS,2} &\leq  \frac{1}{ p_{t_k}^{\opt} } \E \br{\abs{J_*(\Tilde{\Theta}_{t_k}^\prime, \sigma_w^2 I) - \E\br{J_*(\Tilde{\Theta}_{t_k}, \sigma_w^2 I) \, \big| \, \clf{F}_{t_k}^{\text{cnt}}, E_{t_k}} } \, \big| \, \clf{F}_{t_k}^{\text{cnt}}, E_{t_k}} \nn \\
    &= \frac{\sigma_w^2}{p_{t_k}^{\opt}} \E \br{\abs{\Tr \pr{P(\Tilde{\Theta}_{t_k}^\prime) - \E\br{P(\Tilde{\Theta}_{t_k}) \, \big| \, \clf{F}_{t_k}^{\text{cnt}}, E_{t_k}} }} \, \big| \, \clf{F}_{t_k}^{\text{cnt}}, E_{t_k}} \nn \\
    &\leq \frac{n\sigma_w^2}{p_{t_k}^{\opt}} \E \br{\Norm[2]{P(\Tilde{\Theta}_{t_k}^\prime) - \E\br{P(\Tilde{\Theta}_{t_k}) \, \big| \, \clf{F}_{t_k}^{\text{cnt}}, E_{t_k}} } \, \big| \, \clf{F}_{t_k}^{\text{cnt}}, E_{t_k}} \label{TSstep3}
\end{align}
where we used the relation $\abs{\tr(A)} \leq n \norm[2]{A}$. Denoting $P_k \defeq \E\br{P(\Tilde{\Theta}_{t_k}) \, \big| \, \clf{F}_{t_k}^{\text{cnt}}, E_{t_k}}$, the following definition will be used in the rest of the section to understand the behavior of $R_k^{TS,2}$
\begin{align}
    \Delta_k \defeq \E \br{\norm[2]{P(\Tilde{\Theta}_{t_k}) - P_k } \, \big| \, \clf{F}_{t_k}^{\text{cnt}}, E_{t_k}} \label{def:delta_k}
\end{align}
The following lemma will be used to bound $\Delta_k$ from above.
\begin{lemma}\label{thm:gradPandH}
    For any $\Theta \in \clf{S}$, any positive definite matrix $V \in \R^{(n+d) \times (n+d)}$, and for any $i,j \in [n]$, 
    \begin{align}
        \norm[V]{\grad P_{ij} (\Theta)} \leq \Gamma \norm[V]{H(\Theta)} \nn,
    \end{align}
\end{lemma}
where $\Gamma \geq 0$ is a problem dependent constant.
\begin{proof}
Let $\delta P (\tt, \delta \tt)$ be the differential of $P(\tt)$ in the direction $\delta \tt$. Then, we have that
\begin{align} \label{eq:deltaP}
    \delta P (\tt, \delta \tt) &= A_{c}(\tt) ^\tp \delta P (\tt, \delta \tt) A_{c}(\tt)   \\
    &\quad\quad + A_{c}(\tt) ^\tp  P (\tt) \delta \tt^\tp H(\tt) + H(\tt) ^\tp \delta\tt  P (\tt) A_{c}(\tt) \nonumber
\end{align}
where $A_{c}(\tt) = \tt^\tp H(\tt)$ is the closed-loop matrix. We know that $P(\tt)$ satisfies the Riccati equation as
\begin{align}
    P-  A_{c} ^\tp P  A_{c}= Q+K^\tp R K \psdg 0 \implies \pr{P^{\frac{1}{2}} A_{c} P^{-\frac{1}{2}}}^\tp P^{\frac{1}{2}} A_{c} P^{-\frac{1}{2}} \psdl I \nn
\end{align}
where we dropped $\tt$ dependence for simplicity. Therefore, similarity transformation of the closed-loop matrix  $\bar{A}_{c} \defeq P^{\frac{1}{2}} A_{c} P^{-\frac{1}{2}}$ is a contraction, \ie, $\norm[2]{P^{\frac{1}{2}} A_{c} P^{-\frac{1}{2}}} \eqdef \sigma_{\tt} < 1$. Multiplying both sides of \eqref{eq:deltaP} by $P^{-\frac{1}{2}}$ we obtain 
\begin{align} 
    \delta \bar P (\delta \tt) &= \bar A_{c} ^\tp \delta \bar P (\delta \tt) \bar A_{c} + \bar A_{c} ^\tp  P^{\frac{1}{2}} \delta \tt^\tp H  P^{-\frac{1}{2}} + P^{-\frac{1}{2}} H ^\tp \delta\tt  P^{\frac{1}{2}}  \bar A_{c} \nn
\end{align}
where $\delta \bar P(\delta \tt) = P^{-\frac{1}{2}} \delta P(\delta \tt) P^{-\frac{1}{2}}$. Taking the spectral norm of both sides and using sub-multiplicativity of spectral norm as well as equivalence of matrix norms, we have that 
\begin{align}
    \norm[2]{\delta \bar P (\delta \tt)} &\leq   \norm[2]{A_{c}}^2 \norm[2]{\delta \bar P (\delta \tt)}  + 2\norm[2]{\bar A_{c}} \norm[2]{ P^{\frac{1}{2}} \delta \tt^\tp H  P^{-\frac{1}{2}}} \nn \\
    &\leq   \norm[2]{A_{c}}^2 \norm[2]{\delta \bar P (\delta \tt)}  + 2\norm[2]{\bar A_{c}} \norm[F]{ P^{\frac{1}{2}} \delta \tt^\tp H  P^{-\frac{1}{2}}} \nn \\
    &=  \sigma_{\tt}^2 \norm[2]{\delta \bar P (\delta \tt)}  + 2\sigma_{\tt} \norm[F]{\delta \tt^\tp H } \nn
\end{align}
By rearranging the inequality and using the property $ \norm[F]{\delta \tt^\tp H } \leq  \norm[V^{-1}]{\delta \tt } \norm[V]{H }$, we obtain
\begin{align}
    \norm[2]{\delta \bar P (\delta \tt)} \leq \frac{2\sigma_\tt}{1-\sigma_\tt^2} \norm[V^{-1}]{\delta \tt } \norm[V]{H} \nn
\end{align}
Observing that $\norm[2]{\delta  P (\delta \tt)} = \norm[2]{ P^{\frac{1}{2}}\delta \bar P (\delta \tt) P^{\frac{1}{2}}} \leq\norm[2]{P}\norm[2]{ \delta \bar P (\delta \tt) }\leq D\norm[2]{ \delta \bar P (\delta \tt) }$ and noting that $\norm[V]{\grad P_{ij} (\Theta)} = \sup_{\norm[V^{-1}]{\delta \tt } = 1} \abs{\delta  P_{ij} (\delta \tt)} \leq \sup_{\norm[V^{-1}]{\delta \tt } = 1} \norm[2]{\delta  P (\delta \tt)}$, one can get
\begin{align}
    \norm[V]{\grad P_{ij} (\Theta)} \leq \frac{2D\sigma_\tt}{1-\sigma_\tt^2} \norm[V]{H(\tt)} \nn
\end{align}
Observing that the function $\sigma_{\tt} : \S \to \R_+$ is continuous on $\S$ and $\sigma_* \defeq \max_{\tt \in \S} \sigma_\tt < 1$ as $\S$ is compact, we can further bound the scalar from above by $\tt$ independent constant $\Gamma = \frac{2D\sigma_*}{1-\sigma_*^2} >0$. 
\end{proof}

The following lemma gives a useful upper bound on $\Delta_k$.
\begin{lemma}\label{thm:DeltaBound}
Let $\Delta_k$ be defined as in \eqref{def:delta_k}. Then, for all $k\geq 0$, we have that
\begin{align}
     \Delta_k \leq 2n^2 \upsilon_{t_k} \Gamma \E \br{\norm[V_{t_k}^{-1}]{H(\ttt_{t_k})} \, \big| \, \clf{F}_{t_k}^{\text{cnt}}, E_{t_k}}. \nn
\end{align}
\end{lemma}
\begin{proof}
The proof follows directly from applying the bound in Lemma~\ref{thm:gradPandH} to Equation 11 in \cite{abeille2018improved}.
\end{proof}
Finally, we are ready to give a bound on the summation of $\Delta_k$ terms 
\begin{lemma}\label{thm:SumDelta}
Let $\Delta_k$ be defined as in \eqref{def:delta_k} for any $k\geq 0$. Then, the following bound holds with probability at least $1-\delta$
\begin{align}
    \sum_{k=0}^{K}  \Delta_k &\leq \frac{16 n^2 \alpha \upsilon_{T} \Gamma }{1+\frac{1}{\beta_{T}}}  \pr{\sum_{t=\Tw+1}^{T} \norm[V_{t}^{-1}]{z_{t}} + 2 \alpha \sqrt{2\frac{T-\Tw}{\tau_0}\frac{1+\kappa^2 }{\mu} \log\pr{\frac{2}{\delta}}}}  \nn\\
    &\leq \tilde{O}(\poly(n,d,\log(1/\delta)) \sqrt{T-\Tw}) \nn
\end{align}
where {$\alpha = (1+1/\beta_0^2)(\sqrt{2n\log(3n)}+\upsilon_T +(1+\kappa)S X_s)$}.
\end{lemma}
\begin{proof}
Define $\bar{\Theta}_{t_k} = \tth_{t_k} + \beta_{t_k} V_{t_k}^{-\frac{1}{2}} \eta_{t_k}$. Using Proposition 9 in \cite{abeille2018improved}, we have that
\begin{align}
    \norm[V_{t_k}^{-1}]{H(\bar{\Theta}_{t_k})} &\leq \frac{8}{1+\frac{1}{\beta_{t_k}}} \Norm[V_{t_k}^{-1}]{H(\bar{\Theta}_{t_k}) \E\br{x_{t_k} x_{t_k}^\tp  \indic_{\norm{x_{t_k}}\leq \alpha}\,\big|\, \clf{F}_{t_k-1}, E_{t_k-1}, \bar{\Theta}_{t_k} }}\nn \\
    &\leq \frac{8}{1+\frac{1}{\beta_{t_k}}} \Norm[V_{t_k}^{-1}]{\E\br{H(\bar{\Theta}_{t_k}) x_{t_k} x_{t_k}^\tp  \indic_{\norm{x_{t_k}}\leq \alpha}\,\big|\, \clf{F}_{t_k-1}, E_{t_k-1}, \bar{\Theta}_{t_k} }} \nn\\
    &\leq \frac{8\alpha}{1+\frac{1}{\beta_{t_k}}} \E\br{\Norm[V_{t_k}^{-1}]{H(\bar{\Theta}_{t_k})x_{t_k} } \indic_{\norm{x_{t_k}}\leq \alpha}\,\big|\, \clf{F}_{t_k-1}, E_{t_k-1}, \bar{\Theta}_{t_k} }\nn
\end{align}
By Lemma~\ref{thm:DeltaBound} and the preceding bound, we can write 
\begin{align}
     \Delta_k &\leq 2n^2 \upsilon_{t_k} \Gamma \E \br{\norm[V_{t_k}^{-1}]{H(\ttt_{t_k})} \, \big| \, \clf{F}_{t_k}^{\text{cnt}}, E_{t_k}} = 2n^2 \upsilon_{t_k} \Gamma \frac{\E \br{\norm[V_{t_k}^{-1}]{H(\bar{\Theta}_{t_k})} \indic_{\bar{\Theta}_{t_k} \in \clf{S}} \, \big| \, \clf{F}_{t_k}^{\text{cnt}}, E_{t_k}}}{\Pr \cl{\bar{\Theta}_{t_k} \in \clf{S} \,\big|\, \clf{F}_{t_k}^{\text{cnt}}, E_{t_k}}} \nonumber \\
     &\leq \frac{16 n^2 \alpha \upsilon_{t_k} \Gamma }{1+\frac{1}{\beta_{t_k}}} \frac{\E \br{\E\br{\norm[V_{t_k}^{-1}]{H(\bar{\Theta}_{t_k})x_{t_k}} \indic_{\norm{x_{t_k}}\leq \alpha}\,\big|\, \clf{F}_{t_k-1}, E_{t_k-1}, \bar{\Theta}_{t_k} } \indic_{\bar{\Theta}_{t_k} \in \clf{S}} \, \big| \, \clf{F}_{t_k}^{\text{cnt}}, E_{t_k}}}{\Pr \cl{\bar{\Theta}_{t_k} \in \clf{S} \,\big|\, \clf{F}_{t_k}^{\text{cnt}}, E_{t_k}}} \nonumber \\
     &= \frac{16 n^2 \alpha \upsilon_{t_k} \Gamma }{1+\frac{1}{\beta_{t_k}}} \underbrace{\E [\E[\norm[V_{t_k}^{-1}]{\underbrace{H(\ttt_{t_k})x_{t_k}}_{z_{t_k}} } \indic_{\norm{x_{t_k}}\leq \alpha}\,\big|\, \clf{F}_{t_k-1}, E_{t_k-1}, \ttt_{t_k} ]  \, \big| \, \clf{F}_{t_k}^{\text{cnt}}, E_{t_k}]}_{\eqdef Y_k} \nn
\end{align}
Notice that $\E\br{Y_k \big| \clf{F}_{t_{k-1}}} = \E\br{\norm[V_{t_k}^{-1}]{z_{t_k}} \indic_{\norm{x_{t_k}}\leq \alpha}\,\big|\, \clf{F}_{t_{k-1}}}$ by law of iterated expectations and {$ \norm[V_{t_k}^{-1}]{z_{t_k}} \indic_{\norm{x_{t_k}}\leq \alpha} \leq \frac{1}{\sqrt{\mu}}\norm{H(\ttt_{t_k})x_{t_k}}\indic_{\norm{x_{t_k}}\leq \alpha} \leq \sqrt{\frac{1+\kappa^2}{\mu}}\alpha$}.
\\ \\
Therefore, the sequence $\cl{Y_k - \norm[V_{t_k}^{-1}]{z_{t_k}} \indic_{\norm{x_{t_k}}\leq \alpha} }_{k\geq 0}$ is a bounded martingale difference sequence. By Azuma's inequality, we have that with probability at least $1-\delta$,
\begin{align}
    \sum_{k=0}^{K} \pr{Y_k - \norm[V_{t_k}^{-1}]{z_{t_k}} \indic_{\norm{x_{t_k}}\leq \alpha} } \leq 2 \alpha \sqrt{2\frac{T-T_w}{\tau_0}\frac{1+\kappa^2 }{\mu} \log\pr{\frac{2}{\delta}}} \nn
\end{align}
We can bound the sum of $\norm[V_{t_k}^{-1}]{z_{t_k}}$ terms using Lemma 10 of \cite{abbasi2011improved} and Hölder's inequality as
\begin{align}
     \sum_{k=0}^{K} \norm[V_{t_k}^{-1}]{z_{t_k}} &\leq \sum_{k=0}^{K} \norm[V_{t_k}^{-1}]{z_{t_k}} +  \sum_{k=0}^{K} \sum_{t = t_k + 1}^{t_{k+1}-1} \norm[V_{t}^{-1}]{z_{t}} \nn \\
     &= \sum_{t=\Tw+1}^{T} \norm[V_{t}^{-1}]{z_{t}} \leq \sqrt{T-T_w}\log{\frac{\det(V_T)}{\det(V_{\Tw})}}  \nn
\end{align}
Combining these results, we obtain the desired bound 
\begin{align}
    \sum_{k=0}^{K} \Delta_k &\leq \frac{16 n^2 \alpha \upsilon_{T} \Gamma }{1+\frac{1}{\beta_{T}}}  \pr{\sum_{t=\Tw+1}^{T} \norm[V_{t}^{-1}]{z_{t}} + 2 \alpha \sqrt{2\frac{T-\Tw}{\tau_0}\frac{1+\kappa^2 }{\mu} \log\pr{\frac{2}{\delta}}}}  \nn
\end{align}
\end{proof}
Now, we are ready to bound $R_K^{TS,2}$. Under the event $E_T$ Theorem~\ref{thm:optimistic_prob} suggests that $1/p_{t^{\opt}} \leq O(1)$ if $\Tw = \omega(\sqrt{T} \log{T})$ for singular $A_{c,*}$ and $\Tw = \omega(\log{T})$ for non-singular $A_{c,*}$. Using this result together with Lemma~\ref{thm:SumDelta}, we have that 
\begin{equation}
    R_K^{TS,2} = \sum_{k=0}^{K} \tau_0 R_k^{TS,2} \leq {n\sigma_w^2 \tau_0} \sum_{k=0}^{K} \frac{\Delta_k}{p_{t_
    k^{\opt}}} \leq \tilde{O}(\poly(n,d)\sqrt{(T-\Tw)\log(1/\delta)})  \label{TSstep}
\end{equation}
with probability at least $1-\delta$.
Combining the above with \eqref{TSstep0} and \eqref{TSstep2}, we obtain the desired bound.
\end{proof} 


\subsection{Bounding \texorpdfstring{$R_{T}^{\text{gap}}$}{Rgap}} \label{subsec:reg_gap}
\begin{lemma}[Bounding $R_{T}^{\text{gap}}$ for \alg]
Let $R_{T}^{\text{gap}}$ be as defined by \eqref{def:R_gap}. Under the event of $E_T$, we have that 
	\begin{align*}
	 \left|R_{T}^{\text{gap}}\right| = \tilde{O}\pr{\poly(n,d)\sqrt{T\log(1/\delta)}}.
	\end{align*}
\end{lemma}
with probability at least $1-2\delta$ for large enough $T$.
\begin{proof}
 \begin{align}
    R_{T}^{\text{gap}} &=  \sum_{t=0}^{T} \E\br{x_{t+1}^\tp \pr{ P(\Tilde{\Theta}_{t+1}) - P(\Tilde{\Theta}_{t})}x_{t+1} \indic_{E_{t+1}} \,\big|\,\clf{F}_t } \\
     &=  \sum_{t=0}^{K} \E\br{x_{t_k+1}^\tp \pr{ P(\Tilde{\Theta}_{t_k+1}) - P(\Tilde{\Theta}_{t_k})}x_{t_k+1} \indic_{E_{t_k+1}} \,\big|\,\clf{F}_{t_k} } 
\end{align}
Separating the duration of \alg into two parts at $t = T_r$, we obtain two same term achieved in \citep{abeille2018improved}. Note that in \citet{abeille2018improved}, the authors follow frequent update rule and \alg updates every $\tau_0$ time-steps. The proof of these terms similarly follow Section 5.2 in \cite{abeille2018improved} and using Lemma \ref{thm:SumDelta} we obtain $\OO((n\!+\!d)^{n+d} \sqrt{T_r} \!+\! \text{poly}(n,d) \sqrt{T\!-\!T_r})$. Note that there is an additional $\tau_0$ factor in these bounds, due to ``relatively slower'' update of \alg. For large enough $T$ such that the second term dominates the overall upper bound, we obtain the advertised guarantee. 
\end{proof}

\subsection{Proof of Theorem \ref{reg:exp_s}} \label{subsec:reg_all}
Collecting the regret terms derived in subsections of Appendix \ref{apx:regret_anlz}, for large enough $T$, under the event $E_T$, we have that
\begin{align}
    R_{T_w}^{\text{exp}} &= \tilde{O}\pr{ (n+d)^{n+d}\Tw}, \quad \text{w.p. } 1-\delta \nn \\
    R_{T}^{\text{RLS}} &= \tilde{O}\pr{ (n+d)^{n+d}\sqrt{T_r} + \poly(n,d,\log(\nicefrac{1}{\delta}))\sqrt{T-T_r}}, \nn \\
    R_{T}^{\text{mart}} &= \tilde{O}\pr{ (n+d)^{n+d}\sqrt{T_r} + \poly(n,d,\log(\nicefrac{1}{\delta}))\sqrt{T-\Tw}}, \quad \text{w.p. } 1-\delta \nn \\
    R_{T}^{\text{gap}} &= \tilde{O}\pr{ (n+d)^{n+d}\sqrt{T_r} + \poly(n,d,\log(\nicefrac{1}{\delta}))\sqrt{T-T_r}}, \quad \text{w.p. } 1-2\delta  \nn
\end{align}
and choosing $\Tw = \omega(\sqrt{T} \log{T})$ for singular $A_{c,*}$ and $\Tw = \omega(\log{T})$ for non-singular $A_{c,*}$ gives
\begin{align}
    R_{T}^{\text{TS}} &= \tilde{O}\pr{ \poly(n,d)\Tw + \poly(n,d,\log(\nicefrac{1}{\delta}))\sqrt{T-T_r}}, \quad \text{w.p. } 1-2\delta.  \nn
\end{align}
Recall that the event $E_T$ is true with probability at least $1-4\delta$. Combining all these bounds, we have the overall regret bound as
\begin{align}
    R_T &= \tilde{O}\pr{ (n+d)^{n+d}\Tw + \poly(n,d,\log(\nicefrac{1}{\delta}))\sqrt{T-T_w}}, \quad \text{w.p. } 1-10\delta. \label{final_regret}
\end{align}
Notice that $R_T$ is linear in the initial exploration time $\Tw$ with an exponential dimension dependency. Also note that $\Tw \geq T_0 \coloneqq \poly(\log(1/\delta), \sigma_w^{-1}, n, d, \bar{\alpha}, \gamma^{-1}, \kappa)$ guarantees a stabilizing controller by Lemma~\ref{lem:2norm_bound}. In order to control the growth of $R_T$ by $\tilde{O}(\sqrt{T})$, the initial exploration time can maximally be in the order of $(\sqrt{T})^{1+o(1)}$ where $T^{o(1)}$ hides all multiplicative sub-polynomial growths, \ie, $\Tw = O\pr{(\sqrt{T})^{1+o(1)}}=\tilde{O}(\sqrt{T})$. 
\\ \\ 
On the other hand, Theorem~\ref{thm:optimistic_prob} puts strict lower bounds on the growth of $\Tw$ in order to maintain asymptotically constant optimistic probability. In particular, for singular $A_{c,*}$, this condition is stated as $\Tw = \omega(\sqrt{T} \log{T})$. Combined with the required upper bound $O\pr{(\sqrt{T})^{1+o(1)}}$, it must be that $\Tw = \max\pr{T_0,\, c (\sqrt{T} \log{T})^{1+o(1)}}$ for a constant $c>0$ for large enough $T$. Inserting this result in \eqref{final_regret} gives us
\begin{align}
    R_T &= \tilde{O}\pr{ (n+d)^{n+d}\sqrt{T}}, \quad \text{w.p. } 1-10\delta \nn
\end{align}
for large enough $T$. Observe that exponential dimension dependence is unavoidable in this case as the system is excited with isotropic noise in every direction long enough to dominate with exponential dimension. 
\\ \\
For non-singular $A_{c,*}$, the lower bound is stated as $\Tw = \omega(\log{T})$. For large enough $T$, choosing $\Tw= \max\pr{T_0,\, c (\log{T})^{1+o(1)}}$ for a constant $c >0$ is sufficient to satisfy both the upper and lower bounds on $\Tw$. Inserting this result in \eqref{final_regret} gives us
\begin{align}
    R_T &= \tilde{O}\pr{ \poly(n,d,\log(\nicefrac{1}{\delta}))\sqrt{T}}, \quad \text{w.p. } 1-10\delta \nn
\end{align}
for large enough $T$. Observe that the exponential dimension dependence is not dominant anymore since logarithmically large $\Tw$ is sufficient to guarantee asymptotically constant optimistic probability.


%% file: appendix/technical.tex
\begin{theorem}[Theorem 1 of \citet{abbasi2011improved}]
\label{selfnormalized}
Let $\left(\mathcal{F}_{t} ; k \geq\right.$
$0)$ be a filtration and $\left(m_{k} ; k \geq 0\right)$ be an $\mathbb{R}^{d}$-valued stochastic process adapted to $\left(\mathcal{F}_{k}\right),\left(\eta_{k} ; k \geq 1\right)$
be a real-valued martingale difference process adapted to $\left(\mathcal{F}_{k}\right) .$ Assume that $\eta_{k}$ is conditionally sub-Gaussian with constant $R$. Consider the martingale
\begin{equation*}
S_{t}=\sum\nolimits_{k=1}^{t} \eta_{k} m_{k-1}
\end{equation*}
and the matrix-valued processes
\begin{equation*}
V_{t}=\sum\nolimits_{k=1}^{t} m_{k-1} m_{k-1}^{\tp}, \quad \overline{V}_{t}=V+V_{t}, \quad t \geq 0
\end{equation*}
Then for any $0<\delta<1$, with probability $1-\delta$
\begin{equation*}
\forall t \geq 0, \quad\left\|S_{t}\right\|^2_{\overline{V}_{t}^{-1}} \leq 2 R^{2} \log \left(\frac{\operatorname{det}\left(\overline{V}_{t}\right)^{1 / 2} \operatorname{det}(V)^{-1 / 2}}{\delta}\right)
\end{equation*}

\end{theorem}

\begin{theorem}[Azuma's inequality]  \label{Azuma}
Assume that $X_{s}$ is a supermartingale and $|X_{s}-X_{s-1}| \leq c_{s}$ almost surely for $s\geq 0$. Then for all $t>0$ and all $\epsilon>0$,
\begin{equation*}
P\left(\left|X_{t}-X_{0}\right| \geq \epsilon\right) \leq 2 \exp \left(\frac{-\epsilon^{2}}{2 \sum_{s=1}^{t} c_{s}^{2}}\right)
\end{equation*}
\end{theorem}

\begin{lemma}[Lemma 10 of~\citet{abbasi2011lqr}]\label{upperboundlemma}
The following holds for any $t \geq 1$ :
\begin{equation*}
\sum_{k=0}^{t-1}\left(\left\|z_{k}\right\|_{V_{k}^{-1}}^{2} \wedge 1\right) \leq 2 \log \frac{\operatorname{det}\left(V_{t}\right)}{\operatorname{det}(\lambda I)}
\end{equation*}
Further, when the covariates satisfy $\left\|z_{t}\right\| \leq c_{m}, t \geq 0$ with some $c_{m}>0$ w.p. 1 then
\begin{equation*}
\log \frac{\operatorname{det}\left(V_{t}\right)}{\operatorname{det}(\lambda I)} \leq(n+d) \log \left(\frac{\lambda(n+d)+t c_{m}^{2}}{\lambda(n+d)}\right)
\end{equation*}
\end{lemma}

\begin{lemma}[Norm of Subgaussian vector]\label{subgauss lemma}
Let $v\in \mathbb{R}^d$ be a entry-wise $R$-subgaussian random variable. Then with probability $1-\delta$, $\|v\| \leq R\sqrt{2d\log(d/\delta)}$.
\end{lemma}

\begin{lemma}[Theorem 20 of \citet{cohen2019learning}] \label{lem:persistence}
 Let $z_{t} \in \mathbb{R}^{n+d}$ for $t=0, 1, \ldots$ be a sequence random variables that is adapted to a filtration $\left\{\mathcal{F}_{t}\right\}_{t=0}^{\infty} .$ Suppose that $z_{t}$ are conditionally Gaussian on $\mathcal{F}_{t-1}$ and that $\mathbb{E}\left[z_{t} z_{t}^{T} \mid \mathcal{F}_{t-1}\right] \succeq$ $\sigma_{z}^{2} I$ for some fixed $\sigma_{z}^{2}>0$. Then for $t \geq 200 (n+d) \log \frac{12}{\delta}$ we have that with probability at least $1-\delta$
$$
\sum_{s=1}^{t} z_{s} z_{s}^{T} \succeq \frac{t \sigma_{z}^{2}}{40} I .
$$
\end{lemma}



%% file: appendix/simulation.tex
The LQR problem for the longitudinal flight control of Boeing 747 with linearized dynamics~\citep{747model} is given as
\begin{equation} \label{eq:boeing}
    A_* = \begin{bmatrix}
    0.99 & 0.03 & -0.02 & -0.32 \\
    0.01 & 0.47 & 4.7 & 0 \\
    0.02 & -0.06 & 0.4 & 0 \\
    0.01 & -0.04 & 0.72 & 0.99 
\end{bmatrix}, ~ 
B_* = \begin{bmatrix}
  0.01 & 0.99 \\
  -3.44 & 1.66 \\
  -0.83 & 0.44 \\
  -0.47 & 0.25
\end{bmatrix}, ~ Q = I, ~ R = I, ~ w \sim \mathcal{N}(0,I).
\end{equation}

This system has been studied in \citep{sabag2021regret,lale2022reinforcement}. It corresponds to the dynamics for level flight of Boeing 747 at the altitude of 40000ft with the speed of 774ft/sec, for a discretization of 1 second. The first element of the state corresponds to the velocity of aircraft along body axis, the second is the velocity of aircraft perpendicular to body axis, the third is the angle between body axis and horizontal and the fourth is the angular velocity of aircraft. The system takes two dimensional inputs, where the first is the elevator angle and the second one is thrust.

For this task we deploy 4 different adaptive control algorithms that do not require initial stabilizing controller: (i) \alg, (ii) StabL of \citet{lale2022reinforcement}, (iii) TS-LQR of \citet{abeille2018improved}, and (iv) OFULQ of \citet{abbasi2011lqr}. Each algorithm has certain hyperparamters and we tune each parameter in terms of its effect on refret and present the performance of the best performing hyperparameter choices. We use the actual estimation errors in the algorithm design. Note that this has been observed to have negligible effect on the performance~\citep{dean2018regret}. 

To have fair comparison in the regret performance in a stabilizable system like \eqref{eq:boeing}, we follow fixed update rule in TS-LQR in parallel with \alg, and add an additional minimum policy duration constraint to the standard design matrix determinant doubling of OFULQ. Moreover, in the implementation of optimistic parameter search we deploy projected gradient descent (PGD). Even though this approach works efficiently for the small dimensional problems such as \eqref{eq:boeing}, it becomes computationally challenging as the dimensionality of the system grows. Nevertheless, our results show that PGD is effective to find optimistic parameters and as observed in \citet{lale2022reinforcement} yields the superior performance of StabL with a small margin between \alg. This difference is in parallel with the predictions of theory. As we show in our analysis, TS samples an optimistic model with a fixed probability. However, an effective way of solving the optimistic control design problem yields optimistic controllers at every time-step and gives more effective control over exploration vs. exploitation trade-off.